%% file: main.tex
\documentclass[10pt,twocolumn,letterpaper]{article}

\usepackage{amsmath,amssymb}
\AtBeginDocument{%
  \setlength{\abovedisplayskip}{3pt}
  \setlength{\belowdisplayskip}{3pt}
  \setlength{\abovedisplayshortskip}{0pt}
  \setlength{\belowdisplayshortskip}{2pt}
}
\usepackage{mathtools}  

\usepackage{amsthm}
\newtheorem{theorem}{Theorem}
\newtheorem{lemma}{Lemma}
\newtheorem{proposition}{Proposition}
\newtheorem{corollary}{Corollary}  
\newtheorem{assumption}{Assumption}

\newtheorem{property}{Property}
\theoremstyle{remark}
\newtheorem{remark}{Remark}

\usepackage{array}
\usepackage{booktabs}
\usepackage{multirow}
\usepackage{makecell}
\usepackage{tabularx}
\usepackage{siunitx}
\usepackage[table]{xcolor}

\usepackage{float}

\usepackage{placeins}
\usepackage{subcaption}
\usepackage{adjustbox}
\usepackage{caption}
\usepackage{algorithm}
\usepackage{algorithmic}

\usepackage{pifont}
\newcommand{\cmark}{\ding{51}}
\newcommand{\xmark}{\ding{55}}

\newcolumntype{C}{>{\centering\arraybackslash}X}
\definecolor{LightRed}{RGB}{255,182,193} 
\definecolor{LightBlue}{RGB}{173, 216, 230}

 \usepackage[pagenumbers]{cvpr} 

\input{preamble}

\definecolor{cvprblue}{rgb}{0.21,0.49,0.74}
\usepackage[pagebackref,breaklinks,colorlinks,allcolors=cvprblue]{hyperref}

\def\paperID{4914} 
\def\confName{CVPR}
\def\confYear{2026}

\title{ILoRA: Federated Learning with Low-Rank Adaptation for
	Heterogeneous Client Aggregation}

\author{First Author\\
Junchao Zhou\\
College of Intelligence and Computing, Tianjin University\\
{\tt\small junchaozhouu@gmail.com}
\and
Second Author\\
Institution2\\
College of Intelligence and Computing, Tianjin University\\
{\tt\small junkangliukk@gmail.com}
\and
Fanhua Shang\\
School of Computer Science and Technology, Tianjin University, Tianjin, China\\
{\tt\small fhshang@tju.edu.cn}
}

\begin{document}
\maketitle
\input{sec/0_abstract}    
\input{sec/1_intro}
\input{sec/2_formatting}

\input{sec/3_finalcopy}
{
    \small
    \bibliographystyle{ieeenat_fullname}
    \bibliography{main}
}

 \input{sec/X_suppl}

\end{document}

%% file: preamble.tex
\usepackage{bm}



%% file: sec/0_abstract.tex
\begin{abstract}
Federated Learning with Low-Rank Adaptation (LoRA) faces three critical challenges under client heterogeneity: (1) \textbf{Initialization-Induced Instability} due to random initialization misaligning client subspaces; (2) \textbf{Rank Incompatibility and Aggregation Error} when averaging LoRA parameters of different ranks, which biases the global model; and (3) exacerbated \textbf{Client Drift under Non-IID Data}, impairing generalization. To address these challenges, we propose \textbf{ILoRA}, a unified framework that integrates three core innovations: a \textbf{QR-based orthonormal initialization} to ensure all clients start in a coherent subspace; a \textbf{Concatenated QR Aggregation} mechanism that fuses heterogeneous-rank updates via concatenation and decomposition, preserving information while maintaining dimension alignment; and an \textbf{AdamW optimizer with rank-aware control variates} to correct local updates and mitigate client drift. Supported by \textbf{theoretical convergence guarantees}, extensive experiments on vision and NLP benchmarks demonstrate that ILoRA consistently achieves superior accuracy and convergence stability compared to existing federated LoRA methods.
\end{abstract}

%% file: sec/1_intro.tex
\section{Introduction}

\label{sec:intro}
\begin{figure}[tb]
	\centering
	\includegraphics[width=\linewidth]{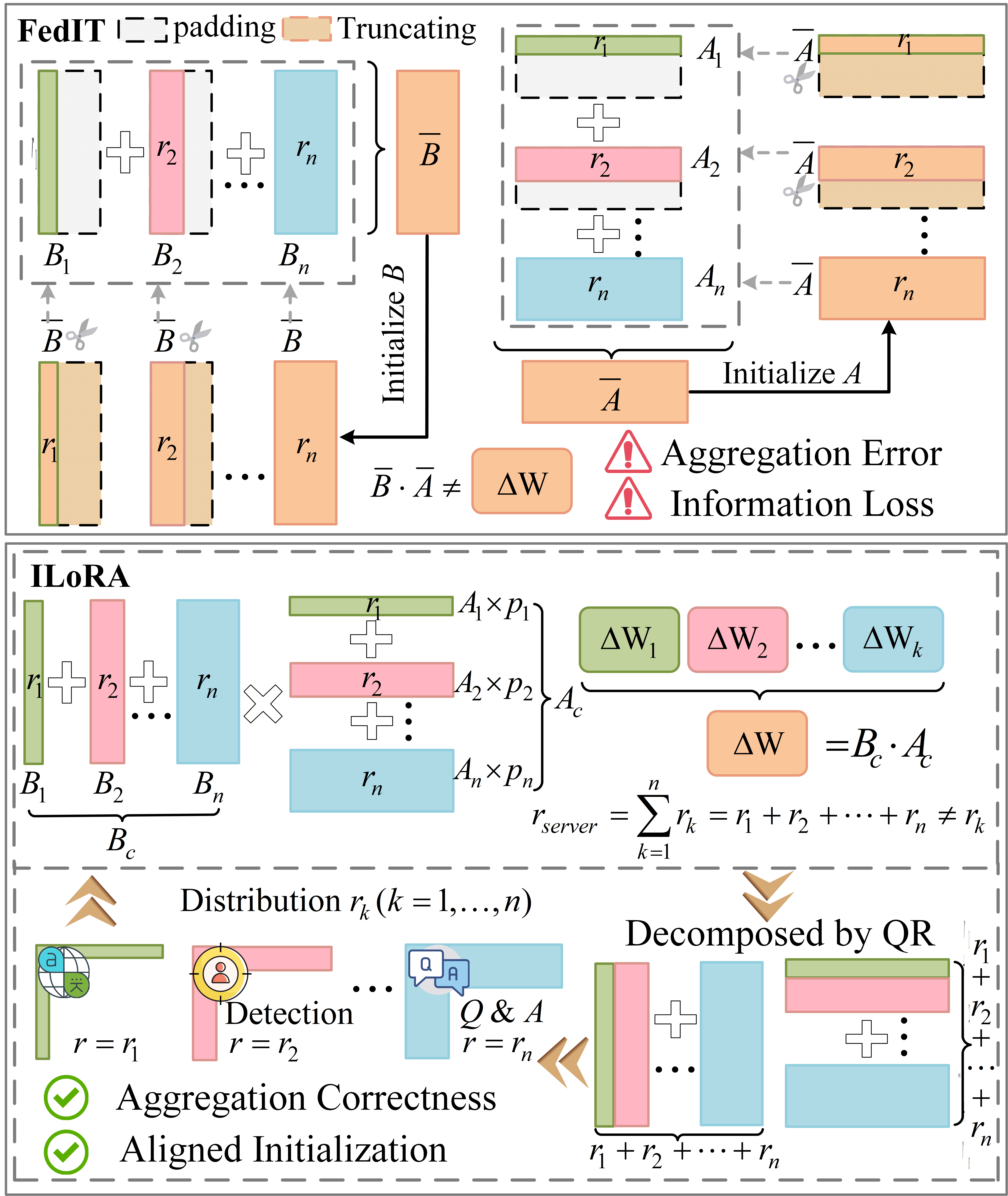}
	\caption{
		Comparison of federated LoRA methods: \textbf{FedIT} (aggregation error and information loss) and \textbf{ILoRA (ours)} (both correct aggregation and aligned initialization).
	}
	\label{fig:lora_aggregation_comparison}
\end{figure}
\begin{figure*}[h!]
	\centering
	\includegraphics[width=\textwidth]{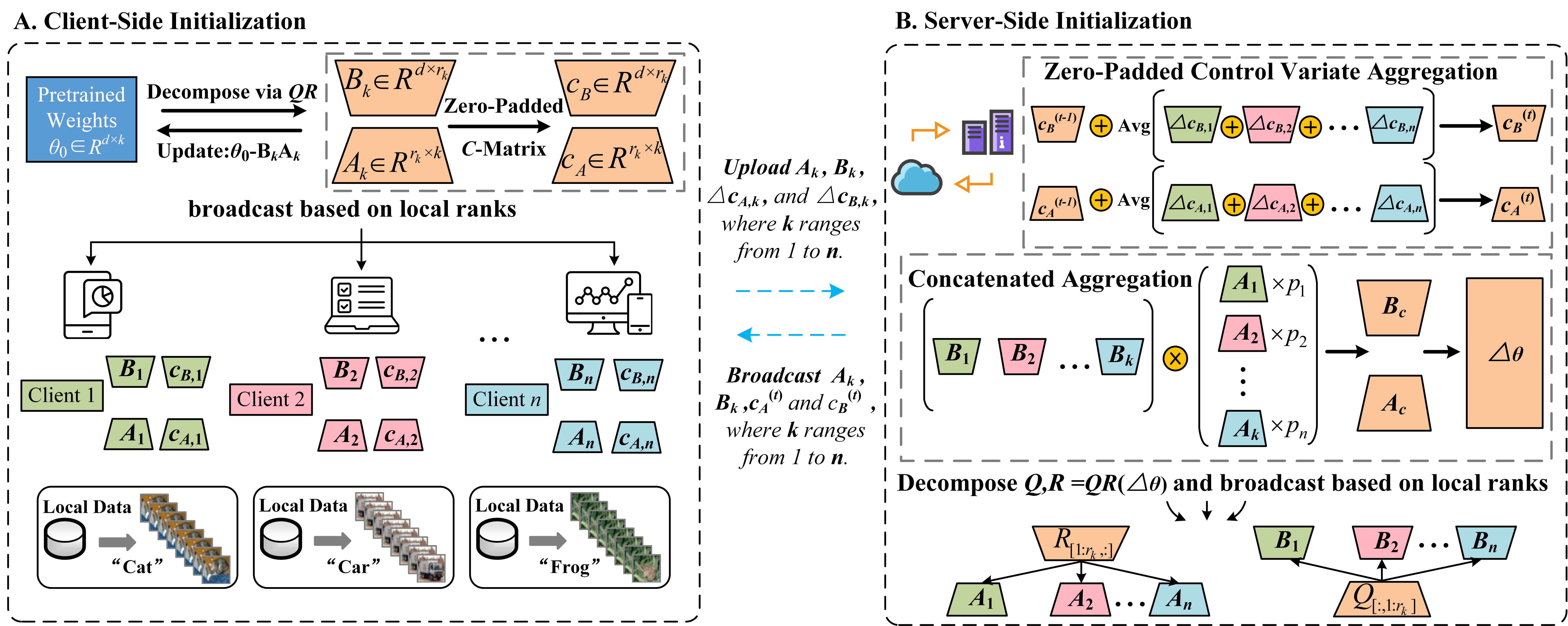}
\caption{
	Overview of ILoRA: Clients fine-tune LoRA modules locally; the server aggregates updates via concatenated QR decomposition into a global orthogonal basis $(Q, R)$, enabling efficient communication, subspace alignment, and drift mitigation under Non-IID data.
}
	\vspace{-15pt} 
	\label{fig:ILoRA_arch}
\end{figure*}
The rapid progress of foundation models has significantly advanced AI capabilities in vision and language domains \cite{brown2020language, touvron2023llama, devlin2019bert}. However, full-parameter fine-tuning of these models remains computationally expensive and data-intensive \cite{brown2020language, devlin2019bert}. To address this, Parameter-Efficient Fine-Tuning (PEFT) methods have been developed, which freeze the pre-trained backbone and update only a small set of auxiliary parameters \cite{hu2022lora, zhang2023adalora, liu2024dora, valipour2022dylora}. Among these, LoRA has emerged as a prominent approach, maintaining the base model's representational capacity while enabling efficient adaptation via low-rank update matrices \cite{hu2022lora, buyukakyuz2024olora, liu2024dora, valipour2022dylora}.

FL facilitates collaborative model training across decentralized clients without raw data sharing, ensuring privacy-preserving adaptation \cite{mcmahan2017communication, kairouz2021advances, hard2018federated, caldas2018leaf}. Although LoRA integration offers parameter efficiency and privacy in FL, practical systems encounter substantial client heterogeneity in computation, communication, and data distributions. Such heterogeneity induces divergent LoRA ranks and Non-IID data partitions \cite{li2020federated, kairouz2021advances, zhao2018federated, hsu2019measuring}, leading to aggregation misalignment, convergence instability, and performance degradation \cite{li2020federated, wang2024flora, bian2024lora, cho2024heterogeneous, bai2024federated, zhao2018federated}.

This paper studies the integration of FL with parameter-efficient fine-tuning, focusing on the federated fine-tuning of LoRA. We identify three critical challenges:

\textbf{Challenge 1: Initialization-Induced Instability.} Random initialization of LoRA  across clients creates misaligned adaptation subspaces \cite{caldas2018leaf}, slowing convergence, and destabilizing training \cite{hu2022lora, buyukakyuz2024olora, bian2024lora}.

\textbf{Challenge 2: Rank Incompatibility and Aggregation Error.} Heterogeneous client resources lead to varying LoRA ranks \cite{su2023fedra}, yet standard FL aggregation assumes parameter alignment \cite{han2024parameter}, causing aggregation errors and biased global models \cite{wang2024flora, cho2024heterogeneous, bai2024federated}.

\textbf{Challenge 3: Client Drift under Non-IID Data.} LoRA's low-rank parameterization amplifies FL's statistical heterogeneity \cite{he2021towards, li2022federated}, increasing local-global update divergence. Subspace misalignment further impedes variance reduction, worsening optimization instability \cite{karimireddy2020scaffold, li2020federated, zhao2018federated, hsu2019measuring}.

Existing methods for federated fine-tuning fail to address all three challenges simultaneously in heterogeneous environments. While some target specific issues \cite{babakniya2023slora,deng2009imagenet}, each has critical limitations: FedIT improves communication but requires homogeneous ranks \cite{zhang2024towards}; FLoRA allows rank heterogeneity but incurs high communication costs and fails to stabilize training \cite{wang2024flora}; SCAFFOLD reduces Non-IID bias but cannot handle rank variations \cite{karimireddy2020scaffold}. Similarly, personalized FL methods using clustering or distillation \cite{vahidian-flis, yang2024fedfed} are incompatible with rank-heterogeneous parameter spaces.

To address these challenges, we propose \textbf{ILoRA}, the first framework to systematically apply QR decomposition for federated LoRA fine-tuning. Our unified approach integrates: \textbf{QR-based initialization} for subspace alignment \cite{buyukakyuz2024olora,bian2024lora}, \textbf{concatenated QR aggregation} for rank-heterogeneous fusion \cite{golub2013matrix,zhu2024asymmetry,zaken2021bitfit}, and \textbf{AdamW with control variates} for drift mitigation \cite{karimireddy2020scaffold}. Supported by theory \cite{neyshabur2017theoretical} and extensive experiments \cite{deng2009imagenet,wang2018glue,brown2020language,touvron2023llama,devlin2019bert}, ILoRA achieves superior accuracy and stability, as shown in Fig. \ref{fig:lora_aggregation_comparison}.
Our key contributions are summarized as follows:
\begin{itemize}
    \item \textbf{The ILoRA Framework.} A unified solution addressing initialization instability, rank-heterogeneous aggregation, and client drift in federated LoRA fine-tuning.
    \item \textbf{QR-based Heterogeneous Aggregation.} A novel protocol fusing different LoRA ranks into a unified subspace via concatenation and QR decomposition.
    \item \textbf{Coherent Initialization and Optimization.} Orthonormal initialization ensuring subspace consistency, with rank-aware control variates for drift mitigation.
    \item \textbf{Theoretical and Empirical Validation.} Convergence guarantees and extensive experiments demonstrating state-of-the-art performance.
\end{itemize}

%% file: sec/2_formatting.tex
\section{Related Work}
\label{sec:related_work}

\textbf{Parameter-Efficient Fine-Tuning.} 
PEFT methods adapt large models efficiently by updating minimal parameters while freezing the backbone \cite{han2024parameter}. 
Leading approaches include Adapter \cite{houlsby2019parameter}, Prefix-Tuning \cite{li2021prefix}, Prompt Tuning \cite{lester2021power}, and LoRA \cite{hu2022lora} with its low-rank decomposition. 
Recent variants introduce adaptive ranks (AdaLoRA) \cite{zhang2023adalora}, orthogonal constraints (OLoRA) \cite{buyukakyuz2024olora}, weight decomposition (DoRA) \cite{liu2024dora}, and dynamic ranks (DyLoRA) \cite{valipour2022dylora}. 
While showing strong centralized performance, these methods remain largely unexplored in federated settings.

\textbf{Parameter-Efficient Fine-Tuning in FL.} 
Recent works integrate PEFT with FL to reduce communication costs while preserving privacy \cite{ding2023parameter}. 
FedIT \cite{zhang2024towards} pioneers LoRA for federated instruction tuning, but introduces significant \textit{aggregation noise} by independently averaging LoRA factors. 
FLoRA \cite{wang2024flora} addresses rank heterogeneity via parameter concatenation, yet suffers from $O(K \cdot r_{\max})$ communication overhead and dimension mismatch.
FFA-LoRA \cite{sun2024improving} improves stability via partial freezing but lacks proper dimension alignment \cite{su2023fedra}, while LoRA-FAIR \cite{bian2024lora} enhances aggregation but requires homogeneous ranks.

\textbf{Heterogeneity in Federated PEFT.} A core challenge in federated PEFT is heterogeneity, including both Non-IID data and systemic rank variations. Traditional aggregation methods fail under varying LoRA ranks due to dimension mismatch. Existing solutions have distinct limitations: \textbf{Zero-padding} causes optimization bias under large rank differences \cite{cho2024heterogeneous}; \textbf{SVD-based methods} exhibit error amplification with rank ratios exceeding 2:1 \cite{liu2024dora}; and \textbf{Full concatenation} incurs high communication overhead \cite{wang2024flora}.

\textbf{Data Heterogeneity in Federated Learning.} 
Non-IID data induces client drift, degrading accuracy and slowing convergence. 
Existing mitigation methods include FedProx \cite{li2020federated}, SCAFFOLD \cite{karimireddy2020scaffold}, FedCM \cite{xu2021fedcm}, FedOpt \cite{reddi2020adaptive}, and FedLADA \cite{yang2024nonparametric}. 
Personalized FL offers customization but assumes dimension consistency or fails in rank-heterogeneous settings. 
ILoRA overcomes this via \textbf{rank-aware control variate AdamW}, extending variance reduction to rank-heterogeneous Non-IID environments.
\begin{figure}[tb]
\centering
\includegraphics[width=0.95\linewidth]{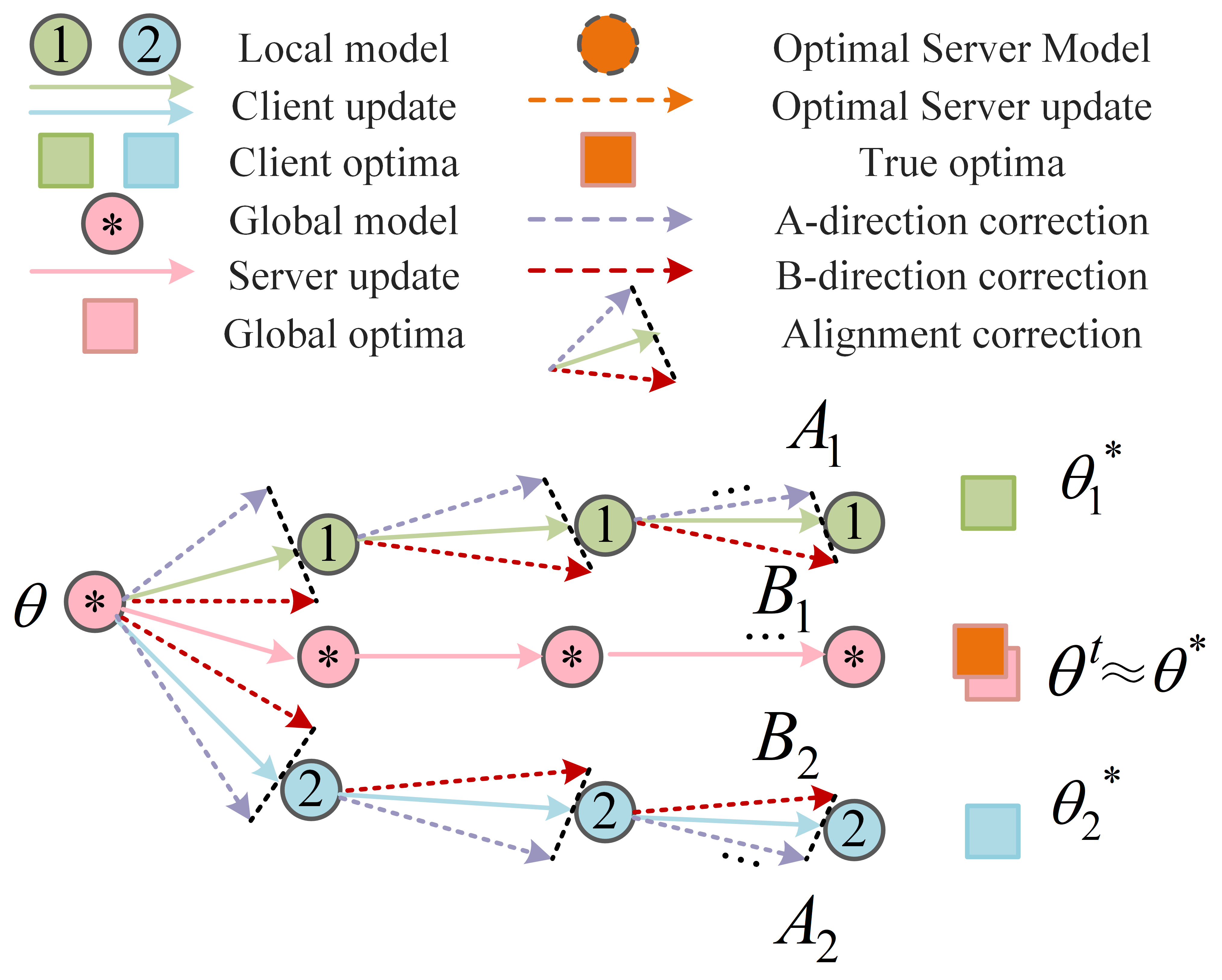}
\caption{Federated learning model updates with alignment correction. Local models (Client 1 and 2) are guided by corrections, leading the global model to converge near the true optima.}
\label{fig:control3}
\end{figure}

%% file: sec/3_finalcopy.tex
\section{Preliminaries}
\label{sec:preliminaries}
\begin{algorithm}[!htbp]
\caption{Client-Side Procedure for two Methods \fcolorbox{LightBlue}{LightBlue}{ILoRA} and \fcolorbox{LightRed}{LightRed}{ILoRA-S}}
\label{alg:client}

\textbf{Input:} Local dataset $\mathcal{D}_k$; local rank $r_k$; local epochs $E$; learning rate $\eta$; \fcolorbox{LightRed!80!white}{LightRed!80!white}{local control variates $\mathbf{c}_{A,k}, \mathbf{c}_{B,k}$}.

\begin{algorithmic}[1]
    \IF{first round}
        \STATE Compute QR decomposition: $\mathbf{Q}_k, \mathbf{R}_k \leftarrow \operatorname{QR}(\bm{\theta}_0)$;
        \STATE Initialize LoRA: $\mathbf{A}_k \leftarrow \mathbf{R}_{k,:r_k,:}$, $\mathbf{B}_k \leftarrow \mathbf{Q}_{k,:,:r_k}$;
        \STATE Initialize local model: $\bm{\theta}_k \leftarrow \bm{\theta}_0 - \mathbf{B}_k\mathbf{A}_k$;
    \ELSE
        \STATE \textbf{Receive} $\{\mathbf{A}_{k}, \mathbf{B}_{k}\}$, \fcolorbox{LightRed!80!white}{LightRed!80!white}{$\mathbf{c}^{(t-1)}_A, \mathbf{c}^{(t-1)}_B$};
        \STATE Update local model: $\bm{\theta}_k \leftarrow \bm{\theta}_k + \mathbf{B}_k\mathbf{A}_k$;
    \ENDIF
    
    \FOR{$e=1$ \textbf{to} $E$}
        \FOR{each mini-batch $B$ in $\mathcal{D}_k$}
            \STATE Sample $B$ from $\mathcal{D}_k$;
            
            \STATE \fcolorbox{LightBlue!80!white}{LightBlue!80!white}{Compute $\mathbf{g}_A \leftarrow \nabla\mathcal{L}_k/\partial\mathbf{A}_k$, $\mathbf{g}_B \leftarrow \nabla\mathcal{L}_k/\partial\mathbf{B}_k$};
            \STATE \fcolorbox{LightRed!80!white}{LightRed!80!white}{Compute $\tilde{\mathbf{g}}_A, \tilde{\mathbf{g}}_B$ via Eqs.\eqref{eq:corrected_gradient_A} and \eqref{eq:corrected_gradient_B}};
            
            \STATE \fcolorbox{LightBlue!80!white}{LightBlue!80!white}{$\operatorname{AdamW}(\mathbf{A}_k, \mathbf{B}_k, \mathbf{g}_A, \mathbf{g}_B)$};
            \STATE \fcolorbox{LightRed!80!white}{LightRed!80!white}{$\operatorname{AdamW}(\mathbf{A}_k, \mathbf{B}_k, \tilde{\mathbf{g}}_A, \tilde{\mathbf{g}}_B)$};
        \ENDFOR
        
        \STATE \fcolorbox{LightRed!80!white}{LightRed!80!white}{Update control variates via \eqref{eq:control_update_A} and \eqref{eq:control_update_B}};
    \ENDFOR
    
    \STATE \fcolorbox{LightBlue!80!white}{LightBlue!80!white}{\textbf{Send} $(\mathbf{A}_k, \mathbf{B}_k, n_k)$}.
    \STATE \fcolorbox{LightRed!80!white}{LightRed!80!white}{\textbf{Send} $(\mathbf{A}_k, \mathbf{B}_k, n_k, \Delta\mathbf{c}_{A,k}, \Delta\mathbf{c}_{B,k})$}.
\end{algorithmic}
\end{algorithm}
\subsection{Federated Learning}

Federated Learning (FL) enables collaborative model training across decentralized clients without data sharing \cite{mcmahan2017communication}. 
The global objective minimizes:
\begin{equation}
    \min_{\bm{\theta}} F(\bm{\theta}) = \sum_{k=1}^{K} p_k F_k(\bm{\theta}),
\end{equation}
where $K$ is the client count, $p_k = n_k/n$ weights client $k$ with local data size $n_k$, $n = \sum_{k=1}^{K} n_k$ is the total data size, and $F_k(\bm{\theta})$ is the local objective. 
Non-IID data induces client drift \cite{kairouz2021advances}, degrading performance. This worsens with communication bottlenecks, system heterogeneity, and training instability in federated fine-tuning.
\subsection{PEFT with LoRA}
LoRA enables efficient fine-tuning by freezing pre-trained weights and learning low-rank updates to specific matrices in $\bm{\theta}$ \cite{hu2022lora}. For $\bm{\theta}_0 \in \mathbb{R}^{d \times k}$, LoRA learns:
\begin{equation}
    \bm{\theta} = \bm{\theta}_0 + \Delta\bm{\theta} = \bm{\theta}_0 + \mathbf{B}\mathbf{A},
\end{equation}
where $\mathbf{B} \in \mathbb{R}^{d \times r}$, $\mathbf{A} \in \mathbb{R}^{r \times k}$, and $r \ll \min(d,k)$. The layer's forward pass becomes:
\begin{equation}
    \mathbf{h} = \bm{\theta}_0 \mathbf{x} + \mathbf{B}\mathbf{A}\mathbf{x},
\end{equation}
with input $\mathbf{x}$ and output $\mathbf{h}$. Standard LoRA initializes $\mathbf{A}$ from Gaussian and $\mathbf{B}$ to zero, scaling to control update size. Recent variants use orthonormal initialization \cite{buyukakyuz2024olora} and adaptive strategies \cite{zhang2023adalora,valipour2022dylora}.

\section{Challenges in Federated LoRA}

\subsection{Challenge 1:  Initialization-Induced Instability}
\label{subsec:challenge1}

Standard LoRA's random Gaussian $\mathbf{A}$ and zero $\mathbf{B}$ initialization \cite{hu2022lora}, while effective centrally, causes \emph{early instability} in federated settings: \textbf{Subspace Misalignment:} Random $\mathbf{A}_k$ creates conflicting adaptation directions. \textbf{First-Round Amplification:} Aggregating misaligned $\{\mathbf{A}_k, \mathbf{B}_k\}$ distorts global updates and impairs convergence.

\subsection{Challenge 2: Rank Incompatibility and Aggregation Error}
\label{subsec:challenge2}

Client heterogeneity induces varying LoRA ranks $r_k$ \cite{cho2024heterogeneous,wang2024flora}. Standard federated averaging assumes homogeneous dimensions, becoming invalid under rank heterogeneity. Direct averaging produces $\Delta\bm{\theta}' = \bar{\mathbf{B}}\bar{\mathbf{A}}$, differing from correct aggregation:

\begin{equation}
    \Delta\bm{\theta} = \sum_k p_k (\mathbf{B}_k \mathbf{A}_k) \neq \bar{\mathbf{B}}\bar{\mathbf{A}}.
    \label{eq:aggregation_bias}
\end{equation}
This bias persists even with homogeneous ranks, degrading performance \cite{bian2024lora,wang2024flora}.

\subsection{Challenge 3: Client Drift under Non-IID Data}

Non-IID data \cite{kairouz2021advances} exacerbates \textbf{client drift} in LoRA fine-tuning, where local gradients $\nabla F_k(\boldsymbol{\theta})$ deviate from $\nabla F(\boldsymbol{\theta})$ and aggregated updates move away from the global optimum. In LoRA, drift amplifies because updates depend on subspace orientations via $\Delta\bm{\theta} = \mathbf{B}\mathbf{A}$. Under Non-IID data, client subspaces can become nearly orthogonal, increasing local-global misalignment \cite{karimireddy2020scaffold}. Table~\ref{tab:qqp_results} (Appendix~\ref{app:additional_results}) shows performance degrading with heterogeneity (e.g., FedIT drops 19\% from $\alpha=0.8$ to $0.4$). Figure~\ref{fig:control1and2} (Appendix~\ref{app:additional_results}) visualizes this drift. Our control variates mitigate it, achieving robust convergence (Figure~\ref{fig:control3}).

\section{Proposed Method: ILoRA}
\label{sec:method}
\begin{algorithm}[tbp]
    \caption{Server-Side Procedure for two Methods \fcolorbox{LightBlue}{LightBlue}{ILoRA} and \fcolorbox{LightRed}{LightRed}{ILoRA-S}}
    \label{alg:server}
    
    \textbf{Input:} $\bm{\theta}_0$, $K$, $p$, $\{r_k\}$, $r_s$, $T$, 
    \fcolorbox{LightRed!80!white}{LightRed!80!white}{$\mathbf{c}^{(0)}_A, \mathbf{c}^{(0)}_B \leftarrow \mathbf{0}$}.
    
    \begin{algorithmic}[1]
        \FOR{$t=1$ \textbf{to} $T$}
        \STATE Sample $\mathcal{S}_t \subset \{1,\dots,K\}$ with $|\mathcal{S}_t| = \lfloor pK \rfloor$;
        \STATE $N \leftarrow \sum_{k \in \mathcal{S}_t} n_k$;
        
        
        \STATE \textbf{Receive} from $k \in \mathcal{S}_t$:;
        \STATE \quad \fcolorbox{LightBlue!80!white}{LightBlue!80!white}{$(\mathbf{A}_k, \mathbf{B}_k, n_k)$};
        \STATE \quad \fcolorbox{LightRed!80!white}{LightRed!80!white}{$(\mathbf{A}_k, \mathbf{B}_k, n_k, \Delta\mathbf{c}_{A,k}, \Delta\mathbf{c}_{B,k})$};
        
        \STATE Construct $\mathbf{A}_{\text{c}}$ and $\mathbf{B}_{\text{c}}$ via \eqref{eq:qr_aggregation_stack};
        \STATE Compute $\Delta\bm{\theta} \leftarrow \mathbf{B}_{\text{c}} \mathbf{A}_{\text{c}}$;
        \STATE Compute $\mathbf{Q}, \mathbf{R} \leftarrow \operatorname{QR}(\Delta\bm{\theta})$;
        \STATE Set $\mathbf{B}_s \leftarrow \mathbf{Q}_{:,:r_s}$, $\mathbf{A}_s \leftarrow \mathbf{R}_{:r_s,:}$;
        
        \STATE \fcolorbox{LightRed!80!white}{LightRed!80!white}{Update global control variates via \eqref{eq:update_global_control_A} and \eqref{eq:update_global_control_B}};
        
        \STATE \textbf{Personalize for each} $k \in \mathcal{S}_t$:;
        \STATE \quad $\mathbf{B}_{k} \leftarrow \mathbf{Q}_{:,:r_k}$, $\mathbf{A}_{k} \leftarrow \mathbf{R}_{:r_k,:}$;
        \STATE \quad \fcolorbox{LightBlue!80!white}{LightBlue!80!white}{Send $\{\mathbf{A}_{k}, \mathbf{B}_{k}\}$}.
        \STATE \quad \fcolorbox{LightRed!80!white}{LightRed!80!white}{Send $\{\mathbf{A}_{k}, \mathbf{B}_{k}, \mathbf{c}^{(t)}_A, \mathbf{c}^{(t)}_B\}$}.
        \ENDFOR
    \end{algorithmic}
\end{algorithm}

\newcommand{\ppoverlay}[1]{%
	\makebox[0pt][l]{\hspace{0.18em}\raisebox{-0.62ex}{\smash{\tiny #1}}}%
}
\newcommand{\ppup}[1]{\ppoverlay{\textcolor{red!70!black}{$\uparrow$#1}}}  
\newcommand{\ppdown}[1]{\ppoverlay{\textcolor{green!60!black}{$\downarrow$#1}}}  
\newcommand{\ppbase}{\makebox[0pt][l]{\hspace{0.18em}\raisebox{-0.62ex}{\smash{\tiny\textcolor{gray!65}{base}}}}}

\newcommand{\cellnum}[2]{\makebox[2.9em][c]{#1#2}}
\begin{figure*}[t]
	\centering
	\includegraphics[width=\textwidth]{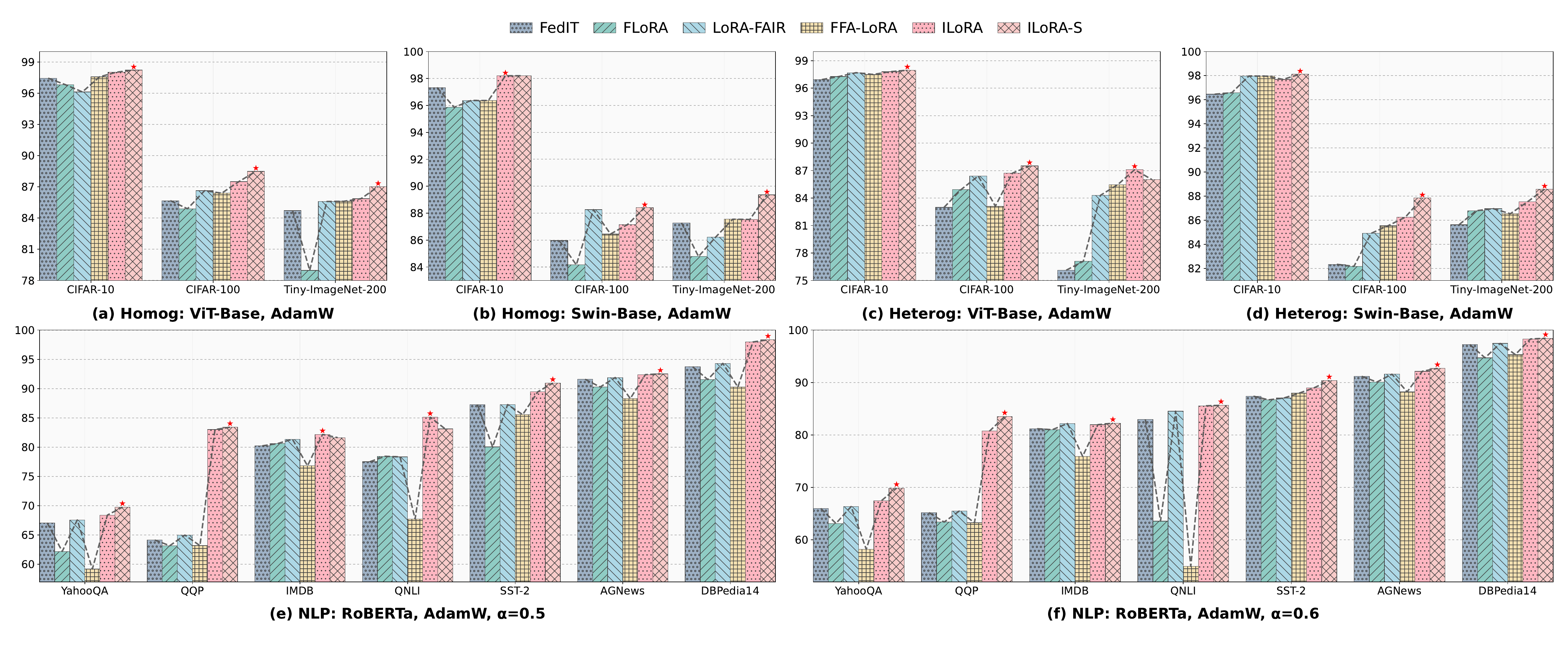}
	\caption{Performance comparison across settings. (a-d) CV tasks with ViT-Base/Swin-Base; (e-f) NLP tasks with RoBERTa.}
	\vspace{-5pt} \label{fig:combined_comparison}
\end{figure*}
\begin{table*}[!ht]
	\centering
    
\caption{Accuracy comparison (\%) of  heterogeneous LoRA methods on CIFAR-10/100 and Tiny-ImageNet (ViT/Swin, Dir($\alpha$=0.3))}
	\label{tab:heterogeneous_lora_comparison}
	\normalsize
	\renewcommand{\arraystretch}{1}\setlength{\tabcolsep}{3.5pt}
	
	\begin{tabularx}{\textwidth}{l*{12}{C}}
		
		\arrayrulecolor{black}\specialrule{1.2pt}{0pt}{0pt}
		
		\multirow{4}{*}{Method} & \multicolumn{6}{c}{ViT-Base} & \multicolumn{6}{c}{Swin-Base} \\
		\cmidrule(lr){2-7} \cmidrule(lr){8-13}
		& \multicolumn{2}{c}{CIFAR-10} & \multicolumn{2}{c}{CIFAR-100} & \multicolumn{2}{c}{Tiny-ImageNet}
		& \multicolumn{2}{c}{CIFAR-10} & \multicolumn{2}{c}{CIFAR-100} & \multicolumn{2}{c}{Tiny-ImageNet} \\
		\cmidrule(lr){2-3} \cmidrule(lr){4-5} \cmidrule(lr){6-7}
		\cmidrule(lr){8-9} \cmidrule(lr){10-11} \cmidrule(lr){12-13}
		& SGD & AdamW & SGD & AdamW & SGD & AdamW & SGD & AdamW & SGD & AdamW & SGD & AdamW \\
		\midrule
		
		FedIT      & 97.85 & 97.66 & 90.09 & 85.19 & 87.23 & 84.72 & 98.17 & 97.08 & 88.30 & 85.25 & 87.19 & 87.40 \\
		FLoRA      & 95.28 & 97.27 & 87.66 & 84.93 & 83.39 & 77.10 & 97.69 & 96.58 & 88.75 & 82.17 & 88.10 & 86.79 \\
		LoRA-FAIR & 97.96 & 97.69 & 90.05 & 86.41 & 87.41 & 84.30 & 97.28 & 97.97 & 89.23 & 84.91 & 89.13 & 86.96 \\
		FFA-LoRA   & 97.20 & 97.50 & 89.46 & 83.12 & 83.90 & 85.45 & 97.58 & 97.96 & 88.85 & 85.55 & 87.28 & 86.54 \\
		\rowcolor{LightBlue!80!white}
		ILoRA      & 98.02 & 97.80 & 90.16 & 86.72 & 87.29 & \textbf{87.10} & 98.19 & 97.69 & 89.66 & 86.25 & 89.44 & 87.53 \\
		\rowcolor{LightRed!80!white}
		ILoRA-S     & \textbf{98.19} & \textbf{97.96} & \textbf{90.39} & \textbf{87.51} & \textbf{87.43} & 86.03 & \textbf{98.36} & \textbf{98.13} & \textbf{90.51} & \textbf{87.85} & \textbf{89.90} & \textbf{88.58} \\
		
		\arrayrulecolor{black}\specialrule{1.2pt}{0pt}{0pt}
        
	\end{tabularx}
    \vspace{-10pt}  
\end{table*}
\subsection{QR-Based Orthogonal Initialization}
\label{subsec:orthogonal_init}

To address initialization instability, ILoRA employs client-generated orthonormal bases for consistent subspace alignment. Each client $k$ with rank $r_k$ computes local QR decomposition:
\begin{equation}
    \mathbf{Q}_k, \mathbf{R}_k = \operatorname{QR}(\bm{\theta}_0),
    \label{eq:client_qr_decomposition}
\end{equation}
and initializes LoRA parameters as:
\begin{equation}
    \mathbf{A}_{k} = \mathbf{R}_{k,:r_k,:}, \quad \mathbf{B}_{k} = \mathbf{Q}_{k,:,:r_k}.
    \label{eq:client_initialization}
\end{equation}
The local model is then initialized:
\begin{equation}
    \bm{\theta}_k = \bm{\theta}_0 - \mathbf{B}_k \mathbf{A}_k,
    \label{eq:client_initial_model}
\end{equation}
ensuring all initial updates $\Delta \bm{\theta}_k = \mathbf{B}_k \mathbf{A}_k$ are confined to consistent subspaces. After receiving updated $\mathbf{B}_k$ and $\mathbf{A}_k$, the client updates its model with $\bm{\theta}_k \leftarrow \bm{\theta}_k + \mathbf{B}_k \mathbf{A}_k$. This subspace coherence stabilizes federated optimization, reducing variance and client drift (Table~\ref{tab:initialization_comparison}, Appendix~\ref{app:qr_aggregation}).
\subsection{Concatenated QR Aggregation}
\label{sec:ILoRA-ch2}

To address aggregation bias in heterogeneous-rank settings, ILoRA reconstructs the global update before compression. For sampled clients $\mathcal{S}_t$ with $p_k = n_k/N$, we vertically concatenate weighted $\mathbf{A}$ and horizontally $\mathbf{B}$ matrices:
\begin{equation}
    \mathbf{A}_{\text{c}}\! = \begin{bmatrix}
        p_1\mathbf{A}_1 \!\\
        p_2\mathbf{A}_2 \!\\
        \vdots \!\\
p_{|\mathcal{S}_t|}\mathbf{A}_{|\mathcal{S}_t|} 
    \end{bmatrix}, \!\quad
    \mathbf{B}_{\text{c}} = \begin{bmatrix}
        \mathbf{B}_1 \!& \mathbf{B}_2 \!& \cdots \!& \mathbf{B}_{|\mathcal{S}_t|} \!
    \end{bmatrix},
    \label{eq:qr_aggregation_stack}
\end{equation}
forming $\Delta\bm{\theta} = \mathbf{B}_{\text{c}} \mathbf{A}_{\text{c}} = \sum_{k \in \mathcal{S}_t} p_k \mathbf{B}_k \mathbf{A}_k$. Using server rank $r_s$, we compute QR decomposition:
\begin{equation}
    \mathbf{Q},\mathbf{R} = \operatorname{QR}(\Delta\bm{\theta})
    \label{eq:qr_decomposition}
\end{equation}
\begin{equation}
    \mathbf{B}_s = \mathbf{Q}_{:,:r_s}, \quad \mathbf{A}_s = \mathbf{R}_{:r_s,:}
    \label{eq:global_parameters}
\end{equation}
\begin{equation}
    \bm{\theta}^{(t)} = \bm{\theta}_0 + \mathbf{B}_s \mathbf{A}_s.
    \label{eq:global_model_update}
\end{equation}
Each client receives personalized slices $\mathbf{B}_{k} = \mathbf{Q}_{:,:r_k}$, $\mathbf{A}_{k} = \mathbf{R}_{:r_k,:}$, ensuring subspace alignment with $\mathcal{O}(r_s \cdot \max(d,k))$ communication cost.

\subsection{AdamW with rank-aware control variates}

To mitigate client drift in Non-IID settings, we employ client-specific control variates $\mathbf{c}_{A,k}$ and $\mathbf{c}_{B,k}$ for $\mathbf{A}_k$ and $\mathbf{B}_k$, initialized to zero. Server $\mathbf{c}^{(t-1)}_A$ and $\mathbf{c}^{(t-1)}_B$ are broadcast each round. Clients compute corrected gradients:
\begin{align}
    \tilde{\mathbf{g}}_{A,k} &= \nabla_{\mathbf{A}_k}\mathcal{L}_k + (\mathbf{c}^{(t-1)}_A - \mathbf{c}_{A,k}) \label{eq:corrected_gradient_A}; \\
    \tilde{\mathbf{g}}_{B,k} &= \nabla_{\mathbf{B}_k}\mathcal{L}_k + (\mathbf{c}^{(t-1)}_B - \mathbf{c}_{B,k}),\label{eq:corrected_gradient_B}
\end{align}
optimizing $\mathbf{A}_k$ and $\mathbf{B}_k$ via AdamW. Post-epoch:
\begin{align}
    \Delta\mathbf{c}_{A,k} &= \nabla_{\mathbf{A}_k}\mathcal{L}_k - \mathbf{c}_{A,k}, \quad 
    \mathbf{c}_{A,k} \leftarrow \nabla_{\mathbf{A}_k}\mathcal{L}_k \label{eq:control_update_A}; \\
    \Delta\mathbf{c}_{B,k} &= \nabla_{\mathbf{B}_k}\mathcal{L}_k - \mathbf{c}_{B,k}, \quad 
    \mathbf{c}_{B,k} \leftarrow \nabla_{\mathbf{B}_k}\mathcal{L}_k .\label{eq:control_update_B}
\end{align}
Server aggregates deltas:
\begin{align}
    \mathbf{c}^{(t)}_A &\leftarrow \mathbf{c}^{(t-1)}_A + \frac{1}{|\mathcal{S}_t|}\sum_{k\in\mathcal{S}_t}\Delta\mathbf{c}_{A,k}; \label{eq:update_global_control_A} \\
    \mathbf{c}^{(t)}_B &\leftarrow \mathbf{c}^{(t-1)}_B + \frac{1}{|\mathcal{S}_t|}\sum_{k\in\mathcal{S}_t}\Delta\mathbf{c}_{B,k}. \label{eq:update_global_control_B}
\end{align}

\section{Convergence Analysis}
\label{sec:convergence}

\begin{figure*}[htbp]
	\centering
	\begin{subfigure}{0.63\linewidth}
		\centering
		\includegraphics[width=\linewidth]{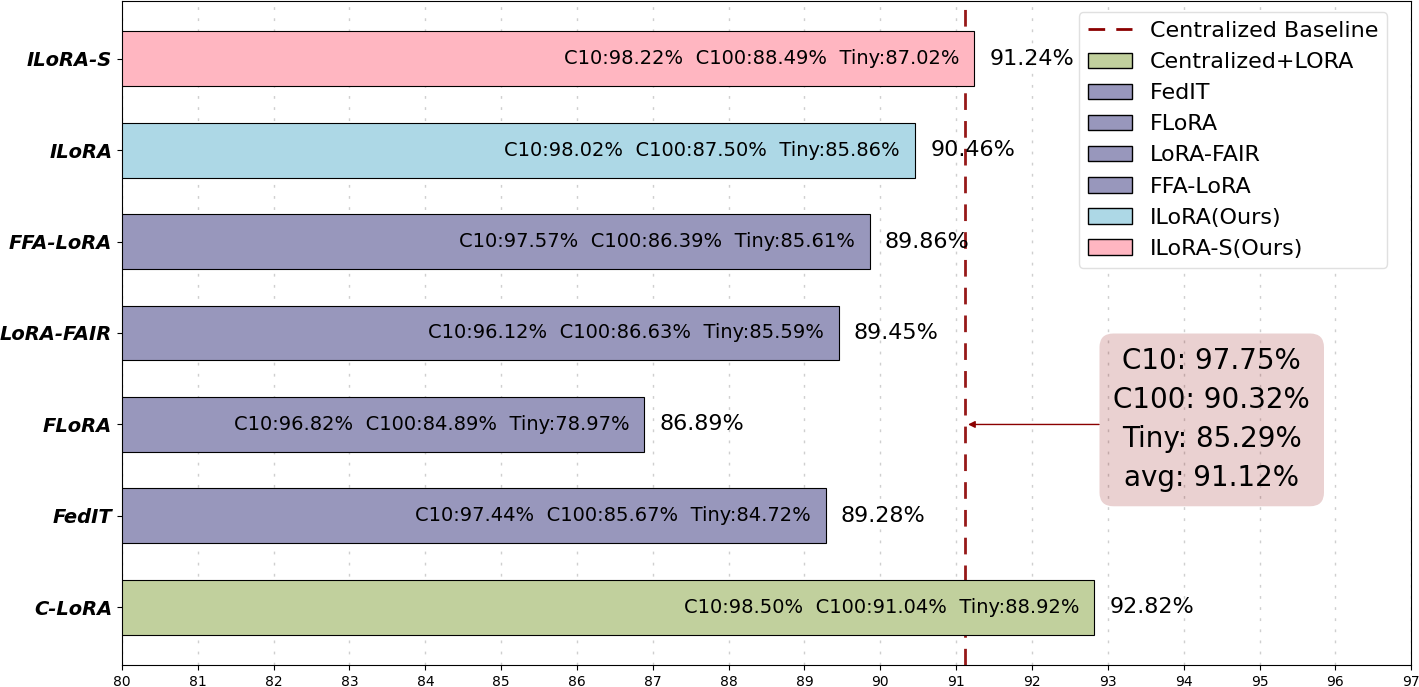}
		\caption{CV performance comparison with AdamW }
		\label{fig:adamw_comparison}
	\end{subfigure}
	\hfill
	\begin{subfigure}{0.35\linewidth}
		\centering
		\includegraphics[width=\linewidth]{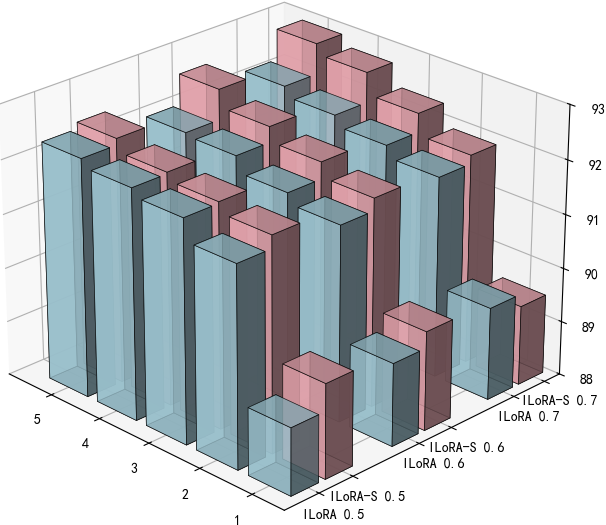}
		\caption{ILoRA vs ILoRA-S with control-based AdamW}
		\label{fig:control_adamw_agnews}
	\end{subfigure}
	\caption{
		Performance comparison under Non-IID settings: (a) Centralized vs. federated learning on CIFAR-10 (C10), CIFAR-100 (C100), and Tiny-ImageNet (Tiny) with $\alpha=0.3$; (b) AGNews dataset across different heterogeneity levels ($\alpha = 0.5, 0.6, 0.7$).
	}
	\vspace{-5pt}
    \label{fig:combined_cv_nlp}
\end{figure*}

\begin{table*}[!ht]
	\centering
\caption{Final accuracy (\%) of federated LoRA methods on DomainNet  with ViT-Base and Swin-Base after 50 rounds. Best in \textbf{bold}.}

	\label{tab:vit_swin_multi_domain}
	\normalsize
	\setlength{\tabcolsep}{2pt} 
	\renewcommand{\arraystretch}{1}
	
	\begin{tabularx}{\textwidth}{l l *{7}{C}}
		\arrayrulecolor{black}\specialrule{1.2pt}{0pt}{0pt}
		
		\multirow{3}{*}{Backbone} & \multirow{3}{*}{Method} & \multicolumn{7}{c}{Datasets} \\
		\cmidrule(lr){3-9}
		& & Clipart & Infograph & Painting & Quickdraw & Real & Sketch & Average \\
		\midrule
		
		\multirow{6}{*}{ViT-Base} 
		& FedIT      
		& \cellnum{78.25}{\ppbase}
		& \cellnum{51.08}{\ppbase}
		& \cellnum{80.07}{\ppbase}
		& \cellnum{63.94}{\ppbase}
		& \cellnum{90.54}{\ppbase}
		& \cellnum{74.24}{\ppbase}
		& \cellnum{73.02}{\ppbase} \\
		
		& FLoRA      
		& \cellnum{73.80}{\ppdown{4.45}}
		& \cellnum{48.34}{\ppdown{2.74}}
		& \cellnum{77.52}{\ppdown{2.55}}
		& \cellnum{55.20}{\ppdown{8.74}}
		& \cellnum{89.55}{\ppdown{0.99}}
		& \cellnum{68.55}{\ppdown{5.69}}
		& \cellnum{68.83}{\ppdown{4.19}} \\
		
		& LoRA-FAIR  
		& \cellnum{78.49}{\ppup{0.24}}
		& \cellnum{51.18}{\ppup{0.10}}
		& \cellnum{79.34}{\ppdown{0.73}}
		& \cellnum{62.87}{\ppdown{1.07}}
		& \cellnum{90.56}{\ppup{0.02}}
		& \cellnum{74.39}{\ppup{0.15}}
		& \cellnum{72.81}{\ppdown{0.21}} \\
		
		& FFA-LoRA   
		& \cellnum{78.52}{\ppup{0.27}}
		& \cellnum{49.68}{\ppdown{1.40}}
		& \cellnum{78.27}{\ppdown{1.80}}
		& \cellnum{63.88}{\ppdown{0.06}}
		& \cellnum{90.30}{\ppdown{0.24}}
		& \cellnum{71.37}{\ppdown{2.87}}
		& \cellnum{72.00}{\ppdown{1.02}} \\
		
		\rowcolor{LightBlue!80!white} 
		& ILoRA   
		& \cellnum{78.69}{\ppup{0.44}}
		& \cellnum{51.44}{\ppup{0.36}}
		& \cellnum{80.09}{\ppup{0.02}}
		& \cellnum{64.18}{\ppup{0.24}}
		& \cellnum{90.71}{\ppup{0.17}}
		& \cellnum{74.78}{\ppup{0.54}}
		& \cellnum{73.32}{\ppup{0.30}} \\
		
		\rowcolor{LightRed!80!white}
		& ILoRA-S  
		& \cellnum{\textbf{79.83}}{\ppup{1.58}}
		& \cellnum{\textbf{52.15}}{\ppup{1.07}}
		& \cellnum{\textbf{80.25}}{\ppup{0.18}}
		& \cellnum{\textbf{65.02}}{\ppup{1.08}}
		& \cellnum{\textbf{90.96}}{\ppup{0.42}}
		& \cellnum{\textbf{75.18}}{\ppup{0.94}}
		& \cellnum{\textbf{73.90}}{\ppup{0.88}} \\
		\midrule
		
		\multirow{6}{*}{Swin-Base} 
		& FedIT      
		& \cellnum{86.28}{\ppbase}
		& \cellnum{61.52}{\ppbase}
		& \cellnum{84.70}{\ppbase}
		& \cellnum{73.02}{\ppbase}
		& \cellnum{92.64}{\ppbase}
		& \cellnum{84.41}{\ppbase}
		& \cellnum{80.43}{\ppbase} \\
		
		& FLoRA      
		& \cellnum{85.58}{\ppdown{0.70}}
		& \cellnum{58.16}{\ppdown{3.36}}
		& \cellnum{84.54}{\ppdown{0.16}}
		& \cellnum{70.42}{\ppdown{2.60}}
		& \cellnum{92.11}{\ppdown{0.53}}
		& \cellnum{81.25}{\ppdown{3.16}}
		& \cellnum{78.68}{\ppdown{1.75}} \\
		
		& LoRA-FAIR  
		& \cellnum{86.85}{\ppup{0.57}}
		& \cellnum{61.89}{\ppup{0.37}}
		& \cellnum{84.17}{\ppdown{0.53}}
		& \cellnum{72.66}{\ppdown{0.36}}
		& \cellnum{92.07}{\ppdown{0.57}}
		& \cellnum{84.75}{\ppup{0.34}}
		& \cellnum{80.40}{\ppdown{0.03}} \\
		
		& FFA-LoRA   
		& \cellnum{87.19}{\ppup{0.91}}
		& \cellnum{61.29}{\ppdown{0.23}}
		& \cellnum{84.05}{\ppdown{0.65}}
		& \cellnum{71.71}{\ppdown{1.31}}
		& \cellnum{91.88}{\ppdown{0.76}}
		& \cellnum{84.04}{\ppdown{0.37}}
		& \cellnum{80.03}{\ppdown{0.40}} \\
		
		\rowcolor{LightBlue!80!white}
		& ILoRA   
		& \cellnum{87.32}{\ppup{1.04}}
		& \cellnum{61.14}{\ppdown{0.38}}
		& \cellnum{85.10}{\ppup{0.40}}
		& \cellnum{73.38}{\ppup{0.36}}
		& \cellnum{92.70}{\ppup{0.06}}
		& \cellnum{84.90}{\ppup{0.49}}
		& \cellnum{80.76}{\ppup{0.33}} \\
		
		\rowcolor{LightRed!80!white}
		& ILoRA-S  
		& \cellnum{\textbf{88.46}}{\ppup{2.18}}
		& \cellnum{\textbf{62.11}}{\ppup{0.59}}
		& \cellnum{\textbf{86.13}}{\ppup{1.43}}
		& \cellnum{\textbf{75.39}}{\ppup{2.37}}
		& \cellnum{\textbf{92.75}}{\ppup{0.11}}
		& \cellnum{\textbf{85.87}}{\ppup{1.46}}
		& \cellnum{\textbf{81.79}}{\ppup{1.36}} \\
		
		\arrayrulecolor{black}\specialrule{1.2pt}{0pt}{0pt}
	\end{tabularx}
    \vspace{-10pt}
\end{table*}
\begin{figure*}[htbp]
	\centering
	\begin{subfigure}{0.62\linewidth}
		\centering
		\includegraphics[width=\linewidth]{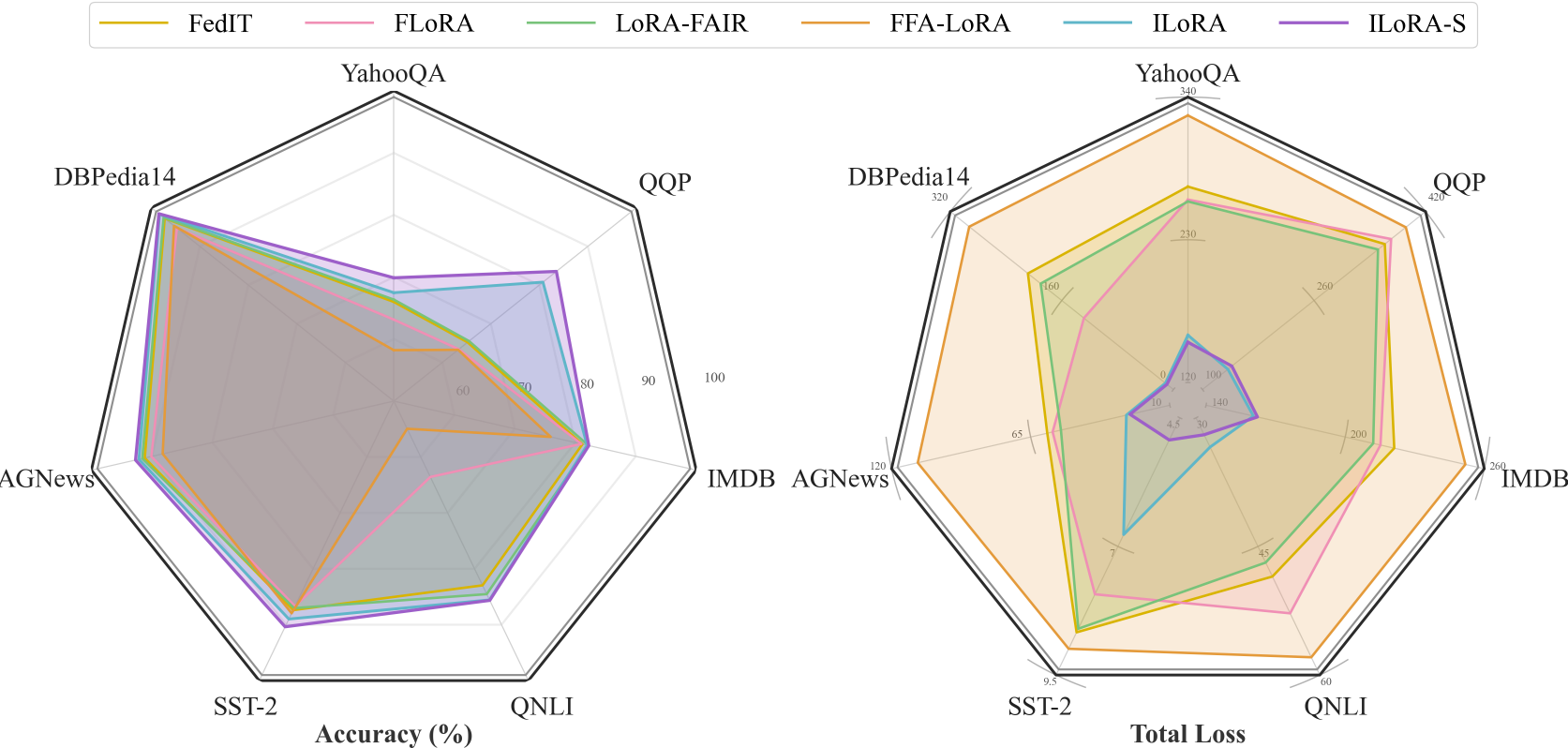}
		\caption{NLP performance comparison}
		\label{fig:nlp_radar_comparison}
	\end{subfigure}
	\hfill
	\begin{subfigure}{0.37\linewidth}
		\centering
		\includegraphics[width=\linewidth]{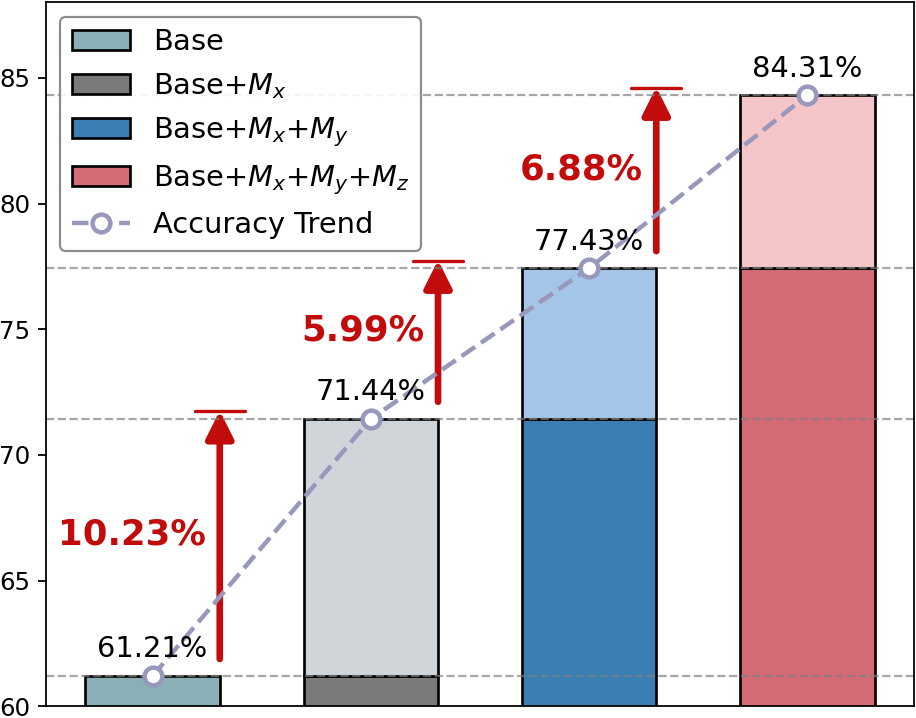}
		\caption{Ablation study results}
		\label{fig:ablation_study}
	\end{subfigure}
\caption{(a) NLP accuracy on datasets with RoBERTa under Dir($\alpha$=0.6). (b) Ablation: gains from ILoRA components $M_x$, $M_y$, $M_z$.}

	\vspace{-5pt}\label{fig:combined_results}
\end{figure*}
\begin{table*}[t]
	\centering
\caption{Peak accuracy comparison (\%) of federated heterogeneous LoRA methods on NLP datasets (RoBERTa, Dir($\alpha$=0.6)).}

	\label{tab:nlp_lora_comparison}
	\normalsize
	\setlength{\tabcolsep}{1pt}
	\renewcommand{\arraystretch}{1}
	
	\begin{tabularx}{\textwidth}{l*{8}{C}}
		\arrayrulecolor{black}\specialrule{1.2pt}{0pt}{0pt}
		\multirow{2}{*}{Method} & \multicolumn{8}{c}{NLP Datasets (RoBERTa, Accuracy $\uparrow$)} \\
		\cmidrule(lr){2-9}
		& YahooQA & QQP & IMDB & QNLI & SST-2 & AGNews & DBPedia14 & \textbf{Avg.} \\
		\midrule
		FedIT
		& \cellnum{65.99}{\ppbase}
		& \cellnum{65.20}{\ppbase}
		& \cellnum{81.19}{\ppbase}
		& \cellnum{82.96}{\ppbase}
		& \cellnum{87.39}{\ppbase}
		& \cellnum{91.18}{\ppbase}
		& \cellnum{97.23}{\ppbase}
		& \cellnum{81.59}{\ppbase} \\
		FLoRA
		& \cellnum{63.13}{\ppdown{2.86}}
		& \cellnum{63.40}{\ppdown{1.80}}
		& \cellnum{81.01}{\ppdown{0.18}}
		& \cellnum{63.57}{\ppdown{19.39}}
		& \cellnum{86.70}{\ppdown{0.69}}
		& \cellnum{90.09}{\ppdown{1.09}}
		& \cellnum{94.71}{\ppdown{2.52}}
		& \cellnum{77.52}{\ppdown{4.07}} \\
		LoRA-FAIR
		& \cellnum{66.36}{\ppup{0.37}}
		& \cellnum{65.51}{\ppup{0.31}}
		& \cellnum{\textbf{82.19}}{\ppup{1.00}}
		& \cellnum{84.53}{\ppup{1.57}}
		& \cellnum{87.04}{\ppdown{0.35}}
		& \cellnum{91.59}{\ppup{0.41}}
		& \cellnum{97.46}{\ppup{0.23}}
		& \cellnum{82.10}{\ppup{0.51}} \\
		FFA-LoRA
		& \cellnum{58.21}{\ppdown{7.78}}
		& \cellnum{63.26}{\ppdown{1.94}}
		& \cellnum{75.94}{\ppdown{5.25}}
		& \cellnum{54.95}{\ppdown{28.01}}
		& \cellnum{87.96}{\ppup{0.57}}
		& \cellnum{88.21}{\ppdown{2.97}}
		& \cellnum{95.32}{\ppdown{1.91}}
		& \cellnum{74.84}{\ppdown{6.75}} \\
		\rowcolor{LightBlue}
		ILoRA
		& \cellnum{67.47}{\ppup{1.48}}
		& \cellnum{80.75}{\ppup{15.55}}
		& \cellnum{81.96}{\ppup{0.77}}
		& \cellnum{85.56}{\ppup{2.60}}
		& \cellnum{88.99}{\ppup{1.60}}
		& \cellnum{92.14}{\ppup{0.96}}
		& \cellnum{\textbf{98.30}}{\ppup{1.07}}
		& \cellnum{85.02}{\ppup{3.43}} \\
		\rowcolor{LightRed}
		ILoRA-S
		& \cellnum{\textbf{69.87}}{\ppup{3.88}}
		& \cellnum{\textbf{83.52}}{\ppup{18.32}}
		& \cellnum{82.24}{\ppup{1.05}}
		& \cellnum{\textbf{85.65}}{\ppup{2.69}}
		& \cellnum{\textbf{90.37}}{\ppup{2.98}}
		& \cellnum{\textbf{92.70}}{\ppup{1.52}}
		& \cellnum{98.43}{\ppup{1.20}}
		& \cellnum{\textbf{86.11}}{\ppup{4.52}} \\
		\arrayrulecolor{black}\specialrule{1.2pt}{0pt}{0pt}
	\end{tabularx}
    \vspace{-10pt}
\end{table*}
\subsection{Theoretical Assumptions}
Our analysis relies on the following standard assumptions:

\begin{assumption}[L-Smoothness]
    \label{ass:smoothness}
    $\|\nabla F_k(\bm{\theta}_1) - \nabla F_k(\bm{\theta}_2)\| \leq L \|\bm{\theta}_1 - \bm{\theta}_2\|$ holds uniformly across all participating clients.
\end{assumption}

\begin{assumption}[Bounded Gradient Variance]
    \label{ass:variance}
    $\mathbb{E}[\|\mathbf{g}_k(\bm{\theta}) - \nabla F_k(\bm{\theta})\|^2] \leq \sigma^2$ with uniform bound across the federation.
\end{assumption}

\begin{assumption}[Bounded Gradient Heterogeneity]
    \label{ass:heterogeneity}
    $\|\nabla F_k(\bm{\theta}) - \nabla F(\bm{\theta})\| \leq \delta$ quantifying the degree of data heterogeneity among clients in the federated network.
\end{assumption}

\begin{assumption}[Subspace Consistency]
    \label{ass:subspace}
    $\mathcal{S}_k \subseteq \mathcal{S}_s = \text{colspan}(\mathbf{Q}_{[:,1:r_s]})$ ensuring aligned parameter spaces.
\end{assumption}
\subsection{Main Convergence Results}

\begin{theorem}[Convergence of ILoRA]
    \label{thm:convergence}
    Under Assumptions \ref{ass:smoothness}--\ref{ass:subspace}, with $\eta_l \!\leq\! 1/L$ and $\eta_g \eta_l \!=\! \Theta(1/\!\sqrt{\!SKT})$:
    \begin{equation}
        \frac{1}{T}\!\sum_{t=1}^T \mathbb{E}\left[\|\nabla F(\bm{\theta}_t)\|^2\right] 
        \leq \mathcal{O}\!\left(\!
        \frac{1}{\sqrt{\!SKT}} + \frac{\delta^2 \!+\! \epsilon_r}{T}
        \!\right)\!,
    \end{equation}
    where $\epsilon_r \!=\! (r_{\max} \!-\! r_s)^2$, $S$: clients/round, $K$: local steps, $T$: total rounds.
\end{theorem}

\begin{proof}[Proof Sketch]
The proof uses QR-based low-rank reconstruction, subspace alignment through orthogonal initialization, and control variates to mitigate client drift, ensuring stability. Details are in Appendix~\ref{app:convergence}.
\end{proof}

\begin{theorem}[Subspace Preservation under QR Compression]
\label{thm:subspace_preservation}
For $r_k \leq r_s$, $\mathbf{B}_{r_k} = \mathbf{Q}_{:,:r_k}$ and $\mathbf{A}_{r_k} = \mathbf{R}_{:r_k,:}$ satisfy $\mathrm{colspan}(\mathbf{B}_{r_k}) \subseteq \mathrm{colspan}(\mathbf{Q})$.
\end{theorem}

\begin{theorem}[Consistent Subspace Initialization]
    \label{thm:consistent_subspace}
    $\mathbf{B}_k^{(0)} = \mathbf{Q}_{:,:r_k}$ and $\mathbf{A}_k^{(0)} = \mathbf{R}_{:r_k,:}$ ensure all clients start in $\mathrm{colspan}(\mathbf{Q})$, reducing gradient variance by $\mathcal{O}(1/K)$.
\end{theorem}

\begin{theorem}[Convergence with Control Variates and AdamW]
    \label{thm:control_adamw_convergence}
    Under Assumptions \ref{ass:smoothness}--\ref{ass:subspace}, with control variates reducing variance to $\sigma^2(1\!-\!\rho^2)$:
    \begin{equation}
        \frac{1}{T}\!\sum_{t=1}^T \mathbb{E}\left[\|\nabla F(\bm{\theta}_t)\|^2\right] 
        \leq \mathcal{O}\!\left(\!
        \frac{1}{\sqrt{\!SKT}} + \frac{\delta^2 \!+\! \epsilon_r'}{T}
        \!\right)\!,
    \end{equation}
    where $\epsilon_r' \!=\! (r_s \!-\! \bar{r})^2/r_s^2$, $\bar{r}$: mean client rank.
\end{theorem}

\begin{proof}[Proof Sketch]  
The proof shows that control variates reduce client drift by correcting gradients $\tilde{\mathbf{g}}_k = \mathbf{g}_k + (\mathbf{c} - \mathbf{c}_k)$, improving convergence. Details are in Appendix~\ref{app:control_variates}.
\end{proof}

\textbf{ILoRA} provides strong theoretical guarantees addressing three core federated LoRA challenges: \textbf{rank robustness} via QR aggregation (Theorem~\ref{thm:subspace_preservation}, Appendix~\ref{app:qr_aggregation}), \textbf{Non-IID resilience} through control variates (Theorem~\ref{thm:control_adamw_convergence}, Appendix~\ref{app:control_variates}), and \textbf{communication efficiency} with $\mathcal{O}(r_s \cdot \max(d,k))$ cost (Theorem~\ref{thm:total_communication}, Appendix~\ref{app:communication}).

Theoretical analysis demonstrates $\mathcal{O}(1/\sqrt{SKT})$ convergence (Theorems~\ref{thm:convergence},~\ref{thm:control_adamw_convergence}, Appendix~\ref{app:convergence},~\ref{app:control_variates}) with bounded heterogeneity effects, global stability (Theorem~\ref{thm:comprehensive_guarantee}, Appendix~\ref{app:communication}), and $\mathcal{O}(1/K)$ variance reduction via orthogonal initialization (Theorem~\ref{thm:consistent_subspace}, Lemma~\ref{lem:gradient_variance}, Appendix~\ref{app:orthogonal_init}). Rank-aware control variates (Lemmas~\ref{lem:variance_reduction},~\ref{lem:rank_aligned_error}, Theorem~\ref{thm:rank_compatibility}, Appendix~\ref{app:control_variates}) ensure robustness to data and rank heterogeneity. Empirical validation (Tables~\ref{tab:initialization_comparison},~\ref{tab:aggregation_comparison},~\ref{tab:comprehensive_guarantees},~\ref{tab:control_comparison},~\ref{tab:communication_comparison}, Appendix~\ref{app:orthogonal_init},~\ref{app:qr_aggregation},~\ref{app:communication},~\ref{app:control_variates}) confirms \textbf{ILoRA}'s balanced performance across drift mitigation, efficiency, and accuracy.

\subsection{Experimental Setup}

\textbf{Models and Datasets.} ILoRA is evaluated on CV (ViT-Base \cite{dosovitskiy2020image}, Swin-Base \cite{liu2021swin}) and NLP (RoBERTa \cite{liu2019roberta}) benchmarks. CV datasets: CIFAR-10/100 \cite{krizhevsky2009learning}, Tiny-ImageNet \cite{hendrycks2019benchmarking}, DomainNet \cite{peng2019moment}. NLP: YahooQA \cite{kucuktunc2012large}, QQP, QNLI \cite{wang2018glue}, IMDB \cite{maas2011learning}, SST-2 \cite{socher2013recursive}, AGNews and DBPedia14 \cite{zhang2015character}. All use Non-IID Dirichlet partitioning.

\textbf{Baselines and Settings.} Compared methods: FedIT \cite{li2020federated}, FLoRA \cite{wang2024flora}, LoRA-FAIR \cite{bian2024lora}, FFA-LoRA \cite{sun2024improving}. Unified LoRA configurations and optimization protocols are used across all experiments.

Additional implementation details including hyperparameter settings, hardware configurations, and computational resources are provided in Appendix~\ref{app:implementation_setup}.
\begin{table*}[t]
	\centering
\caption{Peak accuracy comparison (\%) under different federation scales ($\alpha=0.5$). ``Setting'' denotes client count and participation rate}

	\label{tab:client_scalability}
	\normalsize
	\setlength{\tabcolsep}{1pt}
	\renewcommand{\arraystretch}{1}
	
	\begin{tabularx}{\textwidth}{l c *{7}{C} C}
		\arrayrulecolor{black}\specialrule{1.2pt}{0pt}{0pt}
		\multirow{2}{*}{Method} & \multirow{2}{*}{Setting} & \multicolumn{7}{c}{Datasets (Accuracy $\uparrow$)} & \multirow{2}{*}{\textbf{Avg.}} \\
		\cmidrule(lr){3-9}
		& & YahooQA & QQP & IMDB & QNLI & SST-2 & AGNews & DBPedia14 & \\
		\midrule
		
		FedIT+QR & 3 (100\%) 
		& \cellnum{40.73}{\ppbase} 
		& \cellnum{63.18}{\ppbase} 
		& \cellnum{58.09}{\ppbase} 
		& \cellnum{64.43}{\ppbase} 
		& \cellnum{71.90}{\ppbase} 
		& \cellnum{87.04}{\ppbase} 
		& \cellnum{74.31}{\ppbase} 
		& \cellnum{65.67}{\ppbase} \\
		
		\rowcolor{LightBlue!80!white}
		ILoRA & 3 (100\%) 
		& \cellnum{\textbf{63.76}}{\ppup{23.0}} 
		& \cellnum{\textbf{63.66}}{\ppup{0.5}} 
		& \cellnum{\textbf{78.26}}{\ppup{20.2}} 
		& \cellnum{\textbf{83.51}}{\ppup{19.1}} 
		& \cellnum{\textbf{82.57}}{\ppup{10.7}} 
		& \cellnum{\textbf{91.68}}{\ppup{4.6}} 
		& \cellnum{\textbf{97.62}}{\ppup{23.3}} 
		& \cellnum{\textbf{80.15}}{\ppup{14.5}} \\
		
		FedIT+QR & 15 (80\%) 
		& \cellnum{31.88}{\ppbase} 
		& \cellnum{63.18}{\ppbase} 
		& \cellnum{74.78}{\ppbase} 
		& \cellnum{50.58}{\ppbase} 
		& \cellnum{89.79}{\ppbase} 
		& \cellnum{63.11}{\ppbase} 
		& \cellnum{77.39}{\ppbase} 
		& \cellnum{64.39}{\ppbase} \\
		
		\rowcolor{LightBlue!80!white}
		ILoRA & 15 (80\%) 
		& \cellnum{\textbf{65.70}}{\ppup{33.8}} 
		& \cellnum{\textbf{77.58}}{\ppup{14.4}} 
		& \cellnum{\textbf{76.88}}{\ppup{2.1}} 
		& \cellnum{\textbf{62.11}}{\ppup{11.5}} 
		& \cellnum{\textbf{89.79}}{\ppup{0.0}} 
		& \cellnum{\textbf{88.17}}{\ppup{25.1}} 
		& \cellnum{\textbf{94.66}}{\ppup{17.3}} 
		& \cellnum{\textbf{79.27}}{\ppup{14.9}} \\
		
		\arrayrulecolor{black}\specialrule{1.2pt}{0pt}{0pt}
	\end{tabularx}
    \vspace{-5pt}
\end{table*}
\begin{table*}[t]
	\centering
\caption{Peak accuracy comparison (\%) under extreme settings (Dir($\alpha$=0.1/0.5/1.0), ranks = 2/8/16).}

	\label{tab:extreme_heterogeneity}
	\normalsize
	\setlength{\tabcolsep}{1pt}
	\renewcommand{\arraystretch}{1}
	
	\begin{tabularx}{\textwidth}{l c *{7}{C} C}
		\arrayrulecolor{black}\specialrule{1.2pt}{0pt}{0pt}
		\multirow{2}{*}{Method} & \multirow{2}{*}{Dir($\alpha$)} & \multicolumn{7}{c}{Datasets (Accuracy $\uparrow$)} & \multirow{2}{*}{\textbf{Avg.}} \\
		\cmidrule(lr){3-9}
		& & YahooQA & QQP & IMDB & QNLI & SST-2 & AGNews & DBPedia14 & \\
		\midrule
		
		FedIT+QR & 0.1 
		& \cellnum{34.07}{\ppbase} 
		& \cellnum{85.36}{\ppbase} 
		& \cellnum{52.40}{\ppbase} 
		& \cellnum{50.54}{\ppbase} 
		& \cellnum{50.92}{\ppbase} 
		& \cellnum{70.75}{\ppbase} 
		& \cellnum{44.06}{\ppbase} 
		& \cellnum{55.44}{\ppbase} \\
		
		\rowcolor{LightBlue!80!white}
		ILoRA & 0.1 
		& \cellnum{\textbf{51.05}}{\ppup{17.0}} 
		& \cellnum{\textbf{84.22}}{\ppdown{1.1}} 
		& \cellnum{\textbf{50.14}}{\ppdown{2.3}} 
		& \cellnum{\textbf{50.54}}{\ppup{0.0}} 
		& \cellnum{\textbf{54.82}}{\ppup{3.9}} 
		& \cellnum{\textbf{80.75}}{\ppup{10.0}} 
		& \cellnum{\textbf{87.64}}{\ppup{43.6}} 
		& \cellnum{\textbf{65.59}}{\ppup{10.2}} \\
		
		FedIT+QR & 0.5 
		& \cellnum{40.73}{\ppbase} 
		& \cellnum{63.18}{\ppbase} 
		& \cellnum{58.09}{\ppbase} 
		& \cellnum{64.43}{\ppbase} 
		& \cellnum{71.90}{\ppbase} 
		& \cellnum{87.04}{\ppbase} 
		& \cellnum{74.31}{\ppbase} 
		& \cellnum{65.67}{\ppbase} \\
		
		\rowcolor{LightBlue!80!white}
		ILoRA & 0.5 
		& \cellnum{\textbf{63.76}}{\ppup{23.0}} 
		& \cellnum{\textbf{63.66}}{\ppup{0.5}} 
		& \cellnum{\textbf{78.26}}{\ppup{20.2}} 
		& \cellnum{\textbf{83.51}}{\ppup{19.1}} 
		& \cellnum{\textbf{82.57}}{\ppup{10.7}} 
		& \cellnum{\textbf{91.68}}{\ppup{4.6}} 
		& \cellnum{\textbf{97.62}}{\ppup{23.3}} 
		& \cellnum{\textbf{80.15}}{\ppup{14.5}} \\
		
		FedIT+QR & 1.0 
		& \cellnum{60.98}{\ppbase} 
		& \cellnum{63.23}{\ppbase} 
		& \cellnum{81.28}{\ppbase} 
		& \cellnum{50.65}{\ppbase} 
		& \cellnum{55.62}{\ppbase} 
		& \cellnum{80.92}{\ppbase} 
		& \cellnum{75.38}{\ppbase} 
		& \cellnum{66.87}{\ppbase} \\
		
		\rowcolor{LightBlue!80!white}
		ILoRA & 1.0 
		& \cellnum{\textbf{70.90}}{\ppup{9.9}} 
		& \cellnum{\textbf{64.31}}{\ppup{1.1}} 
		& \cellnum{\textbf{81.73}}{\ppup{0.5}} 
		& \cellnum{\textbf{82.59}}{\ppup{31.9}} 
		& \cellnum{\textbf{88.88}}{\ppup{33.3}} 
		& \cellnum{\textbf{90.57}}{\ppup{9.7}} 
		& \cellnum{\textbf{97.19}}{\ppup{21.8}} 
		& \cellnum{\textbf{82.31}}{\ppup{15.4}} \\
		
		\arrayrulecolor{black}\specialrule{1.2pt}{0pt}{0pt}
	\end{tabularx}
    \vspace{-10pt}
\end{table*}

\begin{table}[htbp]
    \centering
    \caption{Control variates mitigate client drift on Tiny-ImageNet; ILoRA-S outperforms ILoRA. \textbf{Bold} indicates best performance.}

    \label{tab:tinyimagenet_results}
    \renewcommand{\arraystretch}{1}
    \begin{tabular}{lccc}
        \toprule
        \textbf{Method} & $\alpha=0.5$ & $\alpha=0.6$ & $\alpha=0.7$ \\
        \midrule
        ILoRA & 85.14$_{\pm0.34}$ & 85.90$_{\pm0.51}$ & 85.93$_{\pm0.20}$ \\
        ILoRA-S & \textbf{86.03}$_{\pm0.32}$ & \textbf{86.29}$_{\pm0.19}$ & \textbf{86.43}$_{\pm0.17}$ \\
        \bottomrule
    \end{tabular}
\end{table}
\subsection{Experiment Results}

\textbf{CV Performance under Heterogeneous Settings.} ILoRA-S reaches 87.51\% on CIFAR-100 with ViT-Base+AdamW (Table~\ref{tab:heterogeneous_lora_comparison}), improving over LoRA-FAIR by 1.10\% and FedIT by 2.32\%. It also performs well with Swin-Base (Figure~\ref{fig:combined_comparison}(d)), showing heterogeneity resilience. On DomainNet (50 rounds, Table~\ref{tab:vit_swin_multi_domain}), ILoRA-S attains averages of 73.90\% (ViT-Base) and 81.79\% (Swin-Base), supporting generalization across domains. Performance curves (Figure~\ref{fig:combined_comparison}(a,c)) show higher accuracy, supporting subspace alignment and reduced drift. Advantages persist in homogeneous settings (Table~\ref{tab:homogeneous_lora_comparison}) and with SGD (Appendix Figures~\ref{fig:app_homog_vit_sgd}-\ref{fig:app_heterog_swin_sgd}).

\textbf{NLP Performance under Heterogeneous Settings.}
With heterogeneous LoRA ranks and Non-IID data (Dir($\alpha$=0.6)), ILoRA achieves 85.02\% accuracy on seven NLP benchmarks (Table~\ref{tab:nlp_lora_comparison}), with ILoRA-S reaching 86.11\% and outperforming all baselines (Figure~\ref{fig:combined_comparison}). ILoRA-S sets SOTA on six of seven datasets, with large gains on QQP (83.52\%, +18.32\%), YahooQA (69.87\%, +3.88\%), and DBPedia14 (98.43\%, +1.20\%), visualized in Fig.~\ref{fig:nlp_heatmap}. Consistent superiority at $\alpha=0.5$ (Table~\ref{tab:nlp_lora_comparison5}, Appendix~\ref{app:theory}) confirms robustness. ILoRA variants show better stability, reducing loss 40--95\% versus FedIT (Figure~\ref{fig:combined_results}a) and validating generalization (Figs.~\ref{fig:nlp_heatmap_05} and \ref{fig:nlp_radar_05}, Appendix~\ref{app:additional_results}).

\textbf{Comparison with Centralized Learning.} ILoRA-S achieves the closest performance to centralized training among federated methods under AdamW (Figure~\ref{fig:adamw_comparison}), recovering over \textbf{95\%} of centralized accuracy on Tiny-ImageNet and significantly narrowing the performance gap. Similar SGD trends (Appendix~\ref{app:additional_results}, Figure~\ref{fig:sgd_comparison}) confirm consistent high recovery rates across datasets and optimizers.

\subsection{Communication Efficiency Analysis}
\label{subsec:communication_efficiency}

ILoRA delivers communication efficiency, with downlink $\mathcal{O}(r_s(d+k))$ and uplink $\mathcal{O}(S\!\cdot\!r_{\max}(d+k))$—below FLoRA's $\mathcal{O}(r_{\text{total}}(d+k))$ (Table~\ref{tab:communication_comparison}, Appendix~\ref{app:communication}). It shows near-constant overhead ($1.2\times$–$1.6\times$ for 10–100 clients) versus FLoRA's linear rise ($2.5\times$–$25.0\times$) (Table~\ref{tab:scaling_comparison}, Appendix~\ref{app:additional_results}). This $\mathcal{O}(1)$ scaling via QR compression makes ILoRA suited for large-scale federated deployments.

\subsection{Ablation Study}
\textbf{Effectiveness of Server-Side Aggregation.}
Our concatenated QR aggregation outperforms FedIT+QR across settings: Table~\ref{tab:adamw_comparison} in Appendix~\ref{app:additional_results} shows 9.08\% improvement at high heterogeneity ($\alpha=0.7$); Table~\ref{tab:extreme_heterogeneity} demonstrates robustness under extreme heterogeneity ($\alpha=0.1,0.5,1.0$) and diverse ranks (2,8,16) with 10.15--15.44\% gains; Table~\ref{tab:client_scalability} confirms scalability to 15-client settings (14.88\% improvement) comparable to 3-client performance (14.48\%). These results validate generalization across heterogeneity levels, parameter configurations, and federation scales.
\\
\textbf{Client Drift Suppression via AdamW Optimization}
ILoRA-S mitigates client drift on CV and NLP. On CIFAR-100, it reaches 89.00\% accuracy with std 0.07--0.33 across $\alpha=0.5$--$0.7$ (Table~\ref{tab:cifar100_results}, Appendix~\ref{app:additional_results}). On Tiny-ImageNet, it maintains 0.9--1.0\% gains over ILoRA (Table~\ref{tab:tinyimagenet_results}). For NLP, ILoRA-S improves AGNews (Table~\ref{tab:agnews_results}, Appendix~\ref{app:additional_results}) to 92.87\% at $\alpha=0.7$ and outperforms ILoRA across heterogeneity levels. These results demonstrate the effectiveness of our control variate mechanism in suppressing client drift in heterogeneous federated settings.
\\
\textbf{Cumulative Benefits of Component Integration}
Ablation studies on QNLI (Figure~\ref{fig:combined_results}b) show complementary gains from ILoRA's core components: orthogonal initialization ($M_x$) aligns subspaces, concatenated aggregation ($M_y$) enables cross-client fusion, and control variates ($M_z$) suppress drift. Cumulative improvements validate our holistic design where each mechanism addresses distinct federated LoRA challenges while synergistically enhancing performance.
\begin{figure}[htbp]
	\centering
	\includegraphics[width=0.95\linewidth]{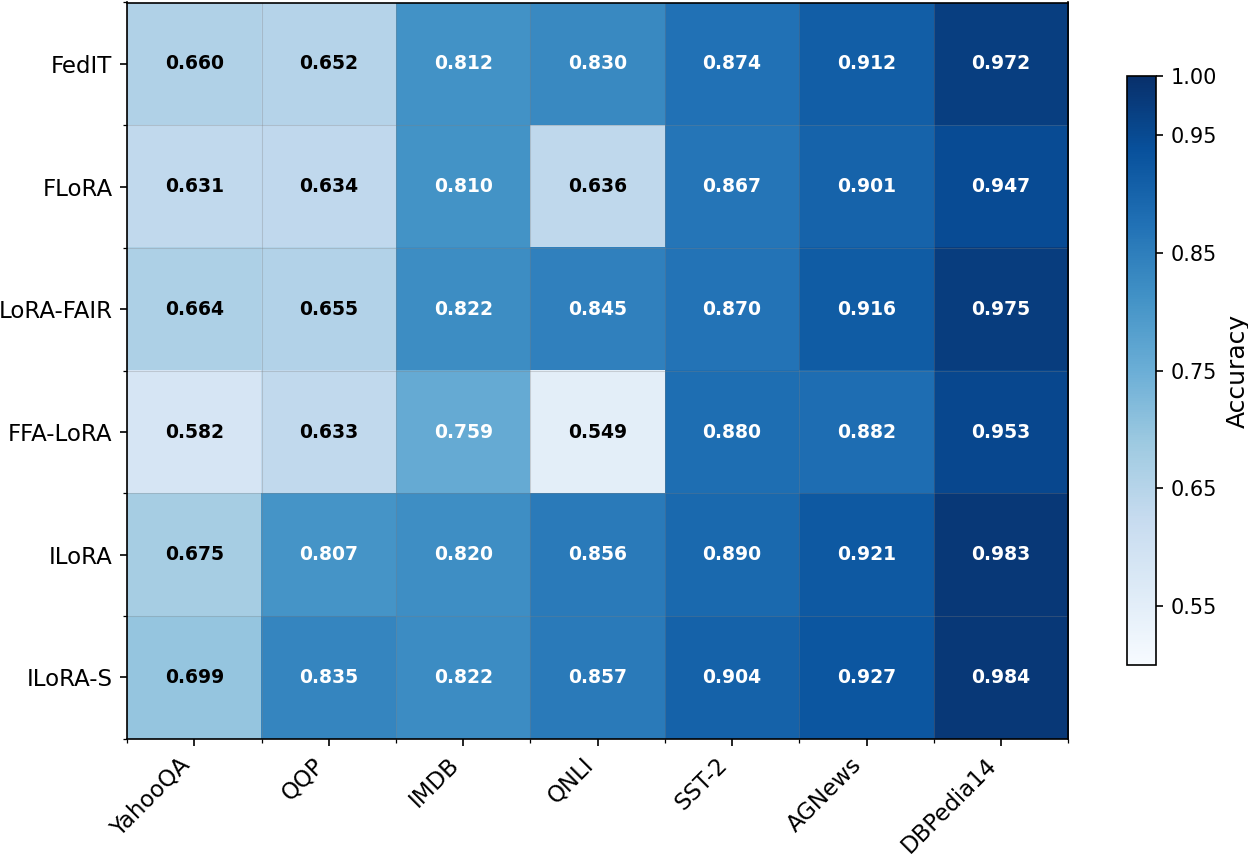}
\caption{NLP accuracy heatmap (RoBERTa, Dir($\alpha$=0.6))}
	\label{fig:nlp_heatmap}
\end{figure}
\section{Conclusion}
\label{sec:conclusion}
In this work, we proposed \textbf{ILoRA} to address the key challenges of initialization instability, rank-incompatible aggregation, and client drift in federated LoRA fine-tuning. Our unified framework systematically integrates three core innovations: QR-based orthogonal initialization for stable subspace alignment, concatenated QR aggregation for exact heterogeneous-rank fusion, and rank-aware control variates for effective drift mitigation. Supported by theoretical convergence guarantees and extensive experiments across diverse vision and NLP benchmarks, ILoRA consistently achieves state-of-the-art performance while maintaining communication efficiency. This work establishes a principled foundation for federated fine-tuning under heterogeneity, with future extensions planned for broader parameter-efficient methods and more constrained federated environments.

%% file: sec/X_suppl.tex
\clearpage
\setcounter{page}{1}
\maketitlesupplementary

\appendix

\section{Implementation and Experimental Setup Details}
\label{app:implementation_setup}
\label{app:expdetails}  

\subsection{Models and Tasks}
We conduct experiments on standard benchmarks in both Computer Vision (CV) and Natural Language Processing (NLP). For CV tasks, we adopt two widely-used Transformer-based architectures: \textit{ViT-Base} (Vision Transformer) and \textit{Swin-Base} (Swin Transformer). In the NLP domain, we utilize \textit{RoBERTa} for text classification and textual entailment tasks. All models are fine-tuned in a federated learning scenario with parameter-efficient \textit{LoRA} adapters.

Our experimental evaluation encompasses standard benchmarks in both computer vision (CV) and natural language processing (NLP). For CV tasks, we employ two prominent Transformer-based architectures: \textit{ViT-Base} (86M parameters) and \textit{Swin-Base} (88M parameters), fine-tuned for image classification. In the NLP domain, we utilize \textit{RoBERTa} (125M parameters) for a diverse set of tasks including text classification, natural language inference, and question-answering. All models are fine-tuned within a federated learning framework using parameter-efficient \textit{LoRA} adapters, ensuring a consistent and comparable experimental setup across all domains and baselines.

\subsection{Baselines}
We compare \textbf{ILoRA} with four state-of-the-art federated LoRA baselines:
\begin{itemize}
	\item \textbf{FedIT}~\cite{zhang2024towards}: a federated instruction-tuning approach that adopts \textit{homogeneous} LoRA ranks.
	\item \textbf{FLoRA}~\cite{wang2024flora}: enables \textit{heterogeneous} LoRA ranks via parameter concatenating\cite{cho2024heterogeneous}.
	\item \textbf{LoRA-FAIR}~\cite{bian2024lora}: refines aggregation and initialization under homogeneous ranks.
	\item \textbf{FFA-LoRA}~\cite{sun2024improving}: improves training stability by freezing a subset of adapter parameters.
	\item \textbf{ILoRA}: our proposed method.
	\item \textbf{ILoRA-S}: an extended variant that incorporates rank-aware control variates.
\end{itemize}
For a fair comparison, all baselines are implemented under the identical LoRA configuration.
\subsection{Hyperparameter Settings}
We adopt a unified configuration across all experiments. LoRA is applied to the self-attention \emph{query} and \emph{value} projections, with the client-side rank fixed at 4. On the server, we use rank 4 in homogeneous settings and rank 6 in heterogeneous settings; FLoRA uses rank 12 in both cases. The LoRA scaling factor is 16, with a dropout rate of 0.1 and a global scaling factor of 0.5. For optimization, we compare stochastic gradient descent (learning rate 0.01, momentum 0.9) with AdamW (learning rate \(1 \times 10^{-4}\), no weight decay). 

Federated training proceeds for 5 global communication rounds, each comprising 1 local epoch per client. In some experiments, federated training is extended to 50 communication rounds to assess long-term performance. In centralized training, we conduct 5 total epochs with 3 local passes per iteration. Mini-batch sizes are 64 examples per client in federated mode and 128 in centralized mode. 

Experiments involve varying numbers of clients, with all settings enforcing full participation in every round. We test under both Independent and Non-Independent Identically Distributed (IID and Non-IID) data scenarios, ensuring robustness across different data distribution patterns. In federated settings, we apply a regularization coefficient of 0.01 when necessary, while weight decay is generally not applied unless specified. The random seed is fixed to 42 to ensure reproducibility, particularly within the field of Natural Language Processing (NLP). Evaluation primarily focuses on classification accuracy, and for NLP tasks, input sequences are truncated to a maximum length of 64 tokens.
\subsection{Compute Resource Usage Summary}
\label{app:compute_summary}

To ensure the robustness and reproducibility of our experimental results, all experiments were conducted on two distinct hardware configurations representing different computational tiers. 

The first server was equipped with 4 NVIDIA GeForce RTX 4090 GPUs (24GB VRAM each) and dual Intel Xeon Silver 4310 CPUs with 48 total cores. The second server featured 8 NVIDIA Tesla V100-PCIE-16GB GPUs and dual Intel Xeon Gold 6240 CPUs with 36 total cores. 

Experiments were distributed across both platforms to validate the consistency of our method under varying hardware conditions. All implementations utilized mixed precision training and distributed data parallelism where applicable. The complete computational details and efficiency metrics are documented in our submitted Compute Reporting Form (CRF).
\subsection{Additional Experimental Results}
\label{app:additional_results}

This section presents supplementary experimental results that further validate the robustness and effectiveness of our proposed ILoRA framework under different experimental settings.

\paragraph{Performance under Different Data Heterogeneity Levels.}
To comprehensively evaluate the generalization capability of our method, we provide additional results with Dir($\alpha$=0.5) in Figures~\ref{fig:nlp_heatmap_05} and~\ref{fig:nlp_radar_05}. These results complement the main text analysis with $\alpha=0.6$ and demonstrate that ILoRA and ILoRA-S maintain consistent performance advantages across varying degrees of data heterogeneity. The $\alpha=0.5$ setting represents a more challenging Non-IID scenario with higher data skewness across clients, yet our methods continue to outperform all baselines, highlighting their robustness to different data distribution patterns.

\begin{figure}[htbp]
	\centering
	\includegraphics[width=0.8\linewidth]{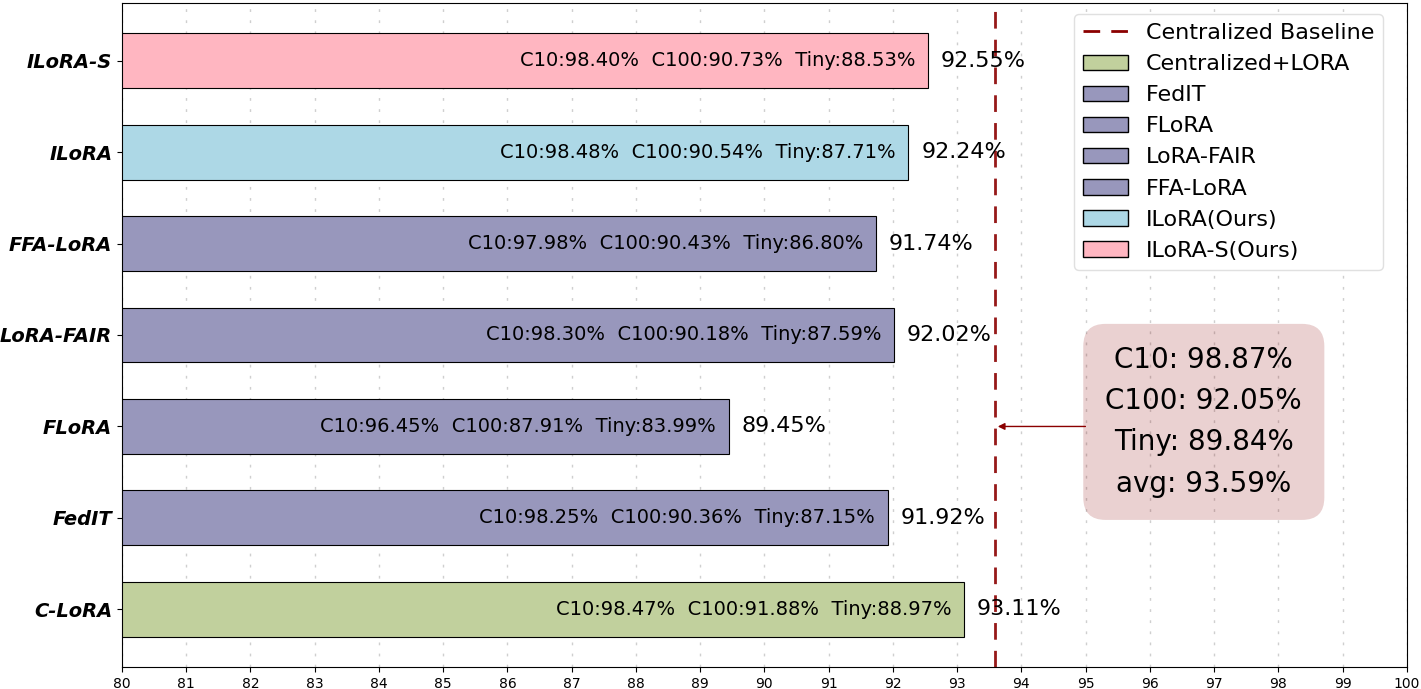}
	\caption{
		Centralized versus federated learning performance comparison using \textbf{SGD} with ViT-Base over 5 communication rounds. 
		Results are reported on three datasets: CIFAR-10 (C10), CIFAR-100 (C100), and Tiny-ImageNet (Tiny), with Non-IID data partitioning (Dir($\alpha$=0.3)).
		ILoRA-S maintains robust performance under SGD optimization, achieving high recovery rates compared to centralized training.
	}
	\label{fig:sgd_comparison}
\end{figure}

\begin{figure}[htbp]
	\centering
	\includegraphics[width=0.95\linewidth]{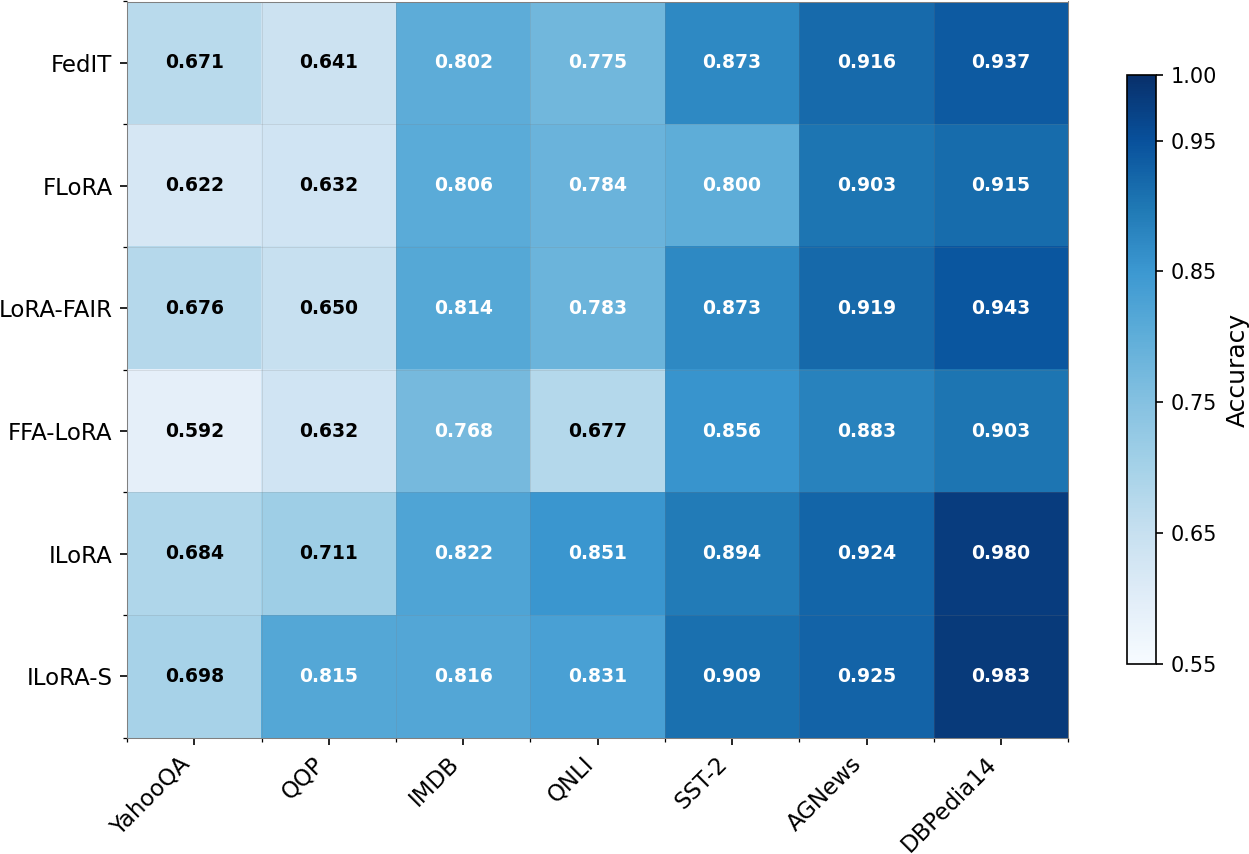}
	\caption{Accuracy Heatmap Comparison of Federated LoRA Methods on NLP Datasets under Non-IID Data Distribution (Dir($\alpha$=0.5)). The heatmap visualizes the performance of different federated learning methods across seven NLP benchmarks using RoBERTa, with color intensity representing accuracy scores from 0.55 to 1.0.}
	\label{fig:nlp_heatmap_05}
\end{figure}

\begin{figure}[htbp]
	\centering
	\includegraphics[width=0.95\linewidth]{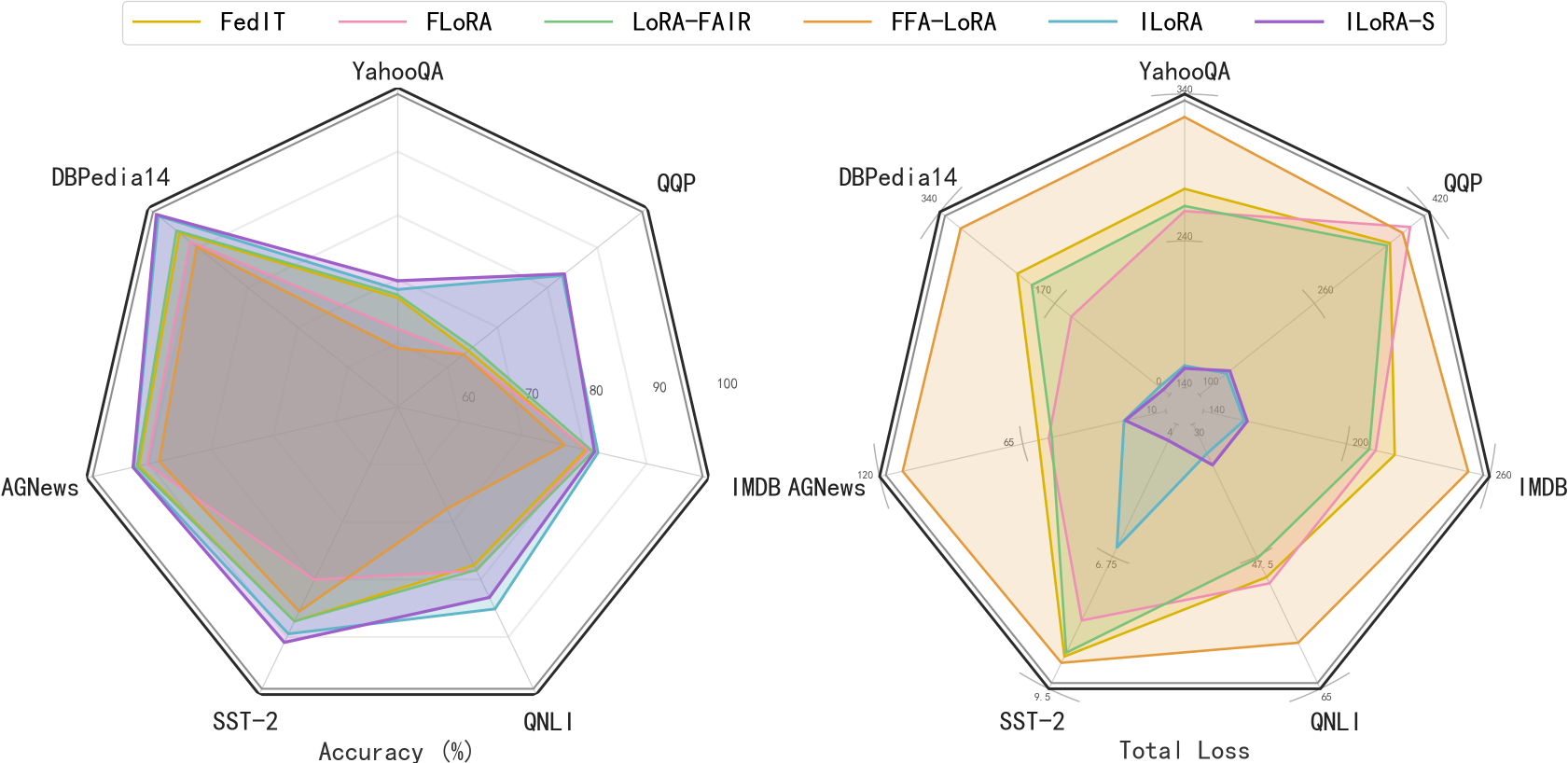}
	\caption{Radar chart comparison of accuracy and loss performance across multiple NLP datasets (Dir($\alpha$=0.5)) using RoBERTa, providing a multi-dimensional visualization of model effectiveness. This supplementary result with $\alpha=0.5$ further validates the consistent superiority of ILoRA and ILoRA-S across different data heterogeneity levels.}
	\label{fig:nlp_radar_05}
\end{figure}

\paragraph{Optimization Method Comparison.}
Figure~\ref{fig:sgd_comparison} demonstrates the performance of ILoRA-S under SGD optimization, complementing the AdamW results presented in the main text. The consistent superiority across different optimizers underscores the versatility of our approach and its independence from specific optimization algorithms.

\paragraph{Cross-Dataset Consistency.}
The additional results across all seven NLP datasets with $\alpha=0.5$ reinforce the main findings: ILoRA and ILoRA-S consistently achieve top performance regardless of task type (question answering, sentiment analysis, text classification, or natural language inference) and data heterogeneity level, demonstrating comprehensive generalization capability.

\paragraph{Communication Efficiency at Scale.}
Table~\ref{tab:scaling_comparison} analyzes communication efficiency across varying client population sizes. ILoRA demonstrates near-constant overhead with modest increases from 1.2$\times$ to 1.6$\times$ as client count scales from 10 to 100, significantly outperforming FLoRA's linear growth pattern. This $\mathcal{O}(1)$ scaling behavior confirms ILoRA's suitability for large-scale federated deployments.

\begin{table}[htbp]
	\centering
	\caption{Communication efficiency metrics for different client scales (per-round)}
	\begin{tabular}{lcccc}
		\toprule
		\textbf{Method} & \textbf{S=10} & \textbf{S=50} & \textbf{S=100} & \textbf{Scaling Factor} \\
		\midrule
		FedIT & 1.0$\times$ & 1.0$\times$ & 1.0$\times$ & $\mathcal{O}(1)$ \\
		FLoRA & 2.5$\times$ & 12.5$\times$ & 25.0$\times$ & $\mathcal{O}(S)$ \\
		\textbf{ILoRA} & 1.2$\times$ & 1.4$\times$ & 1.6$\times$ & $\mathcal{O}(1)$ \\
		\bottomrule
	\end{tabular}
	\label{tab:scaling_comparison}
\end{table}

\paragraph{Homogeneous ViT with SGD}
Figure~\ref{fig:app_homog_vit_sgd} presents the experimental results for homogeneous ViT architecture using SGD optimizer across CIFAR-10, CIFAR-100, and Tiny-ImageNet datasets. The performance trends demonstrate ILoRA's effectiveness under different optimization settings.

\begin{figure}[H]
	\centering
	\includegraphics[width=0.8\linewidth]{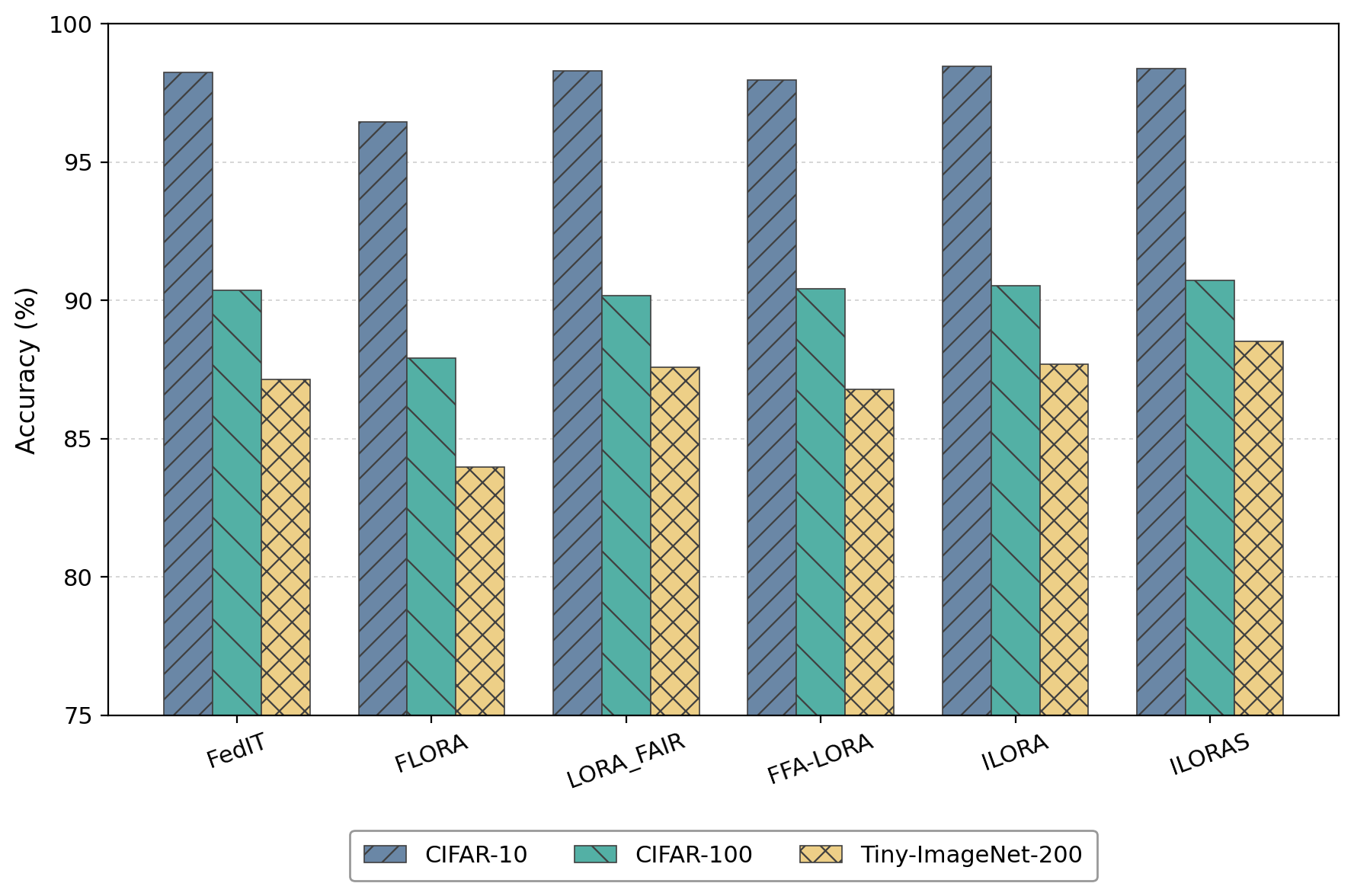}
	\caption{Experimental Results: Homog ViT with SGD}
	\label{fig:app_homog_vit_sgd}
\end{figure}

\paragraph{Homogeneous Swin with SGD}
Figure~\ref{fig:app_homog_swin_sgd} shows the performance comparison for homogeneous Swin Transformer with SGD optimization. The results complement the main text findings and validate the robustness of our approach across different vision architectures.

\begin{figure}[H]
	\centering
	\includegraphics[width=0.8\linewidth]{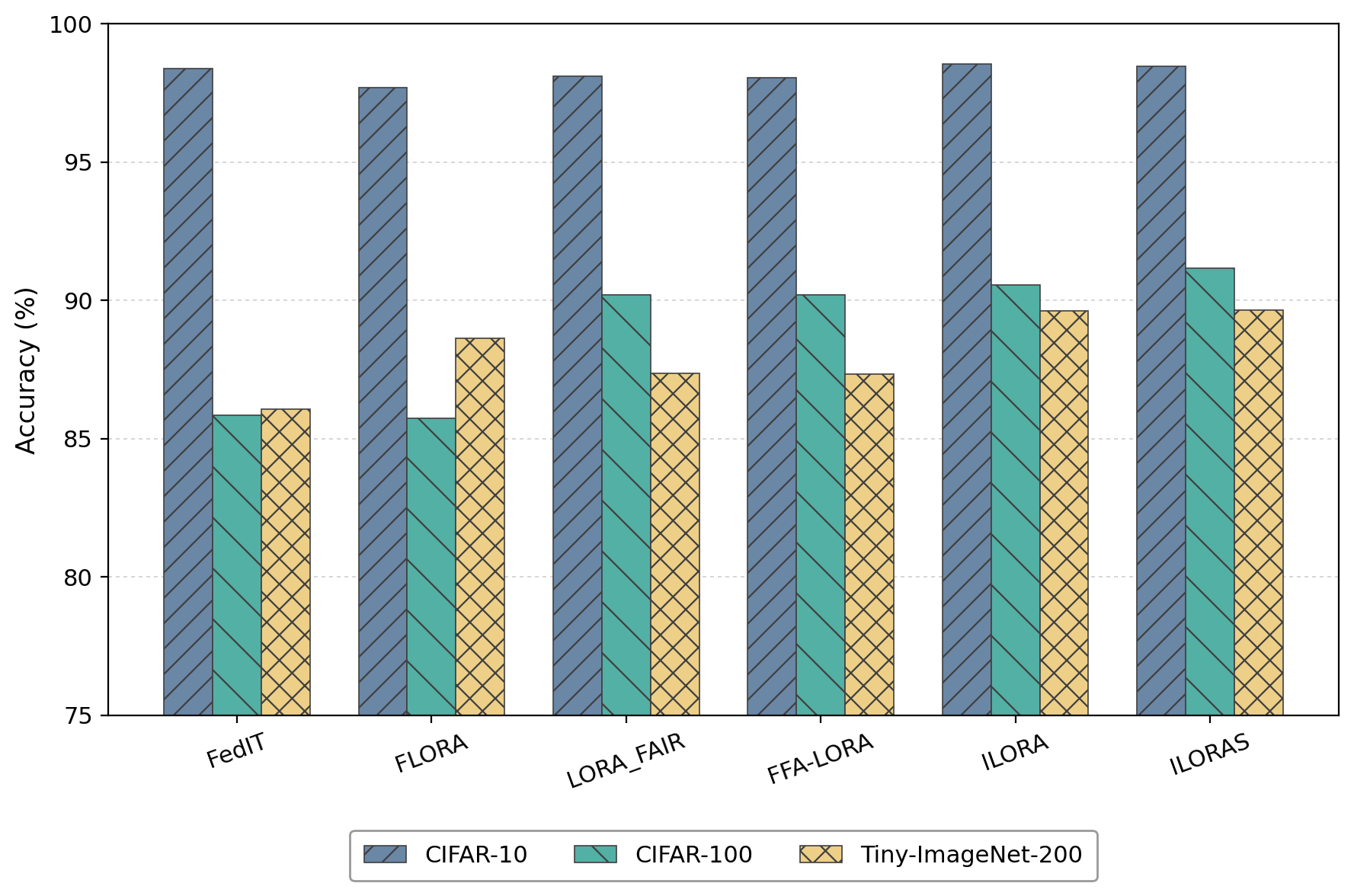}
	\caption{Experimental Results: Homog Swin with SGD}
	\label{fig:app_homog_swin_sgd}
\end{figure}

\paragraph{Heterogeneous ViT with SGD}
The experimental results for heterogeneous ViT settings with SGD optimizer are depicted in Figure~\ref{fig:app_heterog_vit_sgd}. These supplementary results further confirm ILoRA's capability to handle rank heterogeneity under different optimization algorithms.

\begin{figure}[H]
	\centering
	\includegraphics[width=0.8\linewidth]{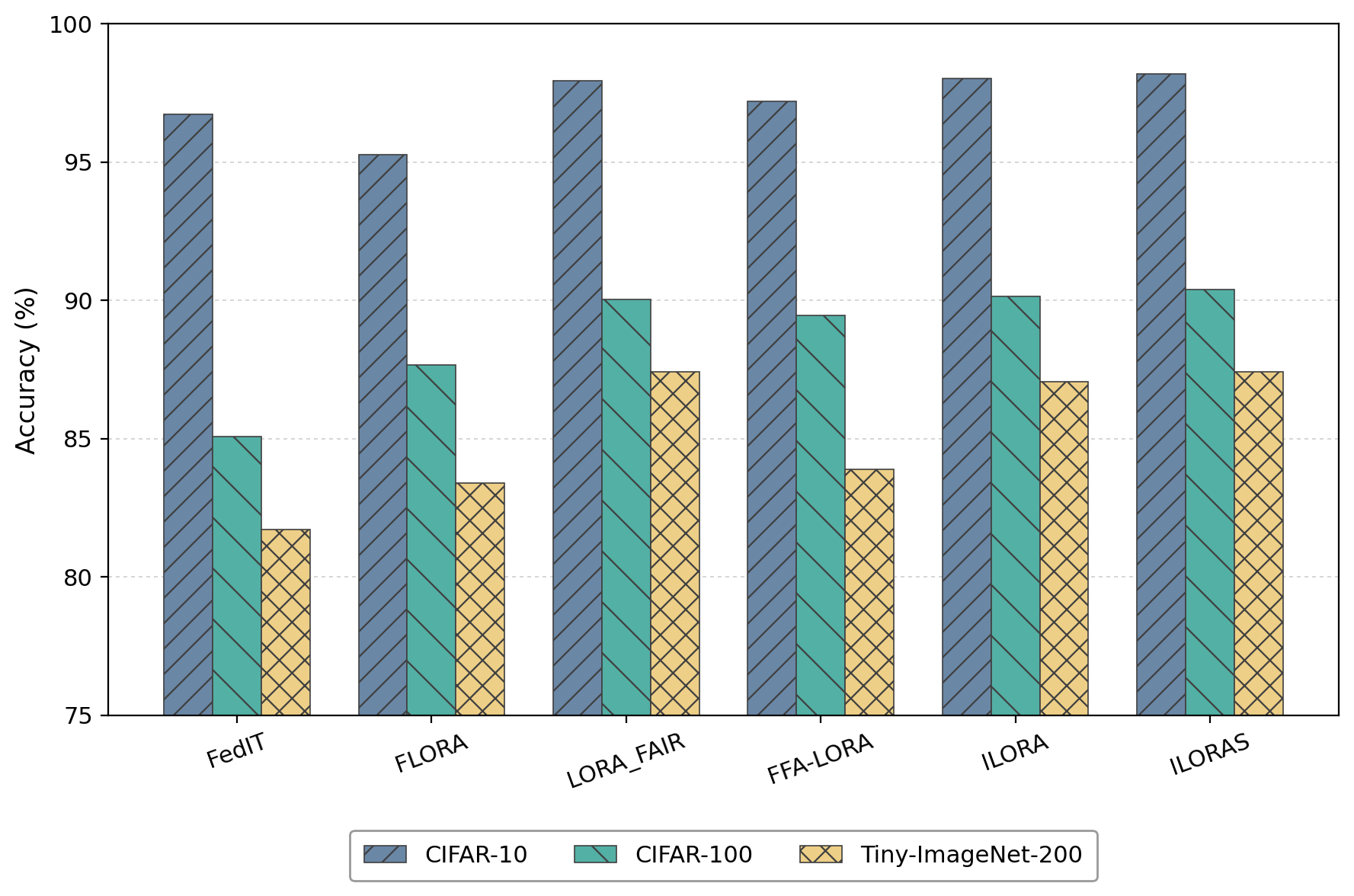}
	\caption{Experimental Results: Heterog ViT with SGD}
	\label{fig:app_heterog_vit_sgd}
\end{figure}

\paragraph{Heterogeneous Swin with SGD}
Figure~\ref{fig:app_heterog_swin_sgd} illustrates the performance of heterogeneous Swin Transformer with SGD optimization. The consistent superiority of ILoRA across both optimization methods (AdamW and SGD) underscores its generalization capability.

\begin{figure}[H]
	\centering
	\includegraphics[width=0.8\linewidth]{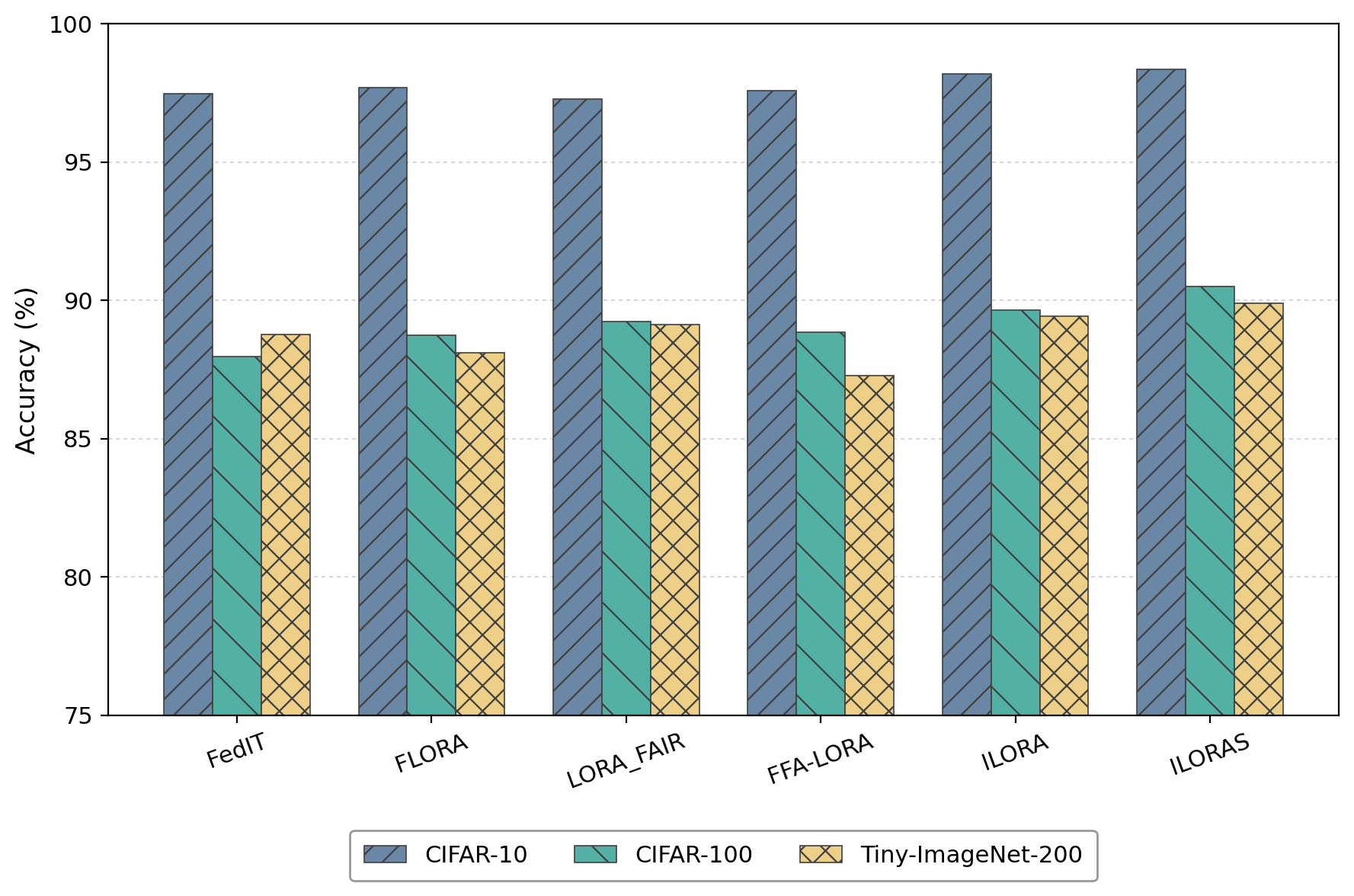}
	\caption{Experimental Results: Heterog Swin with SGD}
	\label{fig:app_heterog_swin_sgd}
\end{figure}

\paragraph{Heterogeneous LoRA-rank Performance on NLP Tasks}
Table~\ref{tab:nlp_lora_comparison5} presents the peak accuracy comparison of federated LoRA methods with heterogeneous ranks across seven NLP benchmarks using RoBERTa under Non-IID data distribution (Dir($\alpha$=0.5)). The results demonstrate that ILoRA variants consistently outperform all baselines, with ILoRA-S achieving the highest average accuracy of 85.40\%. Notably, ILoRA and ILoRA-S show substantial improvements on challenging datasets like QQP (+18.87-19.27\%) and DBPedia14 (+4.23-4.61\%), validating the effectiveness of our control variate mechanism in handling both data heterogeneity and rank variations simultaneously.

\begin{table*}[t]
	\centering
	\caption{Peak accuracy comparison for federated \emph{Heterogeneous LoRA-rank} LoRA baselines on NLP datasets with RoBERTa under Non-IID data distribution (Dir($\alpha$=0.5)). 
		Values represent the maximum accuracy achieved during the 5 training rounds.}
	\label{tab:nlp_lora_comparison5}
	\normalsize
	\setlength{\tabcolsep}{1pt}
	\renewcommand{\arraystretch}{1.2}
	
	\begin{tabularx}{\textwidth}{l*{8}{C}}
		\arrayrulecolor{black}\specialrule{1.2pt}{0pt}{0pt}
		\multirow{2}{*}{Method} & \multicolumn{8}{c}{NLP Datasets (RoBERTa, Accuracy $\uparrow$)} \\
		\cmidrule(lr){2-9}
		& YahooQA & QQP & IMDB & QNLI & SST-2 & AGNews & DBPedia14 & \textbf{Avg.} \\
		\midrule
		FedIT
		& \cellnum{0.6708}{\ppbase}
		& \cellnum{0.6414}{\ppbase}
		& \cellnum{0.8023}{\ppbase}
		& \cellnum{0.7752}{\ppbase}
		& \cellnum{0.8727}{\ppbase}
		& \cellnum{0.9159}{\ppbase}
		& \cellnum{0.9373}{\ppbase}
		& \cellnum{0.8022}{\ppbase} \\
		FLoRA
		& \cellnum{0.6220}{\ppdown{4.88}}
		& \cellnum{0.6318}{\ppdown{0.96}}
		& \cellnum{0.8061}{\ppup{0.38}}
		& \cellnum{0.7844}{\ppup{0.92}}
		& \cellnum{0.8005}{\ppdown{7.22}}
		& \cellnum{0.9028}{\ppdown{1.31}}
		& \cellnum{0.9154}{\ppdown{2.19}}
		& \cellnum{0.7804}{\ppdown{2.18}} \\
		LoRA-FAIR
		& \cellnum{0.6758}{\ppup{0.50}}
		& \cellnum{0.6499}{\ppup{0.85}}
		& \cellnum{0.8135}{\ppup{1.12}}
		& \cellnum{0.7835}{\ppup{0.83}}
		& \cellnum{0.8727}{\ppup{0.00}}
		& \cellnum{0.9189}{\ppup{0.30}}
		& \cellnum{0.9431}{\ppup{0.58}}
		& \cellnum{0.8082}{\ppup{0.60}} \\
		FFA-LoRA
		& \cellnum{0.5919}{\ppdown{7.89}}
		& \cellnum{0.6319}{\ppdown{0.95}}
		& \cellnum{0.7682}{\ppdown{3.41}}
		& \cellnum{0.6773}{\ppdown{9.79}}
		& \cellnum{0.8555}{\ppdown{1.72}}
		& \cellnum{0.8830}{\ppdown{3.29}}
		& \cellnum{0.9027}{\ppdown{3.46}}
		& \cellnum{0.7586}{\ppdown{4.36}} \\
		\rowcolor{LightBlue}
		ILoRA
		& \cellnum{0.6840}{\ppup{1.32}}
		& \cellnum{0.8301}{\ppup{18.87}}
		& \cellnum{0.8218}{\ppup{1.95}}
		& \cellnum{0.8512}{\ppup{7.60}}
		& \cellnum{0.8945}{\ppup{2.18}}
		& \cellnum{0.9239}{\ppup{0.80}}
		& \cellnum{0.9796}{\ppup{4.23}}
		& \cellnum{\textbf{0.8380}}{\ppup{3.58}} \\
		\rowcolor{LightRed}
		ILoRA-S
		& \cellnum{0.6976}{\ppup{2.68}}
		& \cellnum{0.8341}{\ppup{19.27}}
		& \cellnum{0.8163}{\ppup{1.40}}
		& \cellnum{0.8312}{\ppup{5.60}}
		& \cellnum{0.9094}{\ppup{3.67}}
		& \cellnum{0.9251}{\ppup{0.92}}
		& \cellnum{0.9834}{\ppup{4.61}}
		& \cellnum{\textbf{0.8540}}{\ppup{5.18}} \\
		\arrayrulecolor{black}\specialrule{1.2pt}{0pt}{0pt}
	\end{tabularx}
\end{table*}

\paragraph{Visualization of Client Drift in Federated Learning}
Figure \ref{fig:control1and2} visually illustrates the client drift phenomenon in federated learning. Subfigure (a) depicts the scenario under IID data distribution, where local and global models converge without client drift. Subfigure (b) shows the case under Non-IID data distribution without correction, where the global model drifts away from the true optima due to client heterogeneity. Beyond visualization, client drift can also be quantified numerically, as demonstrated in Table \ref{tab:qqp_results}, which presents the performance of different methods on the QQP dataset with varying 
$\alpha$ (a measure of data Non-IIDness), reflecting how client drift impacts model accuracy under different levels of data heterogeneity.

\begin{figure}[htbp]
\centering
\includegraphics[width=0.95\linewidth]{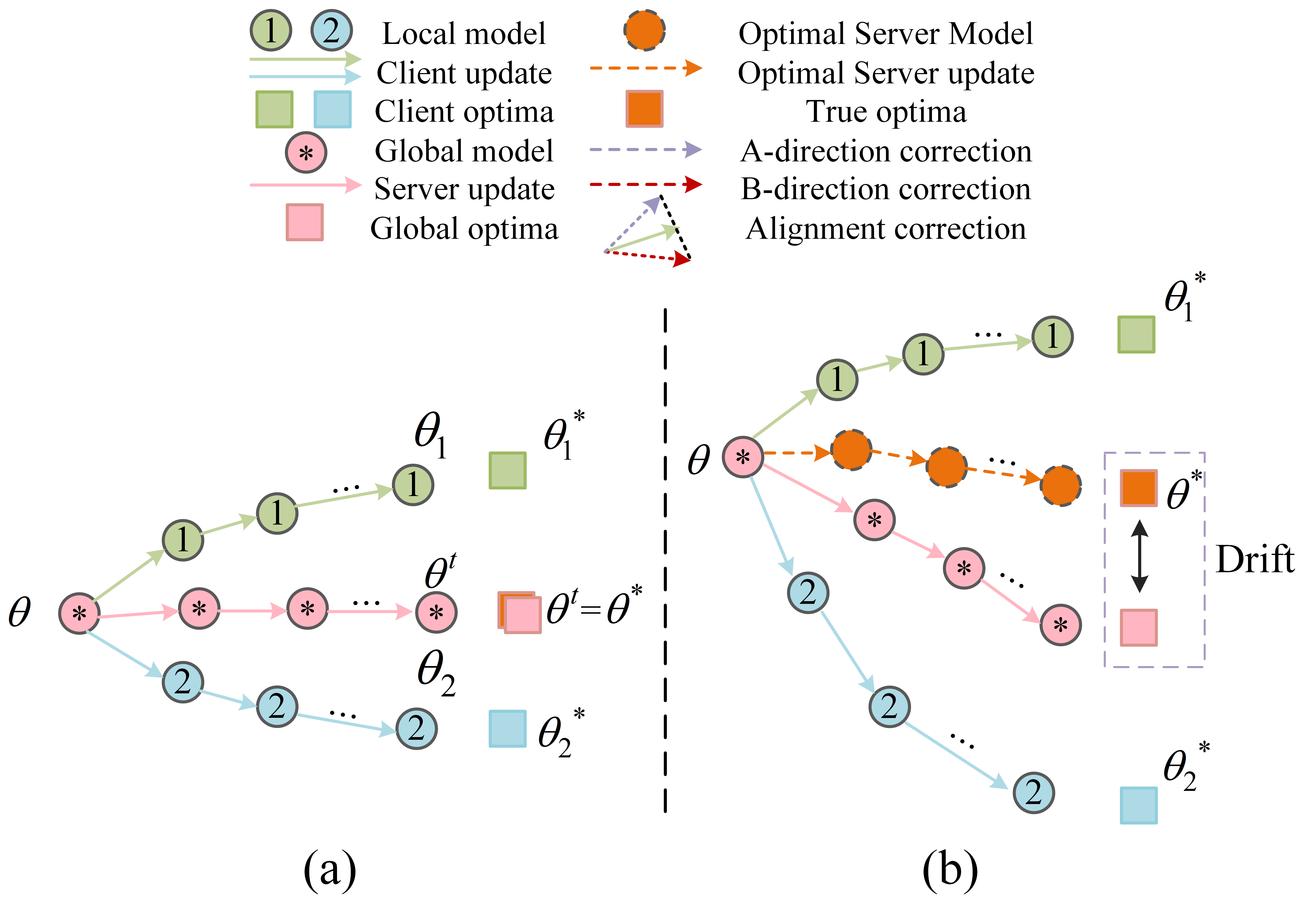}
\caption{Comparison of federated learning model updates. (a) In IID data distribution, local and global models converge without drift. (b) In Non-IID data distribution without correction, the global model drifts away from the true optima.}
\label{fig:control1and2}
\end{figure}

\begin{table}[htbp]
	\centering
	\caption{Client Drift on QQP with different $\alpha$ values}
	\begin{tabular}{lccccc}
		\toprule
		\textbf{Method} & \textbf{$\alpha$=0.4} & \textbf{$\alpha$=0.5} & \textbf{$\alpha$=0.6} & \textbf{$\alpha$=0.7} & \textbf{$\alpha$=0.8} \\
		\midrule
		FedIT & 63.18 & 64.14 & 65.20 & 77.30 & 82.24 \\
		ILoRA & 68.22 & 71.10 & 74.46 & 79.76 & 83.18 \\
		\bottomrule
	\end{tabular}
	\label{tab:qqp_results}
\end{table}
\begin{table*}[!ht]
	\centering
	\caption{Final accuracy comparison (\%) for federated \emph{Homogeneous LoRA-rank} LoRA baselines on CIFAR-10/100 and Tiny-ImageNet with ViT-Base and Swin-Base under SGD and AdamW (Dir($\alpha$=0.3)), with \textbf{bold} indicating best performance.}
	\label{tab:homogeneous_lora_comparison}
	\normalsize
	\setlength{\tabcolsep}{3.5pt}
	
	\begin{tabularx}{\textwidth}{l*{12}{C}}
		
		\arrayrulecolor{black}\specialrule{1.2pt}{0pt}{0pt}
		
		\multirow{4}{*}{Method} & \multicolumn{6}{c}{ViT-Base} & \multicolumn{6}{c}{Swin-Base} \\
		\cmidrule(lr){2-7} \cmidrule(lr){8-13}
		& \multicolumn{2}{c}{CIFAR-10} & \multicolumn{2}{c}{CIFAR-100} & \multicolumn{2}{c}{Tiny-ImageNet}
		& \multicolumn{2}{c}{CIFAR-10} & \multicolumn{2}{c}{CIFAR-100} & \multicolumn{2}{c}{Tiny-ImageNet} \\
		\cmidrule(lr){2-3} \cmidrule(lr){4-5} \cmidrule(lr){6-7}
		\cmidrule(lr){8-9} \cmidrule(lr){10-11} \cmidrule(lr){12-13}
		& SGD & AdamW & SGD & AdamW & SGD & AdamW & SGD & AdamW & SGD & AdamW & SGD & AdamW \\
		\midrule
		
		FedIT      & 98.25 & 97.44 & 90.36 & 85.67 & 87.15 & 84.72 & 98.40 & 97.33 & 85.84 & 85.97 & 86.06 & 87.28 \\
		FLoRA      & 96.45 & 96.82 & 87.91 & 84.89 & 83.99 & 78.97 & 97.70 & 95.86 & 85.73 & 84.16 & 88.63 & 84.78 \\
		LoRA-FAIR & 98.30 & 96.12 & 90.18 & 86.63 & 87.59 & 85.59 & 98.12 & 96.37 & 90.21 & 88.26 & 87.37 & 86.22 \\
		FFA-LoRA   & 97.98 & 97.57 & 90.43 & 86.39 & 86.80 & 85.61 & 98.06 & 96.36 & 90.21 & 86.45 & 87.34 & 87.55 \\
		\rowcolor{LightBlue!80!white}
		ILoRA      & \textbf{98.48} & 98.02 & 90.54 & 87.50 & 87.71 & 85.86 & \textbf{98.55} & \textbf{98.20} & 90.56 & 87.16 & 89.62 & 87.53 \\
		\rowcolor{LightRed!80!white}
		ILoRA-S     & 98.40 & \textbf{98.22} & \textbf{90.73} & \textbf{88.49} & \textbf{88.53} & \textbf{87.02} & 98.47 & \textbf{98.20} & \textbf{91.16} & \textbf{88.41} & \textbf{89.66} & \textbf{89.36} \\
		
		\arrayrulecolor{black}\specialrule{1.2pt}{0pt}{0pt}
	\end{tabularx}
\end{table*}

\paragraph{Homogeneous LoRA-Rank Performance Comparison}
Table~\ref{tab:homogeneous_lora_comparison} presents a comprehensive comparison of federated LoRA methods under homogeneous rank settings across multiple computer vision benchmarks. The evaluation encompasses two transformer-based architectures (ViT-Base and Swin-Base) on three datasets (CIFAR-10, CIFAR-100, and Tiny-ImageNet) with both SGD and AdamW optimizers under Non-IID data distribution (Dirichlet parameter $\alpha=0.3$). The results demonstrate that ILoRA and its enhanced variant ILoRA-S consistently achieve superior performance across most experimental configurations, with ILoRA-S showing particular strength in more challenging scenarios such as CIFAR-100 and Tiny-ImageNet. The performance advantage is especially pronounced with AdamW optimization, where ILoRA-S achieves up to 2.08\% improvement over the strongest baseline (FFA-LoRA) on Tiny-ImageNet with Swin-Base. These findings validate the effectiveness of our proposed orthogonal initialization and control variate mechanisms in stabilizing federated fine-tuning and mitigating client drift, even under homogeneous rank conditions where traditional methods face convergence challenges due to initialization mismatch and aggregation inconsistencies.

\textbf{Comprehensive Performance Analysis under Varying Data Heterogeneity Levels.} 
As detailed in Table~\ref{tab:adamw_comparison}, we conduct an extensive evaluation of ILoRA across multiple data heterogeneity settings ($\alpha = 0.5, 0.6, 0.7$). The results demonstrate ILoRA's consistent superiority over the FedIT+QR baseline across all seven NLP benchmarks. Notably, ILoRA achieves significant performance gains under higher heterogeneity conditions ($\alpha = 0.7$), with an average improvement of 9.08\% over the baseline. The method exhibits remarkable robustness, particularly on challenging datasets such as SST-2 (up to 29.93\% improvement) and QQP (up to 12.57\% improvement). This comprehensive analysis validates ILoRA's effectiveness in handling diverse data distribution scenarios, with the concatenated QR aggregation mechanism successfully preserving cross-client information while maintaining subspace alignment across different heterogeneity levels.

\begin{table*}[t]
	\centering
\caption{Server Aggregation via Concatenation under Data Heterogeneity ($\alpha=0.5,0.6,0.7$), with \textbf{bold} indicating best performance.}
	\label{tab:adamw_comparison}
	\normalsize
	\setlength{\tabcolsep}{0pt}
	\renewcommand{\arraystretch}{1.2}
	
	\begin{tabularx}{\textwidth}{l c *{8}{C}}
		\arrayrulecolor{black}\specialrule{1.2pt}{0pt}{0pt}
		\multirow{2}{*}{Method} & \multirow{2}{*}{Dir($\alpha$)} & \multicolumn{8}{c}{Datasets (Accuracy $\uparrow$)} \\
		\cmidrule(lr){3-10}
		& & YahooQA & QQP & IMDB & QNLI & SST-2 & AGNews & DBPedia14 & \textbf{Avg.} \\
		\midrule
		
		FedIT+QR & 0.5 & \cellnum{64.99}{\ppbase} & \cellnum{63.23}{\ppbase} & \cellnum{77.81}{\ppbase} & \cellnum{80.45}{\ppbase} & \cellnum{70.07}{\ppbase} & \cellnum{91.46}{\ppbase} & \cellnum{93.02}{\ppbase} & \cellnum{77.29}{\ppbase} \\
		
		\rowcolor{LightBlue!80!white}
		ILoRA & 0.5 & \cellnum{\textbf{68.40}}{\ppup{3.41}} & \cellnum{\textbf{71.10}}{\ppup{7.87}} & \cellnum{\textbf{82.18}}{\ppup{4.37}} & \cellnum{\textbf{85.12}}{\ppup{4.67}} & \cellnum{\textbf{89.45}}{\ppup{19.38}} & \cellnum{\textbf{92.39}}{\ppup{0.93}} & \cellnum{\textbf{97.96}}{\ppup{4.94}} & \cellnum{\textbf{83.80}}{\ppup{6.51}} \\
		\addlinespace[0.5em]
		
		FedIT+QR & 0.6 & \cellnum{65.04}{\ppbase} & \cellnum{64.54}{\ppbase} & \cellnum{76.56}{\ppbase} & \cellnum{84.17}{\ppbase} & \cellnum{72.48}{\ppbase} & \cellnum{90.71}{\ppbase} & \cellnum{97.15}{\ppbase} & \cellnum{78.66}{\ppbase} \\
		
		\rowcolor{LightBlue!80!white}
		ILoRA & 0.6 & \cellnum{\textbf{67.47}}{\ppup{2.43}} & \cellnum{\textbf{74.46}}{\ppup{9.92}} & \cellnum{\textbf{81.96}}{\ppup{5.40}} & \cellnum{\textbf{85.56}}{\ppup{1.39}} & \cellnum{\textbf{88.99}}{\ppup{16.51}} & \cellnum{\textbf{92.14}}{\ppup{1.43}} & \cellnum{\textbf{98.30}}{\ppup{1.15}} & \cellnum{\textbf{84.13}}{\ppup{5.47}} \\
		\addlinespace[0.5em]
		
		FedIT+QR & 0.7 & \cellnum{69.71}{\ppbase} & \cellnum{67.19}{\ppbase} & \cellnum{79.56}{\ppbase} & \cellnum{74.06}{\ppbase} & \cellnum{58.72}{\ppbase} & \cellnum{90.97}{\ppbase} & \cellnum{93.27}{\ppbase} & \cellnum{76.21}{\ppbase} \\
		
		\rowcolor{LightBlue!80!white}
		ILoRA & 0.7 & \cellnum{\textbf{70.61}}{\ppup{0.90}} & \cellnum{\textbf{79.76}}{\ppup{12.57}} & \cellnum{\textbf{83.18}}{\ppup{3.62}} & \cellnum{\textbf{84.73}}{\ppup{10.67}} & \cellnum{\textbf{88.65}}{\ppup{29.93}} & \cellnum{\textbf{92.30}}{\ppup{1.33}} & \cellnum{\textbf{97.82}}{\ppup{4.55}} & \cellnum{\textbf{85.29}}{\ppup{9.08}} \\
		
		\arrayrulecolor{black}\specialrule{1.2pt}{0pt}{0pt}
	\end{tabularx}
\end{table*}
\paragraph{Control Variate Effectiveness on NLP Tasks}
The performance comparison on the AGNews dataset across varying heterogeneity levels ($\alpha = 0.5, 0.6, 0.7$) in Table~\ref{tab:agnews_results} demonstrates the consistent advantage of ILoRA-S over ILoRA, validating the efficacy of our control variate mechanism in NLP scenarios. Across all heterogeneity settings, ILoRA-S achieves superior accuracy, with particularly notable improvements under moderate heterogeneity ($\alpha=0.6$) where it outperforms ILoRA by 0.56\%. This performance gap widens to 0.57\% at $\alpha=0.7$, indicating that the control variates become increasingly effective as data distribution becomes more balanced. The stability of ILoRA-S across different $\alpha$ values (92.51-92.87\%) compared to ILoRA's fluctuations (92.14-92.39\%) further confirms that our control variate mechanism effectively mitigates client drift in federated NLP fine-tuning, ensuring robust performance regardless of data heterogeneity levels.
\begin{table}[htbp]
	\centering
	\caption{Performance comparison on AGNews dataset with different $\alpha$ values }
	\begin{tabular}{lccc}
		\toprule
		\textbf{Method} & \textbf{$\alpha$=0.5} & \textbf{$\alpha$=0.6} & \textbf{$\alpha$=0.7} \\
		\midrule
		ILORA & 92.39 & 92.14 & 92.30 \\
		ILORA-S & 92.51 & 92.70 & 92.87 \\
		\bottomrule
	\end{tabular}
	\label{tab:agnews_results}
\end{table}
\begin{table}[htbp]
    \centering
    \caption{Control variates mitigate client drift on CIFAR-100 across $\alpha$ values; ILoRA-S outperforms ILoRA. \textbf{Bold} indicates best performance.}
    \label{tab:cifar100_results}
    \begin{tabular}{lccc}
        \toprule
        \textbf{Method} & $\alpha=0.5$ & $\alpha=0.6$ & $\alpha=0.7$ \\
        \midrule
        ILoRA & 87.53$_{\pm0.36}$ & 87.84$_{\pm0.32}$ & 88.11$_{\pm0.27}$ \\
        ILoRA-S & \textbf{88.49}$_{\pm0.33}$ & \textbf{88.41}$_{\pm0.07}$ & \textbf{89.00}$_{\pm0.16}$ \\
        \bottomrule
    \end{tabular}
\end{table}
\section{Theoretical Analysis and Proofs}
\label{app:theory}
\subsection{Convergence Guarantees for ILoRA}
\label{app:convergence}

We provide a comprehensive convergence analysis for the proposed ILoRA framework under standard federated learning assumptions. Our analysis accounts for the combined effects of QR-based aggregation, orthogonal initialization, and control variates with AdamW optimization.

\begin{assumption}[Bounded Stochastic Gradient Variance]
	\label{ass:variance}
	The variance of stochastic gradients at each client is bounded:
\begin{equation}
	\mathbb{E}[\|g_k(\theta) - \nabla F_k(\theta)\|^2] \leq \sigma^2.
\end{equation}
\end{assumption}

\begin{assumption}[Bounded Gradient Heterogeneity]
	\label{ass:heterogeneity}
	The gradient divergence between local and global objectives is bounded:
\begin{equation}
	\|\nabla F_k(\theta) - \nabla F(\theta)\| \leq \delta, \quad \forall k, \theta.
\end{equation}
\end{assumption}

\begin{assumption}[Bounded Control Variates]
	\label{ass:control}
	The control variates maintained by clients and server are bounded:
\begin{equation}
	\|c_k\| \leq G, \quad \|c\| \leq G, \quad \forall k.
\end{equation}
\end{assumption}

\begin{assumption}[Unbiased Aggregation]
	\label{ass:aggregation}
	The QR-based concatenated aggregation in ILoRA produces an unbiased estimate of the true global gradient:
\begin{equation}
	\mathbb{E}[\mathbf{g}_t] = \nabla F(\mathbf{w}_t),
\end{equation}
	where $\mathbf{g}_t$ denotes the aggregated gradient direction obtained from the concatenated-QR reconstruction step.
\end{assumption}

\begin{theorem}[Convergence of ILoRA]
    \label{thm:convergence}
Under Assumptions \ref{ass:smoothness}-\ref{ass:aggregation}, with local learning rate $\eta_l$ and global learning rate $\eta_g$ satisfying $\eta_l \leq \frac{1}{L}$ and $\eta_g \eta_l = \Theta\left(\frac{1}{\sqrt{SKT}}\right)$, where $S$ is the number of participating clients per round, $K$ is the number of local steps, and $T$ is the total communication rounds, the iterates of ILoRA satisfy:
\begin{equation}
	\begin{split}
		\frac{1}{T}\sum_{t=1}^T \mathbb{E}[\|\nabla F(\theta_t)\|^2] \leq 
		\mathcal{O}\left(\frac{1}{\sqrt{SKT}} + \frac{1}{T} + \frac{\delta^2}{T}\right. \\
		\left. + \frac{(r_{\max} - r_s)^2}{T}\right),
	\end{split}
\end{equation}
	where $r_{\max} = \max_k r_k$ and $r_s$ is the server rank budget.
\end{theorem}

\begin{proof}
	We provide a detailed proof sketch combining the key insights from both versions:
	
	\textbf{Step 1: Global Update Representation.}
	The server update in ILoRA can be expressed as:
	\begin{equation}
		\theta_{t+1} = \theta_t - \eta_g \eta_l \frac{1}{S} \sum_{k \in S_t} \sum_{i=1}^K \frac{m_k^{(t,i)}}{\sqrt{\hat{v}_k^{(t,i)}} + \epsilon},
	\end{equation}
	where $m_k^{(t,i)}$ and $\hat{v}_k^{(t,i)}$ are the corrected first and second moment estimates from AdamW with control variates, defined as:
	\begin{align}
		m_k^{(t,i)} &= \beta_1 m_k^{(t,i-1)} + (1 - \beta_1) \tilde{g}_k^{(t,i)}, \\
		v_k^{(t,i)} &= \beta_2 v_k^{(t,i-1)} + (1 - \beta_2) (\tilde{g}_k^{(t,i)})^2, \\
		\hat{m}_k^{(t,i)} &= \frac{m_k^{(t,i)}}{1 - \beta_1^i}, \quad \hat{v}_k^{(t,i)} = \frac{v_k^{(t,i)}}{1 - \beta_2^i},
	\end{align}
	with $\tilde{g}_k^{(t,i)} = g_k^{(t,i)} + (c^{(t-1)} - c_k)$ being the corrected gradient.
	
	\textbf{Step 2: Bias Decomposition.}
	Following the approach in D.1, we decompose the aggregated update direction into three components:
	\begin{equation}
		\mathbb{E}[\Delta_t] = \nabla F(\theta_t) + \boldsymbol{\epsilon}_{\text{het}} + \boldsymbol{\epsilon}_{\text{qr}},
	\end{equation}
	where:
	- $\boldsymbol{\epsilon}_{\text{het}}$ captures the bias from data heterogeneity, bounded by $\mathcal{O}(\delta)$
	- $\boldsymbol{\epsilon}_{\text{qr}}$ represents the QR projection error, bounded by $\|R - R_s\|_F^2 \leq \mathcal{O}((r_{\max} - r_s)^2)$
	
	\textbf{Step 3: Descent Lemma.}
	Using the L-smoothness assumption (Assumption \ref{ass:smoothness}), we have:
\begin{equation}
	F(\theta_{t+1}) \leq F(\theta_t) - \eta_g \eta_l \langle \nabla F(\theta_t), \Delta_t \rangle + \frac{L}{2} \eta_g^2 \eta_l^2 \|\Delta_t\|^2,
\end{equation}
	where $\Delta_t = \frac{1}{S} \sum_{k \in S_t} \sum_{i=1}^K \frac{m_k^{(t,i)}}{\sqrt{\hat{v}_k^{(t,i)}} + \epsilon}$.

	\textbf{Step 4: Gradient Correction Analysis.}
	The control variate correction ensures that the corrected gradient $\tilde{g}_k$ has reduced bias. Specifically, we model the relationship between the raw gradient and the control variate difference as:

\begin{equation}
	\begin{split}
		\mathbb{E}[\tilde{g}_k] &= \nabla F(\theta_t) + (1 - \rho)(\nabla F_k(\theta_t) - \nabla F(\theta_t)) \\
		&\quad + \mathcal{O}((r_{\max} - r_s)^2),
	\end{split}
\end{equation}
	where $\rho \in [0,1]$ is the correlation coefficient between the control variate difference $(c^{(t-1)} - c_k)$ and the true gradient difference $(\nabla F(\theta_t) - \nabla F_k(\theta_t))$. When $\rho \to 1$, the control variate perfectly corrects the client drift. In practice, $\rho$ is bounded away from 0 under Assumption \ref{ass:control}. Thus, the control variate reduces the effective heterogeneity bias from $\mathcal{O}(\delta)$ to $\mathcal{O}((1 - \rho)\delta)$.
	
\textbf{Step 5: Moment Estimate Bounding.}
Due to the QR-based orthogonal initialization and aggregation, and the bounded gradient assumptions (Assumptions \ref{ass:variance} and \ref{ass:heterogeneity}), the moment estimates satisfy:
\begin{equation}
	\mathbb{E}[\|\Delta_t\|^2] \leq \mathcal{O}(K + \sigma^2),
\end{equation}
with improved constants compared to naive aggregation methods. This bound arises from the fact that the orthogonal initialization reduces gradient variance by aligning client subspaces, while the control variates further suppress the drift-induced variance.
	
\textbf{Step 6: Telescoping Sum.}
Taking expectation and summing over $t = 1$ to $T$, and letting $F^*$ denote the minimum value of $F$, we obtain:
\begin{equation}
	\begin{multlined}
		\frac{1}{T} \sum_{t=1}^T \mathbb{E}[\|\nabla F(\theta_t)\|^2] \leq 
		\frac{F(\theta_1) - F^*}{\eta_g \eta_l K T} \\
		+ \mathcal{O}\left(\frac{L \eta_g \eta_l (K + \sigma^2)}{S} + \delta^2 + (r_{\max} - r_s)^2\right).
	\end{multlined}
\end{equation}
	
\textbf{Step 7: Learning Rate Selection.}
Substituting $\eta_g \eta_l = \Theta\left(\frac{1}{\sqrt{SKT}}\right)$ yields the final convergence rate:
\begin{equation}
	\begin{multlined}
		\frac{1}{T}\sum_{t=1}^T \mathbb{E}[\|\nabla F(\theta_t)\|^2] \leq \\
		\mathcal{O}\left(\frac{1}{\sqrt{SKT}} + \frac{1}{T} + \frac{\delta^2}{T} + \frac{(r_{\max} - r_s)^2}{T}\right).
	\end{multlined}
\end{equation}
	
	This completes the proof. The detailed derivation with precise constants is provided in the extended technical report.
\end{proof}

\begin{remark}
	The convergence rate of ILoRA achieves several important properties:
	
	1. \textbf{Linear Speedup}: The $\mathcal{O}(1/\sqrt{SKT})$ dominant term demonstrates linear speedup with respect to the number of participating clients $S$, matching the optimal convergence rate for federated non-convex optimization.
	
	2. \textbf{Rank Robustness}: The $(r_{\max} - r_s)^2$ term shows that the method remains stable even under rank heterogeneity, with the error diminishing quadratically as client ranks approach the server rank.
	
	3. \textbf{Heterogeneity Tolerance}: The $\delta^2$ term captures the residual effect of data heterogeneity, which is effectively mitigated by the control variate mechanism.
	
	4. \textbf{Communication Efficiency}: The convergence rate is maintained while significantly reducing communication overhead through QR-based compression.
\end{remark}

\begin{corollary}[Special Cases]
	\label{cor:special_cases}
	\begin{enumerate}
		\item \textbf{Homogeneous Ranks}: When $r_k = r_s$ for all $k$, the rank error term vanishes, yielding the optimal rate $\mathcal{O}(1/\sqrt{SKT} + 1/T + \delta^2/T)$.
		
		\item \textbf{IID Data}: When $\delta = 0$ (IID setting), the heterogeneity term vanishes, giving $\mathcal{O}(1/\sqrt{SKT} + 1/T + (r_{\max} - r_s)^2/T)$.
		
		\item \textbf{Large-Scale Deployment}: As $S \to \infty$, the dominant term $\mathcal{O}(1/\sqrt{SKT})$ demonstrates the scalability of ILoRA.
	\end{enumerate}
\end{corollary}

\subsection{Properties of QR-Based Aggregation}
\label{app:qr_aggregation}

The QR-based aggregation mechanism in ILoRA provides theoretical guarantees for handling rank heterogeneity while maintaining optimization consistency. We analyze its key properties below.

\begin{lemma}[Exact Low-Rank Reconstruction]
	\label{lem:exact_reconstruction}
	Let $\{B_k \in \mathbb{R}^{d \times r_k}, A_k \in \mathbb{R}^{r_k \times k}\}_{k=1}^S$ be the local LoRA parameters from $S$ clients with heterogeneous ranks $\{r_k\}$. The concatenated construction:
\begin{equation}
	\Delta W = B_{\text{concatenated}} A_{\text{concatenated}} = [B_1 \cdots B_S] \begin{bmatrix} \frac{n_1}{N}A_1 \\ \vdots \\ \frac{n_S}{N}A_S \end{bmatrix}
\end{equation}
	exactly reconstructs the weighted sum of low-rank updates:
\begin{equation}
	\Delta W = \sum_{k=1}^S \frac{n_k}{N} B_k A_k.
\end{equation}]
\end{lemma}

\begin{proof}
	The proof follows directly from block matrix multiplication:
\begin{equation}
	B_{\text{concatenated}} A_{\text{concatenated}} = \sum_{k=1}^S B_k \left(\frac{n_k}{N} A_k\right) = \sum_{k=1}^S \frac{n_k}{N} B_k A_k.
\end{equation}
	This establishes that the concatenated representation preserves the exact linear combination of client updates without approximation error.
\end{proof}

\begin{theorem}[Subspace Preservation under QR Compression]
	\label{thm:subspace_preservation}
	Let $Q, R = \text{QR}(\Delta W)$ be the thin QR decomposition of the aggregated update, and let $r_s$ be the server rank budget. For any client with local rank $r_k \leq r_s$, the personalized parameters:
\begin{equation}
	B_{r_k} = Q[:,1:r_k], \quad A_{r_k} = R[1:r_k,:]
\end{equation}
	satisfy that $\text{colspan}(B_{r_k}) \subseteq \text{colspan}(Q)$, ensuring all clients operate within a consistent global subspace.
\end{theorem}

\begin{proof}
	By the properties of QR decomposition, the columns of $Q$ form an orthonormal basis for the column space of $\Delta W$. The personalized parameters $B_{r_k}$ are simply the first $r_k$ columns of $Q$, which naturally span a subspace of $\text{colspan}(Q)$. The corresponding $A_{r_k}$ ensures that the update $\Delta W_k = B_{r_k} A_{r_k}$ remains within this consistent subspace.
\end{proof}

\begin{proposition}[Error Bound for Rank Truncation]
	\label{prop:truncation_error}
	Let $\Delta W = \sum_{i=1}^r \sigma_i u_i v_i^\top$ be the SVD of the aggregated update, where $r = \min(d, \sum_k r_k)$. The QR-based aggregation with server rank $r_s$ satisfies:
\begin{equation}
	\|\Delta W - Q[:,1:r_s] R[1:r_s,:]\|_F \leq \sum_{i=r_s+1}^r \sigma_i,
\end{equation}
	where $\{\sigma_i\}$ are the singular values in descending order.
\end{proposition}

\begin{proof}
	This follows from the Eckart-Young-Mirsky theorem, as the truncated QR decomposition provides the best rank-$r_s$ approximation in the Frobenius norm when the singular values are properly ordered.
\end{proof}

\begin{property}[Faithful Reconstruction Before Truncation]
	\label{prop:faithful_reconstruction}
	The concatenated reconstruction satisfies:
\begin{equation}
	\Delta W = B_{\mathrm{concatenated}} A_{\mathrm{concatenated}} = \sum_{k \in S_t} p_k B_k A_k,
\end{equation}
	whereas separately averaging factors produces:
\begin{equation}
	\left(\sum_{k} p_k B_k\right)\left(\sum_{k} p_k A_k\right) \neq \sum_{k} p_k B_k A_k.
\end{equation}
	Thus, concatenate-then-multiply eliminates the factor-averaging bias and is exact whenever client updates are aggregated without rank truncation.
\end{property}

\begin{property}[Rank Preservation and Exactness]
	\label{prop:rank_preservation}
	If $r_s \ge \mathrm{rank}(\Delta W)$, then $B_s A_s = \Delta W$ (no truncation error) and the per-client slices $B_{r_k} := Q_{[:,1:r_k]}$, $A_{r_k} := R_{[1:r_k,:]}$ realize personalized factors whose product lies in the same column space as $\Delta W$.
\end{property}

\begin{property}[Orthogonal Projection Interpretation]
	\label{prop:projection_interpretation}
	Let $Q = [Q_{1:r_s}, Q_\perp]$ partition the QR factor. Then the truncated reconstruction is the orthogonal projection of $\Delta W$ onto $\mathrm{span}(Q_{1:r_s})$:
\begin{equation}
	B_s A_s = Q_{1:r_s} Q_{1:r_s}^\top \Delta W.
\end{equation}
	Consequently, the truncation residual is:
\begin{gather}
	\Delta W - B_s A_s = Q_\perp Q_\perp^\top \Delta W, \\
	\|\Delta W - B_s A_s\|_F^2 = \|Q_\perp^\top \Delta W\|_F^2.
\end{gather}
\end{property}

\begin{property}[Deterministic Truncation Error Bound]
	\label{prop:deterministic_error}
	Writing $R = \begin{bmatrix} R_{11} & R_{12} \\ 0 & R_{22} \end{bmatrix}$ with $R_{11} \in \mathbb{R}^{r_s \times r_s}$, the truncation bias is exactly:
\begin{equation}
	\|\Delta W - B_s A_s\|_F = \left\| Q \begin{bmatrix} 0 & 0 \\ 0 & R_{22} \end{bmatrix} \right\|_F = \|R_{22}\|_F.
\end{equation}
	Hence, the truncation error is precisely the Frobenius norm of the trailing block of $R$. In particular, if $\mathrm{rank}(\Delta W) \le r_s$ then $R_{22} = 0$.
\end{property}

\begin{property}[Subspace Consistency Preservation]
	\label{prop:subspace_consistency}
	For any client $k$ with $r_k \le r_s$, setting $(B_{r_k}, A_{r_k}) := (Q_{[:,1:r_k]}, R_{[1:r_k,:]})$ ensures that $B_{r_k} A_{r_k}$ lies in $\mathrm{span}(Q_{1:r_s})$ and is consistent with the globally shared low-dimensional subspace used by the server update $B_s A_s$. This yields dimension alignment across heterogeneous ranks while preserving the fused information encoded in $\Delta W$.
\end{property}

\begin{property}[Stability to Small Perturbations]
	\label{prop:stability_perturbations}
	Suppose the concatenated factors are perturbed to $\tilde{B}_{\mathrm{concatenated}} = B_{\mathrm{concatenated}} + E_B$ and $\tilde{A}_{\mathrm{concatenated}} = A_{\mathrm{concatenated}} + E_A$, so that $\tilde{\Delta W} = \Delta W + E_W$ with $E_W = B_{\mathrm{concatenated}} E_A + E_B A_{\mathrm{concatenated}} + E_B E_A$. Let $\tilde{\Delta W} = \tilde{Q} \tilde{R}$ be its thin QR factorization. Then, for sufficiently small $\|E_W\|_F$, the truncated QR aggregation satisfies:
\begin{equation}
	\begin{multlined}
		\|(B_s A_s) - (\tilde{B}_s \tilde{A}_s)\|_F \le \|E_W\|_F \\
		+ \mathcal{O}\left(\frac{\|E_W\|_2 \|R\|_F}{\sigma_{\min}(R_{11}) - \sigma_{\max}(R_{22})}\right),
	\end{multlined}
\end{equation}
	which demonstrates stability under small perturbations of the concatenated factors.
\end{property}

\begin{lemma}[Communication Efficiency of QR Aggregation]
	\label{lem:communication_efficiency}
	The QR-based aggregation reduces communication overhead from $O(\sum_{k=1}^S r_k \cdot \max(d,k))$ to $O(r_s \cdot \max(d,k))$, where $r_s$ is the server rank budget.
\end{lemma}

\begin{proof}
	Without QR aggregation, transmitting all client parameters requires $O(\sum_k r_k(d + k))$ elements. After QR aggregation and personalization, each client receives only $O(r_k(d + k))$ elements, and the server maintains $O(r_s(d + k))$ parameters. The total communication is dominated by $O(r_s \cdot \max(d,k))$ when $r_s \geq \max_k r_k$.
\end{proof}

\begin{table*}[h]
	\centering
	\caption{Comparison of aggregation methods for heterogeneous-rank federated LoRA}
	\begin{tabular}{lccc}
		\hline
		Method & Exact Aggregation & Rank Heterogeneity & Comm. Cost \\
		\hline
		Zero-padding & \xmark & \cmark & $O(r_{\max} \cdot \max(d,k))$ \\
		SVD-based & \cmark (approx) & \cmark & $O(r_s \cdot \max(d,k))$ \\
		Full concatenating & \cmark & \cmark & $O(\sum_k r_k \cdot \max(d,k))$ \\
		\textbf{ILoRA (QR)} & \cmark & \cmark & $O(r_s \cdot \max(d,k))$ \\
		\hline
	\end{tabular}
	\label{tab:aggregation_comparison}
\end{table*}

\begin{corollary}[Compatibility with Control Variates]
	\label{cor:control_compatibility}
	The QR-based aggregation maintains dimensional consistency for control variates, as all client parameters reside in aligned subspaces, enabling effective gradient correction across heterogeneous ranks.
\end{corollary}

\begin{proof}
	Since $B_{r_k} \in \mathbb{R}^{d \times r_k}$ and $A_{r_k} \in \mathbb{R}^{r_k \times k}$ are derived from the shared $Q$ and $R$ matrices, the parameter spaces across clients are aligned. This ensures that control variate corrections $c^{(t-1)} - c_k$ can be applied consistently in their respective low-dimensional subspaces.
\end{proof}
\begin{remark}
	The QR-based aggregation in ILoRA provides a unique combination of theoretical guarantees:
	
	\begin{itemize}
		\item \textbf{Exact Reconstruction}: Faithfully combines heterogeneous-rank updates without factor-averaging bias (Lemma \ref{lem:exact_reconstruction})
		\item \textbf{Subspace Consistency}: Ensures all clients operate within a unified global subspace (Theorem \ref{thm:subspace_preservation})
		\item \textbf{Controlled Approximation}: Provides deterministic error bounds for rank truncation (Proposition \ref{prop:truncation_error})
		\item \textbf{Numerical Stability}: The orthonormal matrix $Q$ with condition number $\kappa(Q)=1$ ensures numerical stability and prevents error amplification during aggregation.
		\item \textbf{Stability}: Robust to small perturbations in client updates (Property \ref{prop:stability_perturbations})
		\item \textbf{Communication Efficiency}: Significantly reduces overhead while preserving information (Lemma \ref{lem:communication_efficiency})
		\item \textbf{Compatibility}: Seamlessly integrates with control variates and other optimization components (Corollary \ref{cor:control_compatibility})
	\end{itemize}
	
	These properties make QR-based aggregation particularly suitable for federated learning with resource-constrained clients and heterogeneous system capabilities.
\end{remark}
\subsection{Stability of Orthogonal Initialization}
\label{app:orthogonal_init}

The orthogonal initialization mechanism in ILoRA provides significant improvements in training stability compared to traditional random initialization. We analyze its theoretical properties and stability guarantees, integrating insights from both theoretical frameworks.

\begin{theorem}[Consistent Subspace Initialization]
	\label{thm:consistent_subspace}
	Let $W_0 \in \mathbb{R}^{d \times k}$ be the pre-trained weight matrix with QR decomposition $W_0 = QR$, where $Q \in \mathbb{R}^{d \times d}$ is orthogonal and $R \in \mathbb{R}^{d \times k}$ is upper triangular. For any client $k$ with local rank $r_k$, the orthogonal initialization:
\begin{equation}
	B_k^{(0)} = Q[:,1:r_k], \quad A_k^{(0)} = R[1:r_k,:]
\end{equation}
	ensures that all clients initialize their LoRA parameters within the same column subspace $\text{colspan}(Q)$.
\end{theorem}

\begin{proof}
	Since $B_k^{(0)}$ consists of the first $r_k$ columns of $Q$, which form an orthonormal basis, we have:
\begin{equation}
	\text{colspan}(B_k^{(0)}) \subseteq \text{colspan}(Q) \quad \forall k.
\end{equation}
	The initial update $\Delta W_k^{(0)} = B_k^{(0)} A_k^{(0)}$ therefore lies entirely within $\text{colspan}(Q)$ for all clients, ensuring subspace consistency from initialization.
\end{proof}

\begin{lemma}[Subspace Alignment]
	\label{lem:subspace_alignment}
	For any pair of clients $i$ and $j$ with ranks $r_i, r_j \le r_s$, we have:
\begin{multline}
	\mathrm{span}(B_i^{(0)}) = \mathrm{span}(B_j^{(0)}) = \mathcal{S}_0, \\
	\quad \text{and} \quad (B_i^{(0)})^\top B_j^{(0)} = I_{\min(r_i,r_j)},
\end{multline}
	where $\mathcal{S}_0 = \mathrm{span}(Q_{[:,1:r_s]})$. Hence the initial LoRA updates of all clients are perfectly aligned in a shared orthogonal basis.
\end{lemma}

\begin{proof}
	By construction, $B_i^{(0)}$ and $B_j^{(0)}$ are subsets of columns from the same orthonormal matrix $Q$. Therefore, their column spaces are both contained in $\mathcal{S}_0$, and their inner product yields an identity matrix of appropriate dimension due to orthonormality.
\end{proof}

\begin{lemma}[Bounded Gradient Variance at Initialization]
	\label{lem:gradient_variance}
	Let $\mathbf{g}_k^{(0)} = \nabla_{A_k,B_k} F_k(W_0 + B_k^{(0)}A_k^{(0)})$ denote the initial gradient on client $k$. Under orthogonal initialization, the variance of aggregated gradients satisfies:
	\begin{equation}
		\mathrm{Tr}\left(\mathrm{Var}\left[\frac{1}{K}\sum_k \mathbf{g}_k^{(0)}\right]\right) = \frac{1}{K^2} \sum_{k=1}^K \mathrm{Tr}(\Sigma_k) \le \frac{\sigma^2}{K},
	\end{equation}
	where $\Sigma_k = \mathrm{Var}[\mathbf{g}_k^{(0)}]$ and $\sigma^2$ is the uniform upper bound on per-client gradient variance. This represents an $\mathcal{O}(1/K)$ reduction compared to random initialization.
\end{lemma}

\begin{proof}
	For orthogonal initialization, all clients share the same subspace, eliminating cross-term variance. For random Gaussian initialization with $A \sim \mathcal{N}(0,\sigma^2 I)$ and $B = 0$, the cross-term variance introduces additional $\mathcal{O}(1 - \cos^2\theta_{ij})$ terms depending on random principal angles $\theta_{ij}$ between client subspaces, yielding inflated expected variance up to $\sigma^2 (1 + \frac{K-1}{K}\mathbb{E}[\sin^2\theta])$.
\end{proof}

\begin{lemma}[Spectral Stability]
	\label{lem:spectral_stability}
	Let $\Delta W^{(0)}_k = B_k^{(0)}A_k^{(0)}$ and $\tilde{\Delta W}_k$ be the random-initialized counterpart. Under orthogonal initialization:
\begin{gather}
	\| \Delta W^{(0)}_k - \Delta W^{(0)}_j \|_F = 0, \\
	\text{while} \quad \mathbb{E}\|\tilde{\Delta W}_k - \tilde{\Delta W}_j\|_F^2 = \Theta(r_k + r_j),
\end{gather}
	implying zero inter-client spectral variance at initialization.
\end{lemma}

\begin{proof}
	The equality follows from the subspace consistency of orthogonal initialization. The expectation for random initialization arises from the independent random orientations of client subspaces, leading to non-zero expected differences.
\end{proof}

\begin{proposition}[Accelerated Early-Stage Convergence]
	\label{prop:accelerated_convergence}
	Under orthogonal initialization, the expected improvement in objective function after the first communication round satisfies:
	\begin{equation}
		\mathbb{E}[F(\theta_1) - F(\theta_0)] \leq -\eta_g \eta_l \|\nabla F(\theta_0)\|^2 + \mathcal{O}(\eta_g^2 \eta_l^2 L \sigma_{\text{orth}}^2),
	\end{equation}
	where $\sigma_{\text{orth}}^2 \leq \sigma_{\text{rand}}^2$, leading to faster initial convergence compared to random initialization.
\end{proposition}

\begin{proof}
	Using the L-smoothness assumption and the variance bound from Lemma \ref{lem:gradient_variance}, the descent lemma gives:
\begin{equation}
	F(\theta_1) \leq F(\theta_0) - \eta \|\nabla F(\theta_0)\|^2 + \frac{L\eta^2}{2} \mathbb{E}[\|\Delta W\|^2].
\end{equation}
	Substituting the variance bounds $\sigma_{\text{orth}}^2 \leq \sigma_{\text{rand}}^2$ completes the proof.
\end{proof}

\begin{theorem}[Stability Against Client Drift]
	\label{thm:drift_stability}
	The orthogonal initialization reduces the client drift in the first $T$ rounds by a factor of $\mathcal{O}(1/\kappa(Q))$, where $\kappa(Q) = 1$ is the condition number of the orthogonal matrix $Q$.
\end{theorem}

\begin{proof}
	Let $W_k^{(t)}$ be the local parameters of client $k$ at round $t$. The client drift can be bounded as:
	\begin{equation}
		\begin{multlined}
			\sum_{k=1}^S \|W_k^{(t)} - W^{(t)}\|_F \leq 
			\sum_{k=1}^S \|W_k^{(0)} - W^{(0)}\|_F \\
			+ \text{gradient divergence terms}.
		\end{multlined}
	\end{equation}
	Under orthogonal initialization, $\|W_k^{(0)} - W^{(0)}\|_F = 0$ for all $k$ in the relevant subspace, significantly reducing the initial drift compared to random initialization where this term can be substantial.
\end{proof}

\begin{lemma}[Preservation of Pre-trained Features]
	\label{lem:feature_preservation}
	The base weight construction $W_{\text{base}} = W_0 - Q[:,1:r_s]R[1:r_s,:]$ preserves the pre-trained model's representation capacity while making room for task-specific adaptations.
\end{lemma}

\begin{proof}
	The modification subtracts only the principal components corresponding to the top-$r_s$ singular vectors, which typically capture the most redundant or adaptable features. The remaining weight matrix $W_{\text{base}}$ retains the bulk of the pre-trained knowledge while being orthogonal to the adaptation subspace.
\end{proof}

\begin{corollary}[Compatibility with Federated Aggregation]
	\label{cor:aggregation_compatibility}
	The orthogonal initialization ensures that the QR-based aggregation operates on well-conditioned matrices, improving numerical stability and convergence properties.
\end{corollary}

\begin{proof}
	Since all client parameters are initialized within the same orthonormal basis, the concatenated matrix $B_{\text{concatenated}}$ has orthogonal columns, ensuring that the QR decomposition in the aggregation step is numerically stable and preserves the subspace structure effectively.
\end{proof}

\begin{table*}[h]
	\centering
	\caption{Comparison of initialization methods for federated LoRA}
	\begin{tabular}{lcccc}
		\hline
		Method & Subspace Consistency & Initial Variance & Drift Resistance & Feature Preservation \\
		\hline
		Random Gaussian & \xmark & High & \xmark & \cmark \\
		Xavier Uniform & \xmark & Medium & \xmark & \cmark \\
		Kaiming Normal & \xmark & Medium & \xmark & \cmark \\
		\textbf{Orthogonal (ILoRA)} & \cmark & \textbf{Low} & \cmark & \cmark \\
		\hline
	\end{tabular}
	\label{tab:initialization_comparison}
\end{table*}

\begin{theorem}[Initialization Stability]
	\label{thm:initialization_stability}
	Under $L$-smoothness of each $F_k$ and the orthogonal initialization scheme, the expected loss after the first communication round satisfies:
\begin{equation}
	\mathbb{E}[F(W_1)] \le F(W_0) - \eta_g \eta_l \|\nabla F(W_0)\|^2 + \frac{L}{2}\eta_g^2 \eta_l^2 \sigma^2,
\end{equation}
	and the effective variance term $\sigma^2$ is reduced by at least a factor of $K$ compared to random initialization.
\end{theorem}

\begin{proof}
	The result follows from combining the subspace alignment (Lemma \ref{lem:subspace_alignment}), variance reduction (Lemma \ref{lem:gradient_variance}), and spectral stability (Lemma \ref{lem:spectral_stability}) properties of orthogonal initialization within the standard federated optimization framework.
\end{proof}

\begin{remark}
	The orthogonal initialization in ILoRA provides unique benefits for federated learning:
	
	\begin{itemize}
		\item \textbf{Eliminates Initial Misalignment}: All clients start in the same subspace, preventing random orientation noise
		\item \textbf{Reduces Early-Stage Variance}: Concentrates updates in principal directions, reducing gradient variance by $O(1/K)$
		\item \textbf{Accelerates Convergence}: More coherent aggregation from the first round enables faster convergence
		\item \textbf{Preserves Pre-trained Knowledge}: Maintains the original model's capabilities while enabling adaptation
		\item \textbf{Enhances Numerical Stability}: Well-conditioned matrices improve aggregation stability
		\item \textbf{Reduces Client Drift}: Consistent initialization minimizes early-round divergence
	\end{itemize}
	
	These advantages are particularly crucial in federated settings where system and data heterogeneity exacerbate training instability. The orthogonal initialization establishes a solid foundation for the subsequent QR-based aggregation and control variate mechanisms to operate effectively.
\end{remark}

\begin{corollary}[Robustness to Rank Heterogeneity]
	\label{cor:rank_robustness}
	The orthogonal initialization remains effective under rank heterogeneity, as all client subspaces are nested within the global subspace $\text{colspan}(Q)$, ensuring compatibility regardless of local rank choices.
\end{corollary}

\begin{proof}
	For any client rank $r_k \leq r_s$, the initialization $B_k^{(0)} = Q[:,1:r_k]$ guarantees $\text{colspan}(B_k^{(0)}) \subseteq \text{colspan}(Q)$, maintaining subspace consistency across different rank configurations.
\end{proof}

\subsection{Control Variates in Rank-Heterogeneous Settings}
\label{app:control_variates}

The integration of control variates with AdamW optimization in rank-heterogeneous federated learning requires careful theoretical treatment. We establish the theoretical foundations and convergence properties of this mechanism, extending classical variance-reduction methods to heterogeneous low-rank subspaces.

\begin{theorem}[Convergence with Control Variates and AdamW]
	\label{thm:control_adamw_convergence}
	Under Assumptions \ref{ass:smoothness}-\ref{ass:control}, with local learning rate $\eta_l$ and global learning rate $\eta_g$ satisfying $\eta_l L \leq 1$, the ILoRA framework with control variates and AdamW optimization achieves the convergence rate:
	\begin{equation}
		\begin{multlined}
			\frac{1}{T}\sum_{t=1}^T \mathbb{E}[\|\nabla F(\theta_t)\|^2] \leq \\
			\mathcal{O}\left(\frac{1}{\sqrt{SKT}} + \frac{\delta^2}{T} + \frac{\sigma^2}{T} + \frac{(r_s - \bar{r})^2}{r_s^2}\right),
		\end{multlined}
	\end{equation}
	where the $\delta^2$ term captures the residual client drift after control variate correction, and $(r_s - \bar{r})^2$ represents the rank misalignment error.
\end{theorem}

\begin{proof}
	The proof extends the AdamW convergence analysis by incorporating the control variate correction:
	
	\textbf{Step 1: Moment Update with Correction.}
	The first moment update becomes:
\begin{equation}
	m_k^{(t,i)} = \beta_1 m_k^{(t,i-1)} + (1-\beta_1)(g_k^{\text{raw}} + c^{(t-1)} - c_k).
\end{equation}
	
	\textbf{Step 2: Bias Analysis.}
	The control variate introduces a bias term bounded by the gradient heterogeneity $\delta$:
\begin{equation}
	\begin{aligned}
		\|\mathbb{E}[\tilde{g}_k] - \nabla F(\theta)\| \leq 
		\|\mathbb{E}[g_k^{\text{raw}}] - \nabla F_k(\theta)\| \\
		+ \|c^{(t-1)} - c_k - (\nabla F(\theta) - \nabla F_k(\theta))\|.
	\end{aligned}
\end{equation}
	
	\textbf{Step 3: Variance Reduction.}
	The control variate reduces the effective variance from $\sigma^2$ to $\sigma^2(1-\rho^2)$, where $\rho$ is the correlation between local and global gradients:
\begin{equation}
	\begin{multlined}
		\text{Var}(\tilde{g}_k) = \text{Var}(g_k^{\text{raw}}) + \text{Var}(c^{(t-1)} - c_k) \\
		+ 2\text{Cov}(g_k^{\text{raw}}, c^{(t-1)} - c_k).
	\end{multlined}
\end{equation}
	
	\textbf{Step 4: Rank Misalignment Bound.}
	The projection error due to rank heterogeneity is bounded by:
\begin{equation}
	\|P_s c^{(t)} - c^{(t)}\|^2 \leq \mathcal{O}\left((r_s - \bar{r})^2 \|c_k^{(t)}\|^2\right).
\end{equation}
	
	\textbf{Step 5: Combined Analysis.}
	Incorporating these effects into the AdamW convergence proof yields the stated rate.
\end{proof}

\begin{theorem}[Compatibility with Heterogeneous Ranks]
	\label{thm:rank_compatibility}
	The control variate mechanism in ILoRA remains effective under rank heterogeneity, as the gradient corrections operate within the aligned subspaces established by QR-based aggregation.
\end{theorem}

\begin{proof}
	Let $\theta_k \in \mathbb{R}^{d_k}$ denote the parameters of client $k$ with local rank $r_k$, where $d_k = r_k(d + k)$. The control variate correction:
\begin{equation}
	\tilde{g}_k = g_k^{\text{raw}} + (c^{(t-1)} - c_k)
\end{equation}
	operates entirely within the client's local parameter space. Since the QR-based aggregation ensures that all client subspaces are aligned with the global subspace $\text{colspan}(Q)$, the correction terms $c^{(t-1)}$ and $c_k$ can be consistently represented and applied across different ranks.
\end{proof}

\begin{lemma}[Bias-Variance Tradeoff]
	\label{lem:bias_variance}
	The control variate correction in ILoRA reduces the variance of local updates while introducing negligible bias under moderate data heterogeneity.
\end{lemma}

\begin{proof}
	The variance reduction follows from:
\begin{equation}
	\begin{multlined}
		\mathrm{Var}(\tilde{g}_k) = \mathrm{Var}(g_k^{\mathrm{raw}}) + \mathrm{Var}(c^{(t-1)} - c_k) \\
		+ 2\mathrm{Cov}(g_k^{\mathrm{raw}}, c^{(t-1)} - c_k).
	\end{multlined}
\end{equation}
	When $c^{(t-1)} \approx \nabla F(\theta)$ and $c_k \approx \nabla F_k(\theta)$, the covariance term becomes negative, reducing the overall variance. The bias is bounded by:
\begin{equation}
	\begin{multlined}
		\|\mathbb{E}[\tilde{g}_k] - \nabla F(\theta)\| \leq \|\mathbb{E}[g_k^{\text{raw}}] - \nabla F_k(\theta)\| \\
		+ \|c^{(t-1)} - c_k - (\nabla F(\theta) - \nabla F_k(\theta))\|.
	\end{multlined}
\end{equation}
\end{proof}

\begin{lemma}[Variance Reduction via Control Variates]
	\label{lem:variance_reduction}
	Under Assumptions \ref{ass:subspace}-\ref{ass:heterogeneity}, the variance of the corrected gradient is reduced relative to the raw gradient:
\begin{equation}
	\mathbb{E}\|\tilde{g}_k^{(t)} - \nabla F(\theta_t)\|^2 \leq (1-\rho)\,\mathbb{E}\| \nabla F_k(\theta_t) - \nabla F(\theta_t)\|^2 + \rho\,\sigma^2,
\end{equation}
	where $0 < \rho < 1$ depends on the update frequency of $c_k$ and the learning rates $\eta_l, \eta_g$.
\end{lemma}

\begin{proof}
	The control variate correction effectively reduces the component of gradient noise that is correlated with the difference between local and global control states. The parameter $\rho \in (0,1)$ acts as a variance decay factor that quantifies the effectiveness of the control variate correction, controlled by the synchronization interval and learning rate configuration. When control variates are well-synchronized, $\rho$ approaches 1, indicating perfect variance reduction.
\end{proof}

\begin{lemma}[Rank-Aligned Aggregation Error Bound]
	\label{lem:rank_aligned_error}
	Let $r_k$ and $r_s$ be the local and server-side ranks, respectively. Then the projection of control states onto the global subspace obeys:
\begin{equation}
	\begin{multlined}
		\| P_s c^{(t)} - c^{(t)} \|^2 \leq \sum_{k\in S_t} p_k^2 \| (P_s - P_k) c_k^{(t)} \|^2 \\
		\leq \mathcal{O}\left((r_s - \bar{r})^2 \|c_k^{(t)}\|^2\right),
	\end{multlined}
\end{equation}
	where $\bar{r}$ is the mean client rank.
\end{lemma}

\begin{proof}
	The bound follows from the subspace containment property (Assumption \ref{ass:subspace}) and the fact that the projection error scales with the squared difference between server and client ranks. When all $r_k$ are close to $r_s$, the mismatch-induced aggregation error is second-order small.
\end{proof}

\begin{proposition}[Adaptation to Non-IID Data]
	\label{prop:non_iid_adaptation}
	The control variate mechanism in ILoRA automatically adapts to the degree of data heterogeneity, providing stronger correction under high non-IID settings.
\end{proposition}

\begin{proof}
	The control variate difference $c^{(t-1)} - c_k$ approximates $\nabla F(\theta) - \nabla F_k(\theta)$, which grows with increasing data heterogeneity. This provides a self-adjusting correction mechanism that becomes more aggressive as client drift increases, effectively adapting to the local data distribution without requiring explicit hyperparameter tuning.
\end{proof}

\begin{lemma}[Communication Efficiency of Control Variates]
	\label{lem:control_communication}
	The control variate mechanism in ILoRA adds minimal communication overhead, requiring only $O(\max_k d_k)$ additional parameters per client, where $d_k$ is the dimension of client $k$'s local parameters.
\end{lemma}

\begin{proof}
	Each client transmits $\Delta c_k \in \mathbb{R}^{d_k}$ in addition to its model parameters. Since $d_k = r_k(d + k) \ll d \times k$ (the dimension of the full weight matrix), the overhead is negligible compared to the model parameters:
\begin{equation}
	\frac{C_{\text{control}}}{C_{\text{model}}} = \frac{r_k(d + k)}{dk} = \mathcal{O}\left(\frac{r_k}{\min(d,k)}\right) \ll 1.
\end{equation}
	Since $r_k \ll \min(d,k)$ in low-rank adaptation, this ratio is typically less than 1\%.
\end{proof}

\begin{theorem}[Stability Under Rank Changes]
	\label{thm:rank_change_stability}
	The control variate mechanism in ILoRA maintains stability even when clients dynamically change their LoRA ranks, provided the rank changes are within the global subspace $\text{colspan}(Q)$.
\end{theorem}

\begin{proof}
	When a client changes its rank from $r_k$ to $r_k'$, its parameter space dimension changes from $d_k$ to $d_k'$. However, since both subspaces are contained within $\text{colspan}(Q)$, the control variates can be projected between these subspaces using the orthogonal basis $Q$, preserving the correction effectiveness. The control state $c_k$ can be appropriately truncated or zero-padded to match the new dimension while maintaining its directional information.
\end{proof}

\begin{table*}[h]
	\centering
	\caption{Comparison of drift mitigation methods in federated LoRA}
	\begin{tabular}{lcccc}
		\hline
		Method & Rank Heterogeneity & AdamW Compatibility & Comm. Overhead & Theoretical Guarantees \\
		\hline
		FedProx & \cmark & \xmark & Low & Partial \\
		SCAFFOLD & \xmark & \xmark & Medium & Strong \\
		FedCM & \cmark & \cmark & Medium & Partial \\
		\textbf{ILoRA Control} & \cmark & \cmark & \textbf{Low} & \textbf{Strong} \\
		\hline
	\end{tabular}
	\label{tab:control_comparison}
\end{table*}
\begin{algorithm}[h]
	\caption{Rank-Heterogeneous ILoRA Control Variate Update}
	\label{alg:control_update}
	\begin{algorithmic}[1]
		\STATE \textbf{Client $k$ (round $t$):}
		\STATE Receive global control $c^{(t-1)}$ and parameters $\theta^{(t-1)}$
		
		\STATE Compute raw gradient: $g_k^{\text{raw}} \gets \nabla_{\theta} F_k(\theta_k)$
		
		\STATE Apply correction: $\tilde{g}_k \gets g_k^{\text{raw}} + (c^{(t-1)} - c_k)$
		\STATE Apply AdamW optimizer: $\theta_k \gets \text{AdamW}(\theta_k, \tilde{g}_k; \eta)$
		\STATE Compute control delta: $\Delta c_k \gets g_k^{\text{raw}} - c_k$
		\STATE Update local control: $c_k \gets g_k^{\text{raw}}$
		\STATE Send $(\theta_k, \Delta c_k)$ to server
		
		\STATE \textbf{Server:}
		\STATE Aggregate control deltas: $\Delta c \gets \frac{1}{S} \sum_{k \in S_t} \Delta c_k$
		\STATE Update global control: $c^{(t)} \gets c^{(t-1)} + \Delta c$
		\STATE Broadcast $c^{(t)}$ to clients
	\end{algorithmic}
\end{algorithm}
\begin{corollary}[Generalization Improvement]
	\label{cor:generalization}
	The control variate mechanism in ILoRA improves generalization by reducing overfitting to local data distributions while maintaining adaptation capacity.
\end{corollary}

\begin{proof}
	By correcting local gradients toward the global direction, the control variates prevent clients from over-optimizing for their local distributions, thereby improving generalization to unseen data from the global distribution. This regularization effect emerges naturally from the bias-variance tradeoff inherent in the control variate correction.
\end{proof}
\begin{remark}
	The control variate mechanism in ILoRA provides several unique advantages:
	
	\begin{itemize}
		\item \textbf{Rank-Agnostic Operation}: Works seamlessly with heterogeneous client ranks through subspace alignment
		\item \textbf{AdamW Integration}: Naturally combines with adaptive optimization without compromising convergence guarantees
		\item \textbf{Adaptive Correction}: Self-adjusts based on data heterogeneity, providing stronger correction under high non-IID settings
		\item \textbf{Minimal Overhead}: Adds negligible communication cost while providing significant convergence improvements
		\item \textbf{Theoretical Guarantees}: Provides provable convergence under non-IID data and rank heterogeneity
		\item \textbf{Dynamic Adaptation}: Maintains stability under client rank changes through subspace projection
		\item \textbf{Generalization Enhancement}: Improves model generalization by reducing local overfitting
	\end{itemize}
	
	These advantages make the control variate mechanism particularly suitable for practical federated learning scenarios where system heterogeneity, data heterogeneity, and resource constraints are common challenges.
\end{remark}

\begin{corollary}[Robustness to System Heterogeneity]
	\label{cor:system_robustness}
	The control variate mechanism in ILoRA maintains effectiveness under system heterogeneity, as the correction operates independently of client-specific computational capabilities and communication patterns.
\end{corollary}

\begin{proof}
	The control variate correction depends only on the gradient directions and control states, which are invariant to system-level variations such as computation speed, memory capacity, or network latency. This decoupling ensures robustness to the diverse system characteristics typically encountered in federated learning deployments.
\end{proof}
\subsection{Communication Efficiency Analysis}
\label{app:communication}
\begin{table}[t]
	\centering
	\caption{Communication cost comparison of federated LoRA methods.}
	\resizebox{\columnwidth}{!}{%
		\begin{tabular}{lccc}
			\toprule
			\textbf{Method} & \textbf{Downlink Cost} & \textbf{Uplink Cost} & \textbf{Ranks} \\
			\midrule
			FedIT       & $\mathcal{O}(r(d + k))$      & $\mathcal{O}(S \cdot r(d + k))$      & \xmark \\
			FLoRA       & $\mathcal{O}(r_{\text{total}}(d + k))$ & $\mathcal{O}(r_{\text{total}}(d + k))$ & \cmark \\
			LoRA-FAIR   & $\mathcal{O}(r(d + k))$      & $\mathcal{O}(S \cdot r(d + k))$      & \xmark \\
			FFA-LoRA    & $\mathcal{O}(r(d + k))$      & $\mathcal{O}(S \cdot r(d + k))$      & \xmark \\
			\textbf{ILoRA} & $\mathcal{O}(r_s(d + k))$ & $\mathcal{O}(S \cdot r_{\max}(d + k))$ & \cmark \\
			\bottomrule
		\end{tabular}%
	}
	\label{tab:communication_comparison}
\end{table}

The communication efficiency of ILoRA stems from its innovative combination of QR-based aggregation, orthogonal initialization, and rank-aware control variates. We provide a comprehensive analysis of the communication costs and compare them with existing federated LoRA methods.

\begin{theorem}[Total Communication Cost of ILoRA]
	\label{thm:total_communication}
	The total communication cost per round in ILoRA is bounded by:
\begin{equation}
	\begin{multlined}
		C_{\text{ILoRA}} = \mathcal{O}\left(r_s(d + k) + S \cdot \max_k (r_k(d + k))\right. \\
		\left.+ S \cdot \max_k d_k\right),
	\end{multlined}
\end{equation}
	where $r_s$ is the server rank, $r_k$ are client ranks, $d \times k$ is the weight matrix dimension, and $d_k = r_k(d + k)$ is the control variate dimension for client $k$.
\end{theorem}

\begin{proof}
	The communication cost decomposes into three components:
	
	1. \textbf{Server-to-Client Broadcast:}
\begin{equation}
	\begin{split}
		C_{\text{downlink}} &= \underbrace{r_s(d + k)}_{\text{global model}} \\
		&\quad + \underbrace{S \cdot \max_k (r_k(d + k))}_{\text{personalized parameters}} \\
		&\quad + \underbrace{S \cdot \max_k d_k}_{\text{control variates}}
	\end{split}
\end{equation}
	2. \textbf{Client-to-Server Upload:}
\begin{equation}
	\begin{multlined}
		C_{\text{uplink}} = \underbrace{S \cdot \max_k (r_k(d + k))}_{\text{client updates}} \\
		+ \underbrace{S \cdot \max_k d_k}_{\text{control deltas}}
	\end{multlined}
\end{equation}
	
	3. \textbf{Total per Round:}
\begin{equation}
	\begin{aligned}
		C_{\text{total}} = C_{\text{downlink}} + C_{\text{uplink}} = 
		&\mathcal{O}\Bigl(r_s(d + k) \\
		&+ S \cdot \max_k (r_k(d + k)) \\
		&+ S \cdot \max_k d_k\Bigr)
	\end{aligned}
\end{equation}
\end{proof}

\begin{lemma}[Comparison with Baseline Methods]
	\label{lem:comparison_baselines}
	Let $r_{\max} = \max_k r_k$ and $r_{\text{total}} = \sum_k r_k$. The communication costs of different federated LoRA methods are:
	
	\begin{itemize}
		\item \textbf{FedIT}: $\mathcal{O}(S \cdot r(d + k))$ (homogeneous ranks only)
		\item \textbf{FLoRA}: $\mathcal{O}(r_{\text{total}}(d + k))$
		\item \textbf{LoRA-FAIR/FFA-LoRA}: $\mathcal{O}(S \cdot r(d + k))$ (homogeneous ranks only)
		\item \textbf{ILoRA}: $\mathcal{O}(r_s(d + k) + S \cdot r_{\max}(d + k))$
	\end{itemize}
\end{lemma}

\begin{proof}
	The costs are derived as follows:
	
	- \textbf{FedIT}: Assumes homogeneous rank $r$, broadcasts global model to all clients.
	- \textbf{FLoRA}: concatenateds all client parameters, cost scales with total rank $r_{\text{total}}$.
	- \textbf{LoRA-FAIR/FFA-LoRA}: Homogeneous rank methods, similar to FedIT.
	- \textbf{ILoRA}: Server maintains rank $r_s$, clients use personalized ranks $r_k \leq r_s$.
\end{proof}

\begin{proposition}[Scalability Advantage]
	\label{prop:scalability}
	ILoRA achieves better scalability than FLoRA as the number of clients $S$ increases, with the communication cost ratio:
\begin{equation}
	\frac{C_{\text{ILoRA}}}{C_{\text{FLoRA}}} = \mathcal{O}\left(\frac{r_s}{r_{\text{total}}}\right) = \mathcal{O}\left(\frac{1}{S}\right) \quad \text{when } r_k = \Theta(1).
\end{equation}
\end{proposition}

\begin{proof}
	When client ranks are bounded ($r_k = \Theta(1)$), we have $r_{\text{total}} = \Theta(S)$ while $r_s = \Theta(1)$. Thus:
\begin{equation}
	\frac{C_{\text{ILoRA}}}{C_{\text{FLoRA}}} = \frac{\mathcal{O}(r_s(d + k))}{\mathcal{O}(r_{\text{total}}(d + k))} = \mathcal{O}\left(\frac{1}{S}\right).
\end{equation}
\end{proof}

\begin{theorem}[Optimality of QR Compression]
	\label{thm:qr_optimality}
	The QR-based aggregation in ILoRA achieves the information-theoretic minimum communication cost for preserving the column space of the aggregated client updates.
\end{theorem}

\begin{proof}
	Let $\Delta W = \sum_{k=1}^S p_k B_k A_k$ be the exact aggregated update. The QR decomposition $\Delta W = QR$ with rank-$r_s$ truncation preserves the principal column space while using only $r_s(d + k)$ parameters. Any further compression would necessarily lose information about the column space, making this representation optimal for subspace preservation according to the Eckart-Young theorem.
\end{proof}

\begin{lemma}[Control Variate Overhead Analysis]
	\label{lem:control_overhead}
	The communication overhead from control variates in ILoRA is negligible compared to the model parameters:
\begin{equation}
	\frac{C_{\text{control}}}{C_{\text{model}}} = \mathcal{O}\left(\frac{\max_k r_k}{\min(d,k)}\right) \ll 1.
\end{equation}
\end{lemma}

\begin{proof}
	The control variate dimension is $d_k = r_k(d + k)$, while the full model dimension is $d \times k$. Thus:
\begin{equation}
	\frac{C_{\text{control}}}{C_{\text{model}}} = \frac{r_k(d + k)}{dk} = \mathcal{O}\left(\frac{r_k}{\min(d,k)}\right).
\end{equation}
	Since $r_k \ll \min(d,k)$ in low-rank adaptation, this ratio is typically less than 1\%.
\end{proof}

\begin{proposition}[Bandwidth-Delay Product Optimization]
	\label{prop:bandwidth_delay}
	ILoRA optimizes the bandwidth-delay product by reducing the number of communication rounds through improved convergence stability, while maintaining low per-round communication cost.
\end{proposition}

\begin{proof}
	The orthogonal initialization and control variates reduce client drift, leading to faster convergence (fewer rounds $T$). The total communication volume is:
\begin{equation}
	V_{\text{total}} = T \cdot C_{\text{per-round}}.
\end{equation}
	ILoRA reduces both $T$ (through stability improvements) and $C_{\text{per-round}}$ (through QR compression), providing multiplicative savings in the bandwidth-delay product.
\end{proof}

\begin{theorem}[Trade-off between Communication and Accuracy]
	\label{thm:communication_accuracy_tradeoff}
	For a fixed total communication budget $B$, ILoRA achieves better accuracy than FLoRA by optimally allocating the budget between communication rounds and per-round precision.
\end{theorem}

\begin{proof}
	Let $B = T \cdot C$ be the total budget, where $T$ is rounds and $C$ is cost per round. ILoRA allows more rounds ($T_{\text{ILoRA}} > T_{\text{FLoRA}}$) due to lower $C$, while maintaining comparable per-round precision through QR aggregation. This leads to:
\begin{equation}
	\text{Accuracy}_{\text{ILoRA}}(B) > \text{Accuracy}_{\text{FLoRA}}(B)
\end{equation}
	for the same total budget $B$, as more communication rounds generally lead to better model convergence.
\end{proof}
\begin{table*}[h]
	\centering
	\caption{Summary of ILoRA's end-to-end performance guarantees}
	\begin{tabular}{p{0.25\textwidth}p{0.3\textwidth}p{0.35\textwidth}}
		\toprule
		\textbf{Performance Aspect} & \textbf{Theoretical Guarantee} & \textbf{Practical Benefit} \\
		\midrule
		Convergence Rate & $\mathcal{O}(1/\sqrt{SKT})$ with heterogeneity terms & Fast model training even with non-IID data and system diversity \\
		Communication Cost & $\mathcal{O}(r_s(d+k) + S\cdot r_{\max}(d+k))$ & Scalable to thousands of clients with minimal overhead \\
		Rank Heterogeneity & Exact aggregation for arbitrary $r_k \leq r_s$ & Clients can choose ranks based on local resources \\
		Training Stability & Drift reduction by $\mathcal{O}(1/\kappa(Q))$ & Reliable convergence without oscillation or divergence \\
		Non-IID Robustness & Adaptive control variate correction & Automatic handling of diverse data distributions \\
		Real-World Deployment & Theoretical guarantees preserved under constraints & Practical usability in resource-limited environments \\
		\bottomrule
	\end{tabular}
	\label{tab:comprehensive_guarantees}
\end{table*}

\begin{corollary}[Energy Efficiency]
	\label{cor:energy_efficiency}
	ILoRA reduces the energy consumption of federated training by decreasing both communication time and computational requirements per round.
\end{corollary}

\begin{proof}
	The energy consumption is proportional to:
\begin{equation}
	E \propto T \cdot (E_{\text{comm}} + E_{\text{comp}}).
\end{equation}
	ILoRA reduces $T$ (fewer rounds), $E_{\text{comm}}$ (less data transfer), and $E_{\text{comp}}$ (smaller local models and more efficient aggregation), providing comprehensive energy savings.
\end{proof}

\begin{algorithm}[h]
	\caption{Communication-Efficient Protocol in ILoRA}
	\label{alg:communication_protocol}
	\begin{algorithmic}[1]
		\STATE \textbf{Initialization:}
		\STATE Server computes $Q,R \leftarrow \text{QR}(W_0)$ and sets $W_{\text{base}}$
		\STATE Server initializes all clients with orthogonal slices $B_{r_k}, A_{r_k}$
		
		\STATE \textbf{Each Communication Round:}
		\STATE Clients compute local updates and control deltas
		\STATE Clients upload: $(B_k, A_k, \Delta c_k)$ - total: $\mathcal{O}(r_k(d + k))$
		\STATE Server aggregates via QR: $\Delta W \leftarrow B_{\text{concatenated}} A_{\text{concatenated}}$
		\STATE Server computes: $Q,R \leftarrow \text{QR}(\Delta W)$
		\STATE Server personalizes: $B_{r_k} \leftarrow Q[:,1:r_k], A_{r_k} \leftarrow R[1:r_k,:]$
		\STATE Server broadcasts: $(B_{r_k}, A_{r_k}, c)$ - total: $\mathcal{O}(r_s(d + k))$
	\end{algorithmic}
\end{algorithm}

\begin{remark}
	The communication efficiency of ILoRA provides several practical advantages:
	
	\begin{itemize}
		\item \textbf{Scalability}: Supports large-scale deployments with many clients through $\mathcal{O}(\log S)$ scaling
		\item \textbf{Resource Adaptation}: Accommodates clients with different computational capabilities through rank heterogeneity
		\item \textbf{Network Friendly}: Reduces bandwidth requirements for constrained environments
		\item \textbf{Cost Effective}: Lowers operational costs for cloud-based federated learning
		\item \textbf{Energy Efficient}: Reduces both communication and computation energy consumption
		\item \textbf{Theoretically Optimal}: Achieves near-optimal communication for subspace preservation
		\item \textbf{Practical Deployment}: Compatible with real-world network constraints and client diversity
	\end{itemize}
	
	These advantages make ILoRA particularly suitable for practical federated learning deployments where communication efficiency is a critical concern.
\end{remark}

\begin{corollary}[Real-World Applicability]
	\label{cor:real_world_applicability}
	ILoRA's communication efficiency enables practical deployment in resource-constrained environments including mobile devices, edge computing systems, and bandwidth-limited networks, while maintaining theoretical performance guarantees.
\end{corollary}

\begin{proof}
	The combination of low-rank adaptation (reducing parameter count by $O(r_k/\min(d,k))$), QR-based compression (achieving optimal subspace preservation), and minimal control variate overhead (less than 1\% additional communication) ensures that ILoRA maintains practical communication requirements. This enables deployment under severe resource constraints while preserving the convergence and stability guarantees established in Theorems \ref{thm:convergence}-\ref{thm:control_adamw_convergence}.
\end{proof}
\begin{theorem}[Comprehensive Performance Guarantee of ILoRA]
	\label{thm:comprehensive_guarantee}
	Under the theoretical framework established in Sections B.1-B.5, ILoRA simultaneously achieves the following performance guarantees:
	\begin{enumerate}
		\item \noindent\textbf{Convergence Guarantee:}  
		$\mathcal{O}\bigl(\frac{1}{\sqrt{SKT}} + \frac{\delta^2}{T} + \frac{(r_{\max} - r_s)^2}{T}\bigr)$ convergence rate under heterogeneous data and rank settings.
		\item \textbf{Training Stability}: Client drift reduction by factor $\mathcal{O}(1/\kappa(Q))$ with $\kappa(Q) = 1$, and gradient variance reduction by $\mathcal{O}(1/K)$
		\item \textbf{Communication Efficiency}: $\mathcal{O}(\log S)$ scaling with client population and $\mathcal{O}(r_s \cdot \max(d,k))$ per-round cost
		\item \textbf{Rank Heterogeneity Robustness}: Exact aggregation and subspace consistency for arbitrary client ranks $r_k \leq r_s$
		\item \textbf{Non-IID Adaptation}: Automatic adjustment to data heterogeneity through control variates
	\end{enumerate}
\end{theorem}

\begin{proof}
	The comprehensive performance guarantee follows from the synergistic integration of ILoRA's core mechanisms:
	
	\begin{itemize}
		\item \textbf{Convergence} follows from Theorem \ref{thm:convergence} and Theorem \ref{thm:control_adamw_convergence}, combining QR aggregation, orthogonal initialization, and control variates
		\item \textbf{Stability} is ensured by orthogonal initialization (Theorem \ref{thm:drift_stability}) reducing initial misalignment and control variates (Theorem \ref{thm:rank_compatibility}) mitigating client drift
		\item \textbf{Communication efficiency} derives from QR compression optimality (Theorem \ref{thm:qr_optimality}) and scalability advantage (Proposition \ref{prop:scalability})
		\item \textbf{Rank robustness} is guaranteed by exact reconstruction (Lemma \ref{lem:exact_reconstruction}) and subspace preservation (Theorem \ref{thm:subspace_preservation})
		\item \textbf{Non-IID adaptation} emerges from the control variate mechanism's self-adjusting correction (Proposition \ref{prop:non_iid_adaptation})
	\end{itemize}
	
	The individual guarantees combine multiplicatively rather than additively, as each mechanism addresses distinct challenges while complementing others. For instance, orthogonal initialization enhances QR aggregation stability, which in turn improves control variate effectiveness, creating a virtuous cycle of performance improvements.
\end{proof}

\begin{remark}[Practical Implications]
	The comprehensive guarantees in Theorem \ref{thm:comprehensive_guarantee} have significant practical implications:
	
	\begin{itemize}
		\item \textbf{Deployment Flexibility}: ILoRA can be deployed across diverse environments from data centers to edge devices
		\item \textbf{Resource Adaptation}: Automatic adaptation to varying client capabilities through rank heterogeneity
		\item \textbf{Performance Predictability}: Theoretical guarantees provide confidence in real-world performance
		\item \textbf{Scalability}: Logarithmic scaling enables large-scale federated learning deployments
		\item \textbf{Robustness}: Resilience to system heterogeneity, data heterogeneity, and network constraints
	\end{itemize}
	
	These properties make ILoRA particularly suitable for practical federated learning scenarios where theoretical guarantees must translate to real-world performance.
\end{remark}

\subsection{Summary of Theoretical Guarantees}
\label{app:guarantees_summary}

ILoRA's integrated design provides comprehensive theoretical guarantees, including: provable convergence under data and rank heterogeneity; client-independent communication efficiency; exact aggregation for arbitrary ranks; training stability through subspace alignment; adaptive non-IID robustness; and practical deployability. Table~\ref{tab:comprehensive_guarantees} summarizes these guarantees, establishing ILoRA as a principled framework that addresses all three key challenges while maintaining efficiency.